\documentclass{article}

\PassOptionsToPackage{numbers}{natbib}

\usepackage[final]{neurips_2022}

\usepackage[utf8]{inputenc} 
\usepackage[T1]{fontenc}    
\usepackage{hyperref}       
\usepackage{url}            
\usepackage{booktabs}       
\usepackage{amsfonts}       
\usepackage{nicefrac}       
\usepackage{microtype}      
\usepackage{xcolor}         
\usepackage{float}
\usepackage{autonum}

\usepackage{amsmath, amsfonts, amssymb, amsthm}
\usepackage{bbm}
\usepackage{bm}
\usepackage[english]{babel}
\usepackage[ruled,vlined]{algorithm2e}
\usepackage{natbib}
\usepackage{xcolor}
\usepackage{graphicx}
\usepackage{subcaption}
\usepackage{todonotes}
\usepackage{wrapfig}
\usepackage{enumitem}

\newtheorem{assumption}{Assumption}
\newcommand{\rd}{\mathrm{d}}

\title{A Continuous Time Framework for Discrete Denoising Models}

\newcommand{\Affiliation}{%
\end{tabular}\\\begin{tabular}[t]{c}\ignorespaces%
}

\author{%
  Andrew Campbell$^1$
  \And 
  Joe Benton$^1$
  \And
  Valentin De Bortoli$^2$
  \AND
  Tom Rainforth$^1$
  \And
  George Deligiannidis$^1$
  \And
  Arnaud Doucet$^1$
  \Affiliation\\
  $^1$Department of Statistics, University of Oxford, UK \quad $^2$CNRS ENS Ulm, Paris, France\\
  \texttt{ \{campbell, benton, rainforth, deligian, doucet\}@stats.ox.ac.uk}\\
  \texttt{valentin.debortoli@gmail.com}
}

\usepackage{etoc}

\begin{document}

\etocdepthtag.toc{mtchapter}
\etocsettagdepth{mtchapter}{subsection}
\etocsettagdepth{mtappendix}{none}

\maketitle

\newcommand{\dt}{\Delta t}
\newcommand{\xdt}{\tilde{x}}
\newcommand{\xt}{x}
\newcommand{\E}{\mathbb{E}}
\newcommand{\KL}{\text{KL}}
\newcommand{\Dd}{1:D \backslash d}
\newcommand{\pdata}{p_{\textup{data}}}
\newcommand{\pref}{p_{\textrm{ref}}}
\newcommand{\LCT}{\mathcal{L}_{\textup{CT}}}
\newcommand{\LDT}{\mathcal{L}_{\textup{DT}}}
\newcommand{\LeCT}{\mathcal{L}_{\textup{eCT}}}
\newtheorem{theorem}{Theorem}
\newtheorem{proposition}{Proposition}
\newenvironment{talign}
 {\let\displaystyle\textstyle\align}
 {\endalign}

\maketitle
\begin{abstract}
We provide the first complete continuous time framework for denoising diffusion models of discrete data. This is achieved by formulating the forward noising process and corresponding reverse time generative process as Continuous Time Markov Chains (CTMCs). The model can be efficiently trained using a continuous time version of the ELBO. We simulate the high dimensional CTMC using techniques developed in chemical physics and exploit our continuous time framework to derive high performance samplers that we show can outperform discrete time methods for discrete data. The continuous time treatment also enables us to derive a novel theoretical result bounding the error between the generated sample distribution and the true data distribution.

\end{abstract}

\section{Introduction}

Diffusion/score-based/denoising models \cite{sohl2015deep, song2019generative, ho2020denoising, song2020score} are a popular class of generative models that achieve state-of-the-art sample quality with good coverage of the data distribution \cite{dhariwal2021diffusion} all whilst using a stable, non-adversarial, simple to implement training objective. The general framework is to define a forward noising process that takes in data and gradually corrupts it until the data distribution is transformed into a simple distribution that is easy to sample. The model then learns to reverse this process by learning the logarithmic gradient of the noised marginal distributions known as the score.

Most previous work on denoising models operates on a continuous state space. However, there are many problems for which the data we would like to model is discrete.
This occurs, for example, in text, segmentation maps, categorical features, discrete latent spaces, and the direct 8-bit representation of images. Previous work has tried to realize the benefits of the denoising framework on discrete data problems, with promising initial results~\cite{hoogeboom2021argmax, hoogeboom2021autoregressive, austin2021structured, esser2021imagebart, hoogeboom2022equivariant, cohen2022diffusion, johnson2021beyond, gu2021vector}. 

All of these previous approaches train and sample the model in discrete \emph{time}. 
Unfortunately, working in discrete time has notable drawbacks. It generally forces the user to pick a partition of the process at training time and the model only learns to denoise at these fixed time points. Due to the fixed partition, we are then limited to a simple ancestral sampling strategy. In continuous time, the model instead learns to denoise for any arbitrary time point in the process. This complete specification of the reverse process enables much greater flexibility in defining the reverse sampling scheme. For example, in continuous state spaces, continuous time samplers that greatly reduce the sampling time have been devised \cite{jolicoeur2021gotta, zhang2022fast, salimans2022progressive, chung2021come} as well as ones that improve sample quality \cite{song2020score, dockhorn2021score}. The continuous time interpretation has also enabled the derivation of interesting theoretical properties such as error bounds \cite{de2021diffusion} in continuous state spaces.

To allow these benefits to be exploited for discrete state spaces as well, we formulate a continuous time framework for discrete denoising models. Specifically, our contributions are as follows. We formulate the forward noising process as a Continuous Time Markov Chain (CTMC) and identify the generative CTMC that is the time-reversal of this process. We then bound the log likelihood of the generated data distribution, giving a continuous time equivalent of the ELBO that can be used for efficient training of a parametric approximation to the true generative reverse process. To efficiently simulate the parametric reverse process, we leverage tau-leaping \cite{gillespie2001approximate} and propose a novel predictor-corrector type scheme that can be used to improve simulation accuracy. The continuous time framework allows us to derive a bound on the error between the true data distribution and the samples generated from the approximate reverse process simulated with tau-leaping. Finally, we demonstrate our proposed method on the generative modeling of images from the CIFAR-10 dataset and monophonic music sequences. Notably, we find our tau-leaping with predictor-corrector sampler can provide higher quality CIFAR10 samples than previous discrete time discrete state approaches, further closing the performance gap between when images are modeled as discrete data or as continuous data.

Proofs for all propositions and theorems are given in the Appendix.

\begin{figure}
    \centering
    \hspace{-1cm}
    \includegraphics[width=0.9\linewidth]{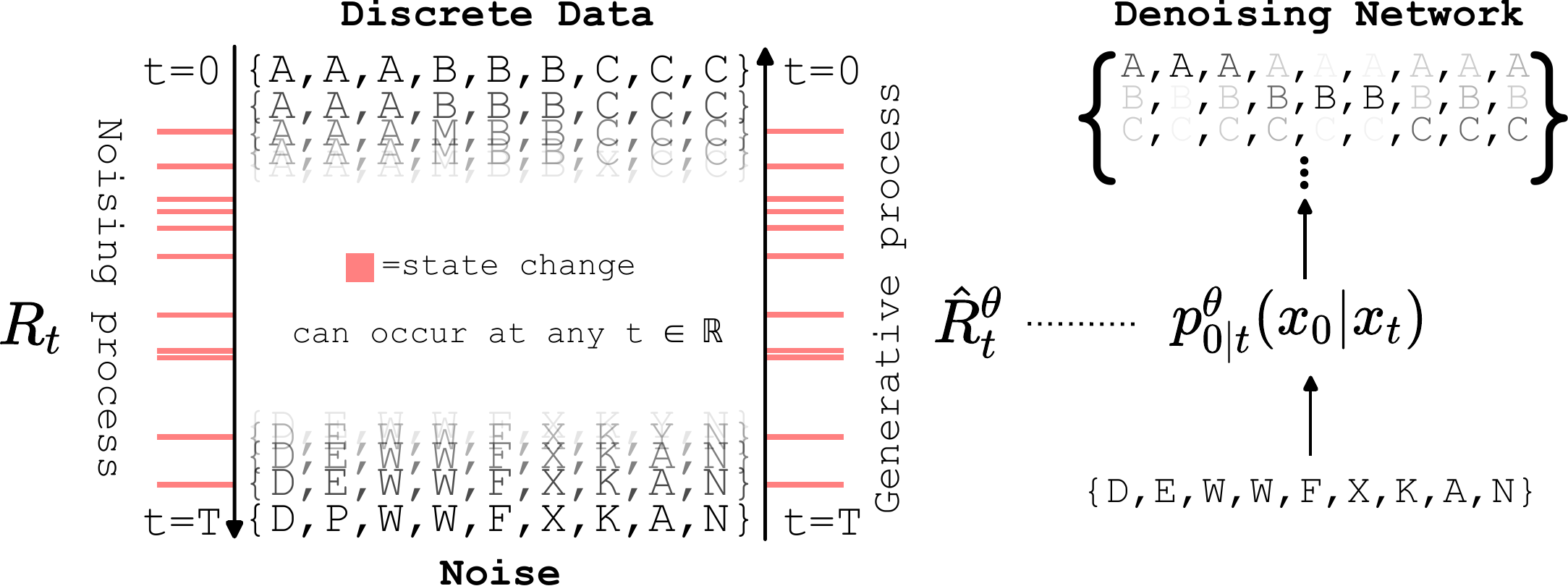}
    \caption{The forward noising process corrupts data according to $R_t$, the rate of corruption events at time $t$. The noising process' time reversal gives the generative process which is defined through $\hat{R}_t^\theta$, the rate of generative events at time $t$. $\hat{R}_t^\theta$ is parameterized through the denoising network, $p_{0|t}^\theta(x_0 | x_t)$, which outputs categorical probabilities over clean $x_0$ values conditioned on a noisy $x_t$.}
    \label{fig:fig1}
\end{figure}

\section{Background on Discrete Denoising Models} \label{sec:Background}
In the discrete time, discrete state space case, we aim to model discrete data $x_0 \in \mathcal{X}$ with finite cardinality $S = | \mathcal{X} |$. We assume $x_0 \sim \pdata(x_0)$ for some discrete data distribution $\pdata(x_0)$. We define a forward noising process that transforms $\pdata(x_0)$ to some distribution $q_K(x_K)$ that closely approximates an easy to sample distribution $\pref(x_K)$. This is done by defining forward kernels $q_{k+1|k}(x_{k+1} | x_k)$ that all admit $\pref$ as a stationary distribution and mix reasonably quickly. For example, one can use a simple uniform kernel \cite{hoogeboom2021argmax, austin2021structured}, $q_{k+1|k}(x_{k+1}|x_k) =  \delta_{x_{k+1}, x_k} (1 - \beta) + ( 1- \delta_{x_{k+1}, x_k})\beta / (S - 1) $ where $\delta$ is a Kronecker delta. The corresponding $\pref$ is the uniform distribution over all states. Other choices include: an absorbing state kernel---where for each state there is a small probability that it transitions to some absorbing state---or a discretized Gaussian kernel---where only transitions to nearby states have significant probability (valid for spaces with ordinal structure) \cite{austin2021structured}.

After defining $q_{k+1|k}$, we have a forward joint decomposition as follows
\begin{equation}
   \textstyle q_{0:K}(x_{0:K}) = \pdata(x_0) \prod_{k=0}^{K-1} q_{k+1|k}(x_{k+1}|x_k).
\end{equation}
The joint distribution $q_{0:K}(x_{0:K})$ also admits a reverse decomposition:
\begin{equation}
  \textstyle  q_{0:K}(x_{0:K}) = q_K(x_K) \prod_{k=0}^{K-1} q_{k|k+1}(x_k | x_{k+1}) ~~ \text{where} ~~ q_{k|k+1}(x_k | x_{k+1}) = \frac{q_{k+1|k}(x_{k+1}|x_k) q_k(x_k)}{q_{k+1}(x_{k+1})}.
\end{equation}
Here $q_k(x_k)$ denotes the marginal of $q_{0:K}(x_{0:K})$ at time $k$.
If one had access to $q_{k|k+1}$ and could sample $q_K$ exactly, then samples from $\pdata(x_0)$ could be produced by first sampling $x_K\sim q_K(\cdot)$ and then ancestrally sampling the reverse kernels, i.e. $x_k\sim q_{k|k+1}(\cdot|x_{k+1})$.

However, in practice, $q_{k|k+1}$ is intractable and needs to be approximated with a parametric reverse kernel, $p^\theta_{k|k+1}$. This kernel is commonly defined  through the analytic $q_{k|k+1, 0}$ distribution and a parametric `denoising' model $p^\theta_{0|k+1}$ \cite{hoogeboom2021argmax, austin2021structured},
\begin{talign}
    p^\theta_{k|k+1}(x_k | x_{k+1}) &\triangleq \sum_{x_0} q_{k|k+1, 0}(x_k | x_{k+1}, x_0) p^\theta_{0|k+1}(x_0 | x_{k+1}) \\
    &= q_{k+1|k}(x_{k+1}|x_k) \sum_{x_0} \frac{q_{k|0}(x_k | x_0)}{q_{k+1|0}(x_{k+1}|x_0)} p_{0|k+1}^\theta(x_0 | x_{k+1}). \label{eq:discreteTimeParametricKernel}
\end{talign}

Though $q_K(x_K)$ is also intractable, for large $K$ we can reliably approximate it with $\pref(x_K)$. 
Note that the faster the transitions mix, the more accurate this approximation becomes. Approximate samples from $\pdata(x_0)$ can then be obtained by sampling the generative joint distribution
\begin{equation}
    \textstyle p^\theta_{0:K}(x_{0:K}) = \pref(x_K) \prod_{k=0}^{K-1} p^\theta_{k|k+1}(x_k | x_{k+1}),
\end{equation}
where $\theta$ is trained through minimizing the negative discrete time (DT) ELBO which is an upper bound on the negative model log-likelihood
\begin{equation}
    \textstyle \E_{\pdata(x_0)} \left[ - \log p^\theta_0(x_0) \right] \leq \E_{q_{0:K}(x_{0:K})} \left[ - \log \frac{p_{0:K}^\theta(x_{0:K})}{q_{1:K|0}(x_{1:K} | x_0)}  \right] = \mathcal{L}_{\textup{DT}} (\theta).
\end{equation}
It was shown in \cite{sohl2015deep} that $\mathcal{L}_{\textup{DT}}$ can be re-written as
\begin{talign}
    \textstyle \mathcal{L}_{\textup{DT}}(\theta) = \mathbb{E}_{\pdata(x_0)} \Big[& \KL( q_{K|0}(x_K | x_0) || \pref(x_K) ) - \E_{q_{1|0}(x_1 | x_0)} \left[ \log p^\theta_{0|1}(x_0 | x_1)\right] \\
    & + \sum_{k=1}^{K-1} \mathbb{E}_{q_{k+1|0}(x_{k+1} | x_0)} \left[ \KL(q_{k|k+1,0}(x_{k}|x_{k+1}, x_0) || p_{k|k+1}^{\theta}(x_{k}|x_{k+1})) \right] \Big]
\end{talign}
where $\KL$ is the Kullback--Leibler divergence. The forward kernels $q_{k+1|k}$ are chosen such that $q_{k|0}(x_k | x_0)$ can be computed efficiently in a time independent of $k$. With this, $\theta$ can be efficiently trained by taking a random selection of terms from $\mathcal{L}_{\textup{DT}}$ in each minibatch and performing a stochastic gradient step.

\section{Continuous Time Framework} \label{sec:ContinuousTimeFramework}

\subsection{Forward process and its time reversal}
Our method is built upon a continuous time process from $t=0$ to $t=T$. State transitions can occur at any time during this process as opposed to the discrete time case where transitions only occur when one of the finite number of transition kernels is applied (see Figure \ref{fig:fig1}). This process is known as a Continuous Time Markov Chain (CTMC), we provide a short overview of CTMCs in Appendix \ref{sec:ApdxPrimerCTMC} for completeness. Giving an intuitive introduction here, we can define a CTMC through an initial distribution $q_0$ and a transition rate matrix $R_t \in \mathbb{R}^{S \times S}$. If the current state is $\tilde{x}$, then the transition rate matrix entry $R_t(\tilde{x}, x)$ is the instantaneous rate (occurrences per unit time) at which state $\tilde{x}$ transitions to state $x$. Loosely speaking, the next state in the process will likely be one for which $R_t(\tilde{x}, x)$ is high, and furthermore, the higher the rate is, the less time it will take for this transition to occur.

It turns out that the transition rate, $R_t$, also defines the infinitesimal transition probability for the process between the two time points $t-\dt$ and $t$
\begin{equation}
    \textstyle q_{t | t-\dt}(x | \tilde{x}) = \delta_{x, \tilde{x}} + R_t(\tilde{x}, x) \dt + o(\dt),
\end{equation}
where $o(\dt)$ represents terms that tend to zero at a faster rate than $\dt$.
Comparing to the discrete time case, we see that $R_t$ assumes an analogous role to the discrete time forward kernel $q_{k+1|k}$ in how we define the forward process.
Therefore, just as in discrete time, we design $R_t$ such that: i) the forward process mixes quickly towards an easy to sample (stationary) distribution, $\pref$, (e.g. uniform), ii) we can analytically obtain $q_{t|0}(x_t | x_0)$ distributions to enable efficient training (see Section \ref{sec:ChoiceOfForwardProcess} for how this is done). We initialize the forward CTMC at $q_0(x_0) = \pdata(x_0)$ at time $t=0$. We denote the marginal at time $t=T$ as $q_T(x_T)$, which should be close to $\pref(x_T)$.

We now consider the time reversal of the forward process, which will take us from the marginal $q_T(x_T)$ back to the data distribution $\pdata(x_0)$ through a reverse transition rate matrix, $\hat{R}_t \in \mathbb{R}^{S \times S}$:
\begin{equation}
    \textstyle q_{t | t+\dt}(\tilde{x} | x) = \delta_{\tilde{x}, x} + \hat{R}_t(x, \tilde{x}) \dt + o(\dt).
    \label{eq:reverseKernelDefn}
\end{equation}
In discrete time, one uses Bayes rule to go from $q_{k+1|k}$ to $q_{k|k+1}$. 
We can use similar ideas to calculate $\hat{R}_t$ from $R_t$ as per the following result.
\begin{proposition}
For a forward in time CTMC, $\{ x_t \}_{t \in [0, T]}$, with rate matrix $R_t$, initial distribution $\pdata(x_0)$ and terminal distribution $q_T(x_T)$, there exists a CTMC with initial distribution $q_T(x_T)$ at $t=T$, terminal distribution $\pdata(x_0)$ at $t=0$ and transition rate matrix $\hat{R}_t$ that runs backwards in time and is almost everywhere equivalent to the time reversal of the forward CTMC, $\{ x_t \}_{t \in [T, 0]}$. Furthermore, $\hat{R}_t$ is related to $R_t$ by the following expression
\begin{equation}
    \textstyle \hat{R}_t(x, \tilde{x}) = R_t(\tilde{x}, x)\sum_{x_0}  \frac{q_{t|0}(\tilde{x}| x_0)}{q_{t|0}(x | x_0)} q_{0|t}(x_0 | x) \quad \text{for}\quad x \neq \tilde{x},
\end{equation}
where $q_{t|0}(x|x_0)$ are the conditional marginals of the forward process and $q_{0|t}(x_0 | x) = q_{t|0}(x | x_0) \pdata(x_0)/q_t(x)$ with $q_t(x)$ being the marginal of the forward process at time $t$. When $x = \tilde{x}$, $\hat{R}_t(x, x) = - \sum_{x' \neq x} \hat{R}_t(x, x')$ because the rows must sum to zero (see Appendix \ref{sec:ApdxPrimerCTMC}).
\label{prop:time_reversal}
\end{proposition}
Unfortunately, $\hat{R}_t$ is intractable due to the intractability of $q_t(x)$ and thus of $q_{0|t}(x_0|x)$. Therefore, we consider an approximation $\hat{R}_t^\theta$ of $\hat{R}_t$ by approximating $q_{0|t}(x_0 | x)$ with a parametric denoising model, $p^\theta_{0|t}(x_0 | x)$:
\begin{equation}
    \textstyle \hat{R}_t^\theta(x, \tilde{x}) = R_t(\tilde{x}, x) \sum_{x_0} \frac{q_{t|0}(\tilde{x} | x_0)}{q_{t|0}(x | x_0)} p^\theta_{0|t}(x_0 | x) \quad \text{for} \quad x \neq \tilde{x}
    \label{eq:RthetaRate}
\end{equation}
and $\hat{R}^\theta_t(x, x) = - \sum_{x' \neq x} \hat{R}^\theta_t(x, x')$ as before. As a further analogy to the discrete time case, notice that when $x \neq \tilde{x}$, $\hat{R}_t^\theta$ has the same form as the discrete time parametric reverse kernel, $p^\theta_{k|k+1}$ defined in eq (\ref{eq:discreteTimeParametricKernel}) but with the forward kernel, $q_{k+1|k}$, replaced by the forward rate, $R_t$.

\subsection{Continuous Time ELBO}
In discrete time, $\theta$ is trained by minimizing the discrete time negative ELBO, $\mathcal{L}_{\textup{DT}}$, formed from the forward and reverse processes. We mirror this approach in continuous time by minimizing the corresponding continuous time (CT) negative ELBO, $\mathcal{L}_{\textup{CT}}$, as derived below.
\begin{proposition}
\label{prop:CTELBO}
For the reverse in time CTMC with initial distribution $\pref(x_T)$, terminal distribution $p_0^\theta(x_0)$, and reverse rate $\hat{R}_t^\theta$, an upper bound on the negative model log-likelihood, $\mathbb{E}_{\pdata(x_0)}[ - \log p_0^\theta(x_0)]$, is given by
\begin{equation}
   \textstyle \LCT(\theta) = T \, \mathbb{E}_{t\sim \mathcal{U}(0, T) q_t(x) r_t(\tilde{x}|x)} \Big[ \Big\{ \sum_{x' \neq x} \hat{R}_t^\theta(x, x') \Big\} - \mathcal{Z}^t(x) \log \left(\hat{R}_t^\theta(\tilde{x}, x) \right) \Big] + C,
\end{equation}
where $C$ is a constant independent of $\theta$ and
\begin{equation}
    \textstyle \mathcal{Z}^t(x) = \sum_{x' \neq x} R_t(x, x') \hspace{2cm} r_t(\tilde{x}|x) = (1 - \delta_{\tilde{x}, x} ) R_t(x, \tilde{x})/\mathcal{Z}^t(x).
\end{equation}

\end{proposition}

Here $r_t(\tilde{x} | x)$ gives the probability of transitioning from $x$ to $\tilde{x}$, given that we know a transition occurs at time $t$. We can optimize this objective efficiently with stochastic gradient descent. For a gradient update, we sample a batch of datapoints from $\pdata(x_0)$, noise each datapoint using a random time, $t \sim \mathcal{U}(0, T)$, $x \sim q_{t|0}(x | x_0)$ and finally sample an auxiliary $\tilde{x}$ from $r_t(\tilde{x} | x)$ for each $x$.
Intuitively, ($x$, $\tilde{x}$) are a pair of states following the forward in time noising process. Minimizing the second term in $\LCT$ maximizes the reverse rate for this pair, but going in the backwards direction, $\tilde{x}$ to $x$. This is how $\hat{R}_t^\theta$ learns to reverse the noising process. Intuition on the first term and a direct comparison to $\LDT$ is given in Appendix \ref{sec:ApdxCTDTELBOComparison}.

The first argument of $\smash{\hat{R}_t^\theta}$ is input into $\smash{p_{0|t}^\theta}$ so we naively require two network forward passes on $x$ and $\tilde{x}$ to evaluate the objective. We can avoid this by approximating the $q_t(x)$ sample in the first term with $\tilde{x}$ meaning we need only evaluate the network once on $\tilde{x}$. The approximation is valid because, as we show in Appendix \ref{sec:ApdxCTELBOOneForwardPass}, $\tilde{x}$ is approximately distributed according to $q_{t+\delta t}$ for $\delta t$ very small.

\section{Efficient Forward and Backward Sampling} \label{sec:EfficientForwardAndBackwardSampling}

\subsection{Choice of Forward Process} \label{sec:ChoiceOfForwardProcess}
The transition rate matrix $R_t$ needs to be chosen such that the forward process: i) mixes quickly towards $\pref$, and ii) the $q_{t|0}(x | x_0)$ distributions can be analytically obtained. The Kolmogorov differential equation for the CTMC needs to be integrated to obtain $q_{t|0}(x | x_0)$. This can be done analytically when $R_t$ and $R_{t'}$ commute for all $t, t'$, see Appendix \ref{sec:ApdxChoiceOfForwardProcess}. An easy way to meet this condition is to let $R_t = \beta(t) R_b$ where $R_b \in \mathbb{R}^{S \times S}$ is a user-specified time independent base rate matrix and $\beta(t) \in \mathbb{R}$ is a time dependent scalar. We then obtain the analytic expression
\begin{equation}
    \textstyle q_{t|0}(x = j | x_0 = i) = \left( Q \text{exp} \left[ \Lambda \int_0^t \beta(s) ds \right] Q^{-1} \right)_{ij}
\end{equation}
where $R_b = Q \Lambda Q^{-1}$ is the eigendecomposition of matrix $R_b$ and $\text{exp} [\cdot]$ the element-wise exponential.

Our choice of $\beta$ schedule is guided by \cite{ho2020denoising, song2020score}, $\beta(t) = a b^t \log(b)$.
The hyperparameters $a$ and $b$ are selected such that $\smash{q_T(x) \approx \pref(x)}$ at the terminal time $t=T$ while having a steady speed of `information corruption' which ensures that $\smash{\hat{R}_t}$ does not vary quickly in a short span of time.

We experiment with a variety of $R_b$ matrices, for example, a uniform rate, $R_b = \mathbbm{1} \mathbbm{1}^T - S \mathrm{Id}$, where $\mathbbm{1} \mathbbm{1}^T$ is a matrix of ones and $\mathrm{Id}$ is the identity. For problems with a heavy spatial bias, e.g. images, we can instead use a forward rate that only encourages transitions to nearby states; details and the links to the corresponding discrete time processes can be found in Appendix \ref{sec:ApdxChoiceOfForwardProcess}.

\subsection{Factorizing Over Dimensions} \label{sec:FactorizeDim}

Our aim is to model data that is $D$ dimensional, with each dimension taking one value from $S$ possibilities. We now slightly redefine notation and say $\bm{x}^{1:D} \in \mathcal{X}^D$, $|\mathcal{X}| = S$. In this setting, calculating transition probabilities naively would require calculating $S^D$ rate values corresponding to each of the possible next states. This is intractable for any reasonably sized $S$ and $D$. We avoid this problem simply by factorizing the forward process such that each dimension propagates independently. Since this is a continuous time process and each dimension's forward process is independent of the others, the probability two or more dimensions transition at exactly the same time is zero. Therefore, overall in the full dimensional forward CTMC, each transition only ever involves a change in exactly one dimension. For the time reversal CTMC, it will also be true that exactly one dimension changes in each transition. This makes computation tractable because of the $S^D$ rate values, only $D \times (S-1) + 1$ are non-zero - those corresponding to transitions where exactly one dimension changes plus the no change transition. Finally, we note that even though dimensions propagate independently in the forward direction, they are not independent in the reverse direction because the starting points for each dimension's forward process are not independent for non factorized $\pdata$.
The following proposition shows the exact forms for the forward and reverse rates in this case.

\begin{proposition}
\label{prop:dimfactorize}
If the forward process factorizes as $q_{t|s}(\bm{x}_t^{1:D} | \bm{x}_s^{1:D}) = \prod_{d=1}^D q_{t|s}(x_t^d | x_s^d)$, $t>s$, then the forward and reverse rates are of the form
\begin{gather}
    \textstyle R_t^{1:D}(\tilde{\bm{x}}^{1:D}, \bm{x}^{1:D}) = \sum_{d=1}^D R_t^d(\tilde{x}^d, x^d) \delta_{\bm{x}^{\Dd}, \tilde{\bm{x}}^{\Dd}},\\
    \textstyle \hat{R}_t^{1:D}(\bm{x}^{1:D}, \tilde{\bm{x}}^{1:D}) = \sum_{d=1}^D R_t^d(\tilde{x}^d, x^d) \delta_{\bm{x}^{\Dd}, \tilde{\bm{x}}^{\Dd}} \sum_{x_0^d} q_{0|t}(x_0^d | \bm{x}^{1:D}) \frac{q_{t|0}(\tilde{x}^d | x_0^d)}{q_{t|0}(x^d | x_0^d)},
\end{gather}
where $R_t^d \in \mathbb{R}^{S \times S}$ and $\delta_{\bm{x}^{\Dd}, \tilde{\bm{x}}^{\Dd}}$ is 1 when all dimensions except for $d$ are equal.
\end{proposition}

To find $\hat{R}_t^{\theta \, 1:D}$ we simply replace $q_{0|t}(x_0^d | \bm{x}^{1:D})$ with $p_{0|t}^\theta(x_0^d | \bm{x}^{1:D})$ which is easily modeled with a neural network that outputs conditionally independent state probabilities in each dimension. In Appendix \ref{sec:ApdxCTELBOFactorization} we derive the form of $\mathcal{L}_{\textup{CT}}$ when we use this factorized form for $R_t^{1:D}$ and $\hat{R}_t^{\theta \, 1:D}$.

\subsection{Simulating the Generative Reverse Process with Tau-Leaping} \label{sec:TauLeaping}

The parametric generative reverse process is a CTMC with rate matrix $\hat{R}_t^{\theta \, 1:D}$. Simulating this process from distribution $\pref(\bm{x}_T^{1:D})$ at time $t=T$ back to $t=0$ will produce approximate samples from $\pdata(\bm{x}_0^{1:D})$. The process could be simulated exactly using Gillespie's Algorithm \cite{gillespie1976general, gillespie1977exact, wilkinson2018stochastic} which alternates between i) sampling a holding time to remain in the current state and ii) sampling a new state according to the current rate matrix, $\hat{R}_t^{\theta \, 1:D}$ (see Appendix \ref{sec:ApdxCTMCSimulation}). This is inefficient for large $D$ because we would need to step through each transition individually and so only one dimension would change for each simulation step.

Instead, we use tau-leaping \cite{gillespie2001approximate,wilkinson2018stochastic}, a very popular approximate simulation method developed in chemical physics. Rather than step back through time one transition to the next, tau-leaping leaps from $t$ to $t-\tau$ and applies all transitions that occurred in $[t-\tau, t]$ simultaneously. To make a leap, we assume $\hat{R}_t^{\theta \, 1:D}$ and $\bm{x}_t^{1:D}$ remain constant in $[t-\tau, t]$. As we propagate from $t$ to $t-\tau$, we count all of the transitions that occur, but hold off on actually applying them until we reach $t-\tau$, such that $\bm{x}_t^{1:D}$ remains constant in $[t-\tau, t]$. Assuming $\hat{R}_t^{\theta \, 1:D}$ and $\bm{x}_t^{1:D}$ remain constant, the number of times a transition from $\bm{x}_t^{1:D}$ to $\tilde{\bm{x}}^{1:D}$ occurs in $[t-\tau, t]$ is Poisson distributed with mean $\tau \hat{R}_t^{\theta \, 1:D}(\bm{x}_t^{1:D}, \tilde{\bm{x}}^{1:D})$. Once we reach $t-\tau$, we apply all transitions that occurred simultaneously i.e. $\bm{x}_{t-\tau}^{1:D} = \bm{x}_t^{1:D} + \sum_{i} P_i (\tilde{\bm{x}}_i^{1:D} - \bm{x}_t^{1:D})$ where $P_i$ is a Poisson random variable with mean $\tau \hat{R}_t^{\theta \, 1:D}(\bm{x}_t^{1:D}, \tilde{\bm{x}}_i^{1:D})$. Note the sum assumes a mapping from $\mathcal{X}$ to $\mathbb{Z}$.

\begin{wrapfigure}{l}{4cm}
\vspace{0.1cm}
\includegraphics[width=3.5cm]{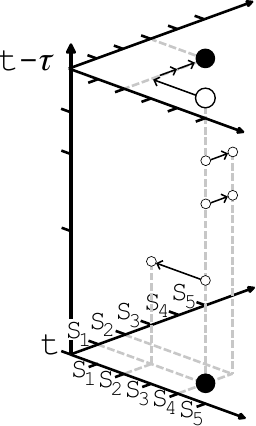}
\vspace{0.28cm}
\caption{3D visualization of one tau-leaping step from $x_t^{1:2} = \{ S_4, S_1 \}$ to $x_{t-\tau}^{1:2} = \{ S_2, S_3 \}$. Here, $D=2$, $|\mathcal{X}|=5$, $P_{12}=1$, $P_{22}=2$, all other $P_{ds}=0$.}
\label{fig:tauleaping}
\vspace{-0.8cm}
\end{wrapfigure}

Using our knowledge of $\hat{R}_t^{\theta \, 1:D}$, we can further unpack this update. Namely, $\hat{R}_t^{\theta \, 1:D}(\bm{x}_t^{1:D}, \tilde{\bm{x}}^{1:D})$ can only be non-zero when $\tilde{\bm{x}}^{1:D}$ has a different value to $\bm{x}_t^{1:D}$ in exactly one dimension (rates for multi-dimensional changes are zero). Explicitly summing over these options we get $\bm{x}_{t-\tau}^{1:D} = \bm{x}_t^{1:D} + \sum_{d=1}^D \sum_{s=1 \backslash x_t^d}^{S} P_{ds} (s - x_t^d) \bm{e}^d$ where $\bm{e}^d$ is a one-hot vector with a $1$ at dimension $d$ and $P_{ds}$ is a Poisson random variable with mean $\tau \hat{R}_t^{\theta \, 1:D}(\bm{x}_t^{1:D}, \bm{x}_t^{1:D} + (s - x_t^d) \bm{e}^d)$. Since multiple $P_{ds}$ can be non-zero, we see that tau-leaping allows $\bm{x}_t^{1:D}$ to change in multiple dimensions in a single step. Figure \ref{fig:tauleaping} visualizes this idea. During the $[t-\tau, t]$ interval, one jump occurs in dimension $1$ and two jumps occur in dimension $2$. These are all applied simultaneously once we reach $t-\tau$. When our discrete data has ordinal structure (e.g. Section \ref{sec:ImageModeling}) our mapping to $\mathbb{Z}$ is not arbitrary and making multiple jumps within the same dimension ($\sum_{s=1 \backslash x_t^d}^S P_{ds} > 1$) is meaningful. In the non-ordinal/categorical case (e.g. Section \ref{sec:MonophonicMusic}) the mapping to $\mathbb{Z}$ is arbitrary and so, although taking simultaneous jumps in different dimensions is meaningful, taking multiple jumps within the same dimension is not. For this type of data, we reject changes to $x_t^d$ for any $d$ for which $\smash{\sum_{s=1 \backslash x_t^d}^S P_{ds} > 1}$. In practice, the rejection rate is very small when $\smash{R_t^{1:D}}$ is suitable for categorical data (e.g. uniform), see Appendix \ref{sec:ApdxMusicExperiment}. In Section \ref{sec:ErrorBound}, our error bound accounts for this low probability of rejection and also the low probability of an out of bounds jump that we observe in practice in the ordinal case.

The tau-leaping approximation improves with smaller $\tau$, recovering exact simulation in the limit as $\tau \rightarrow 0$. Exact simulation is similar to an autoregressive model in that only one dimension changes per step. Increasing $\tau$ and thus the average number of dimensions changing per step gives us a natural way to modulate the `autoregressiveness' of the model and trade sample quality with compute (Figure \ref{fig:cifar10plots} right). We refer to our method of using tau-leaping to simulate the reverse CTMC as $\tau$LDR (tau-leaping denoising reversal) which we formalize in Algorithm \ref{alg:tauleaping} in Appendix \ref{sec:ApdxCTMCSimulation}.

We note that theoretically, one could approximate $\hat{R}_t^{\theta \, 1:D}$ as constant in the interval $[t-\tau, t]$, and construct a transition probability matrix by solving the forward Kolmogorov equation with the matrix exponential $\smash{P_{t-\tau | t} \approx \text{exp} (\tau \hat{R}_t^{\theta \, 1:D} )}$. However, for the learned $\smash{\hat{R}_t^{\theta \, 1:D} \in \mathbb{R}^{S^D \times S^D}}$ matrix, it is intractable to compute this matrix exponential so we use tau-leaping for sampling instead.

\subsection{Predictor-Corrector}
During approximate reverse sampling, we aim for the marginal distribution of samples at time $t$ to be close to $q_t(x_t)$ (the marginal at time $t$ of the true CTMC). The continuous time framework allows us to exploit additional information to more accurately follow the reverse progression of marginals, $\{q_t(x_t)\}_{t \in [T, 0]}$ and improve sample quality. Namely, after a tau-leaping `predictor' step using rate $\hat{R}_t^{\theta}$, we can apply `corrector' steps with rate $R_t^c$ which has $q_t(x_t)$ as its stationary distribution. The corrector steps bring the distribution of samples at time $t$ closer to the desired $q_t(x_t)$ marginal. $R_t^c$ is easy to calculate as stated below

\begin{proposition} \label{prop:corrector_rate}
For a forward CTMC with marginals $\{ q_t(x_t) \}_{t \in [0, T]}$, forward rate, $R_t$, and corresponding reverse CTMC with rate $\hat{R}_t$, the rate $R_t^c= R_t + \hat{R}_t$ has $q_t(x_t)$ as its stationary distribution.
\end{proposition}
In practice, we approximate $R_t^c$ by replacing $\hat{R}_t$ with $\hat{R}_t^\theta$. This is directly analogous to Predictor-Corrector samplers in continuous state spaces \cite{song2020score} that predict by integrating the reverse SDE and correct with score-based Markov chain Monte Carlo steps, see Appendix \ref{sec:ApdxPredictorCorrectorDiscussion} for further discussion.

\subsection{Error Bound} \label{sec:ErrorBound}
Our continuous time framework also allows us to provide a novel theoretical bound on the error between the true data distribution and the sample distribution generated via tau-leaping (without predictor-corrector steps), in terms of the error in our approximation of the reverse rate and the mixing of the forward noising process.

We assume we have a time-homogeneous rate matrix $R_t$ on $\mathcal{X}$, from which we construct the factorized rate matrix $R_t^{1:D}$ on $\mathcal{X}^D$ by setting $R^d_t = R_t$ for each $d$. Note that by rescaling time by a factor of $\beta(t)$ we can transform our choice of rate from Section \ref{sec:ChoiceOfForwardProcess} to be time-homogeneous. We will denote $|R| = \sup_{t \in [0,T], x \in \mathcal{X}} |R_t(x,x)|$, and let $t_{\textup{mix}}$ be the (1/4)-mixing time of the CTMC with rate $R_t$ (see \cite[Chapter 4.5]{levin2009markovchains}).

\begin{theorem}
\label{thm:error_bound}
For any $D \geq 1$ and distribution $\pdata$ on $\mathcal{X}^D$, let $\{x_t\}_{t \in [0,T]}$ be a CTMC starting in $\pdata$ with rate matrix $R^{1:D}_t$ as above. Suppose that $\hat{R}_t^{\theta \, 1:D}$ is an approximation to the reverse rate matrix and let $(y_k)_{k = 0, 1, \dots, N}$ be a tau-leaping approximation to the reverse dynamics with maximum step size $\tau$. Suppose further that there is some constant $M > 0$ independent of $D$ such that
\begin{equation}
\label{eq:error_bound_condition}
    \sum_{y \neq x} \left| \hat R_t^{1:D}(x,y) - \hat{R}_t^{\theta \, 1:D}(x,y) \right| \leq M
\end{equation}
for all $t \in [0,T]$. Then under the assumptions in Appendix \ref{sec:ApdxErrorBound}, there are constants $C_1, C_2 > 0$ depending on $\mathcal{X}$ and $R_t$ but not $D$ such that, if $\mathcal{L}(y_0)$ denotes the law of $y_0$, we have the total variation bound
\begin{equation}
\textstyle
    ||\mathcal{L}(y_0) - \pdata||_{\textup{TV}} \leq 3MT + \left\{ \big(|R|SDC_1\big)^2 + \frac{1}{2} C_2(M + C_1 SD|R|) \right\} \tau T + 2 \exp \left \{ - \frac{T \log^2 2}{t_{\textup{mix}} \log 4D} \right \}
\end{equation}
\end{theorem}

The first term of the above bound captures the error introduced by our approximation of the reverse rate $\hat R_t^{1:D}$ with $\hat{R}_t^{\theta \, 1:D}$. The second term reflects the error introduced by the tau-leaping approximation, and is linear in both $T$ and $\tau$, showing that as we take our tau-leaping steps to be arbitrarily small, the error introduced by tau-leaping goes to zero. The final term describes the mixing of the forward chain, and captures the error introduced since $p_{\textup{ref}}$ and $q_T$ are not exactly equal.

We choose to make the dependence of the bound on the dimension $D$ explicit, since we are specifically interested in applying tau-leaping to high dimensional problems where we make transitions in different dimensions simultaneously in a single time step. The bound grows at worst quadratically in the dimension, versus e.g. exponentially.
The bound is therefore useful in showing us that we do not need to make $\tau$ impractically small in high dimensions. Other than gaining these intuitions, we do not expect the bound to be particularly tight in practice and further it would not be practical to compute because of the difficulty in finding $M$, $C_1$ and $C_2$.

The assumptions listed in Appendix \ref{sec:ApdxErrorBound} hold approximately for tau-leaping in practice when we use spatially biased rates for ordinal data such that jump sizes are small or uniform rates for non-ordinal data such that the dimensional rejection rate is small. These assumptions could be weakened, however, Theorem \ref{thm:error_bound} would become much more involved, obscuring the intuition and structure of the problem.

\section{Related Work} \label{sec:RelatedWork}

The application of denoising models to discrete data was first described in \cite{sohl2015deep} using a binomial diffusion process for a binary dataset. Each reverse kernel $\smash{p^\theta_{k|k+1}}$ was directly parameterized without using a denoising model $\smash{p_{0|k}^\theta}$. In \cite{song2020denoising} an approach for discrete categorical data was suggested using a uniform forward noising kernel, $q_{k+1|k}$, and a reverse kernel parameterized through a denoising model, though no experiments were performed with the approach. Experiments on text and segmentation maps were then performed with a similar model in \cite{hoogeboom2021argmax}. Other forward kernels were introduced in \cite{austin2021structured} that are more appropriate for certain data types such as the spatially biased Gaussian kernel. \cite{esser2021imagebart, gu2021vector} apply the approach to discrete latent space modeling using uniform and absorbing state forward kernels. Whilst a link to continuous time for the forward process is mentioned in \cite{austin2021structured}, all of these approaches train and sample in discrete time. We show in Appendix \ref{sec:ApdxDTDimAssumptions} that this involves making an implicit approximation for multi-dimensional data. We extend this line of work by training and sampling in continuous time.

Other works also operate in discrete space but less rigidly follow the diffusion framework. A corruption process tailored to text is proposed in \cite{johnson2021beyond}, whereby token deletion and insertion is also incorporated. \cite{savinov2021step} also focus on text, creating a generative reverse chain that repeatedly applies the same denoising kernel. The corruption distribution is also defined through the same denoising kernel to reduce distribution shift between training and sampling. In \cite{hoogeboom2021autoregressive}, a more standard masking based forward process is used but the reversal is interpreted from an order agnostic autoregressive perspective. They also describe how their model can be interpreted as the reversal of a continuous time absorbing state diffusion but do not utilize this perspective in training or sampling.
\cite{goyal2017variational} propose a denoising type framework that can be used on binary data where the forward and reverse process share the same transition kernel. Finally, in \cite{cohen2022diffusion}, the discrete latent space of a VQVAE is modeled by quantizing an underlying continuous state space diffusion with probabilistic quantization functions.

\section{Experiments} \label{sec:Experiments}
\subsection{Demonstrative Example}
\captionsetup[subfigure]{labelformat=empty}
\begin{figure}
\centering
\subfloat[]{
    \includegraphics[trim=50 0 0 0, width=0.5\textwidth]{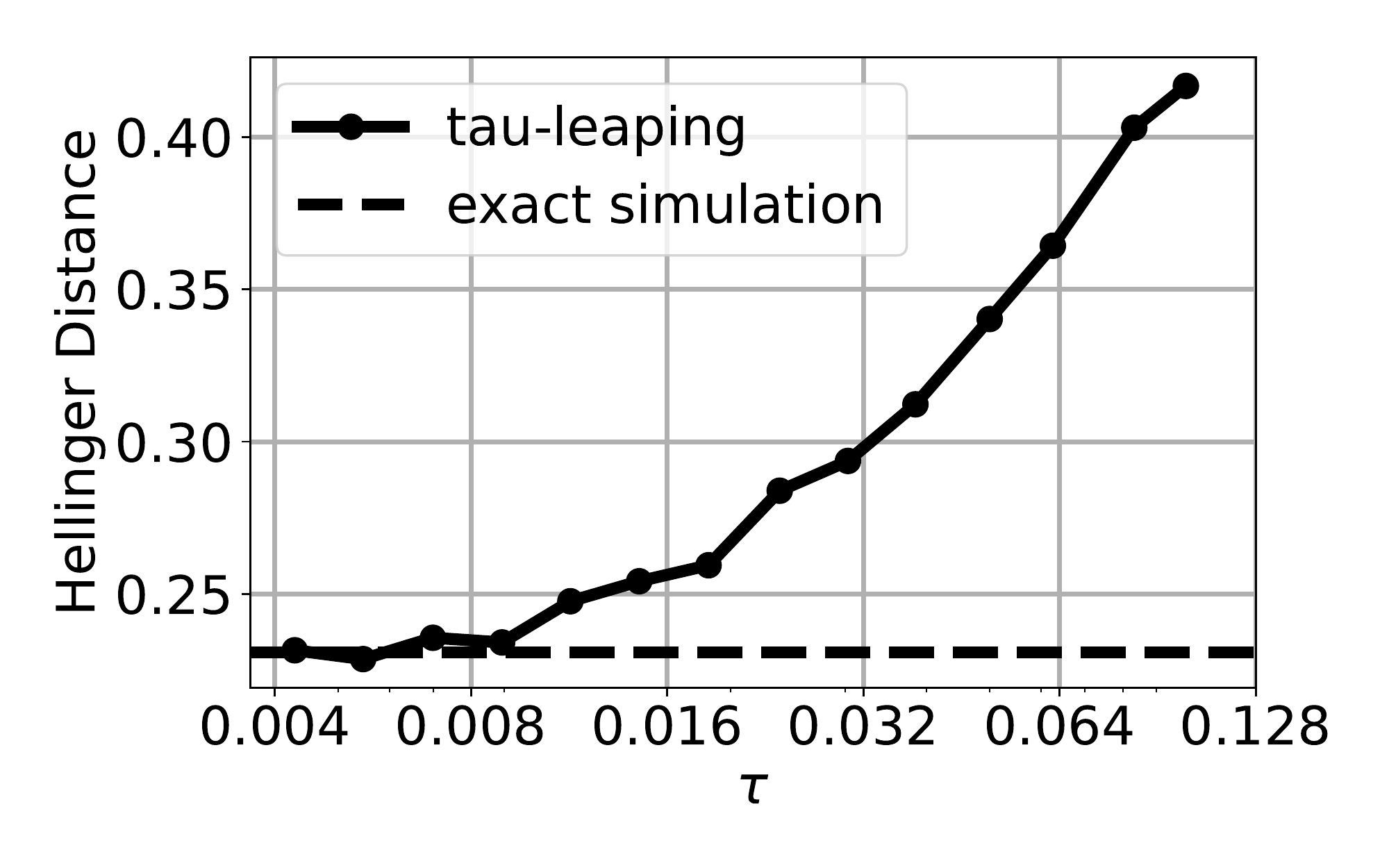}
}
\subfloat[]{
    \includegraphics[trim=30 -40 30 0, width=0.4\textwidth]{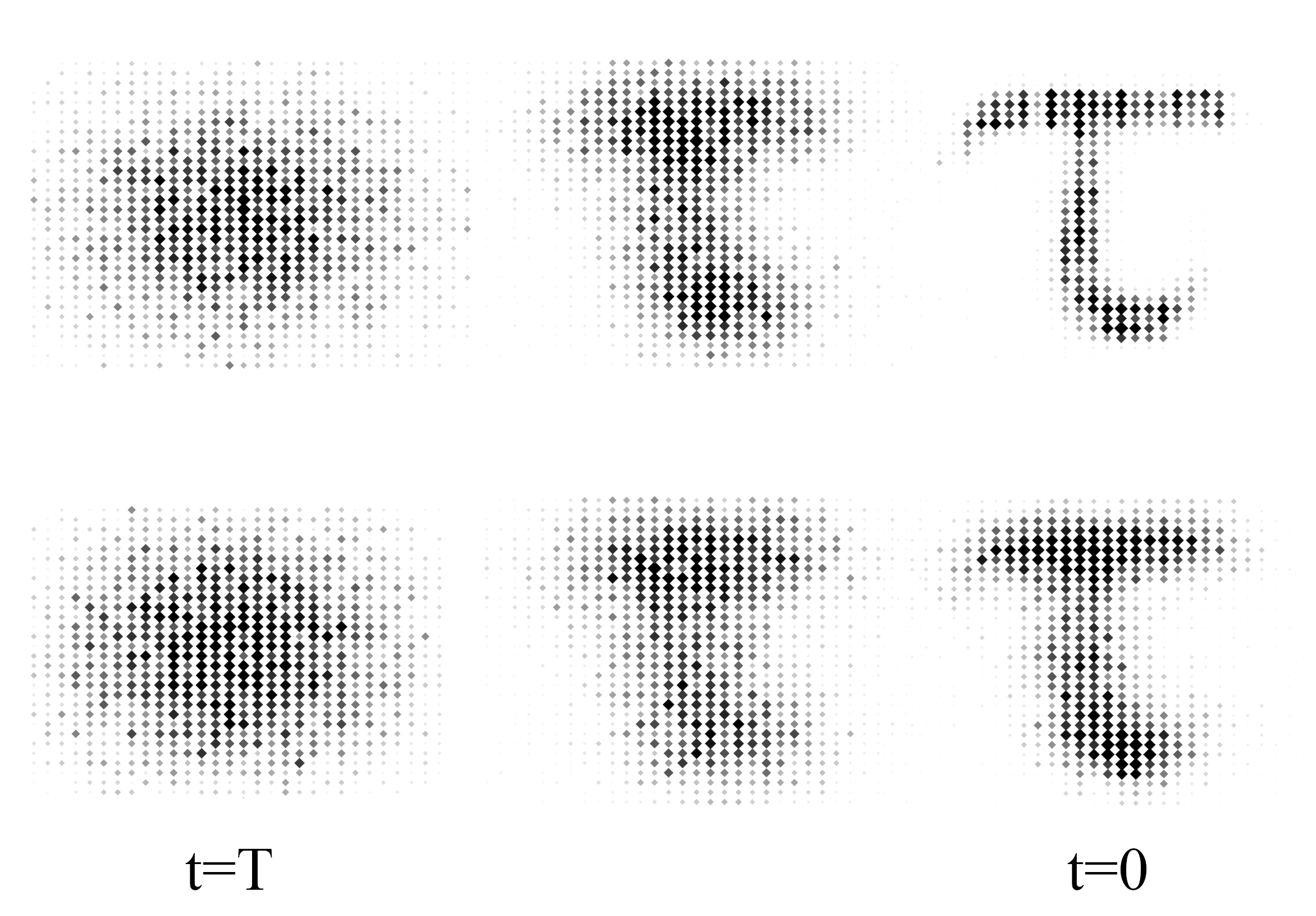}
}
\vspace{-1.0cm}
\caption{ \textbf{Left:} Hellinger distance between the true training distribution and generated sample distributions with exact simulation or tau-leaping. With $\tau$ small, we simulate the reverse CTMC with the same fidelity as the exact simulation. \textbf{Top Right:} Histograms of the marginals during the reverse generative process simulated using tau-leaping with $\tau=0.004$. Darker and larger diamonds represent increased density. \textbf{Bottom Right:} The same for $\tau=0.1$, note the reduced sample quality.}
\label{fig:2d_figures}
\vspace{-0.3cm}
\end{figure}

We first verify the method can accurately produce samples from the entire support of the data distribution and that tau-leaping can accurately simulate the reverse CTMC. To do this, we create a dataset formed of 2d samples of a state space of 32 arranged such that the histogram of the training dataset forms a `$\tau$' shape. We train a denoising model using the $\mathcal{L}_{\textup{CT}}$ objective with $\smash{p_{0|t}^\theta}$ parameterized through a residual MLP (full details in Appendix \ref{sec:ApdxToyExperiment}). We then sample the parameterized reverse process using an exact method (up to needing to numerically integrate the reverse rate) and tau-leaping. Figure \ref{fig:2d_figures} top-right shows the marginals during reverse simulation with $\tau=0.004$ and we indeed produce samples from the entire support of $\pdata$. Furthermore, we find that with sufficiently small $\tau$, we can match the fidelity of exact simulation of the reverse CTMC (Figure \ref{fig:2d_figures} left).
The value of $\tau$ dictates the number of network evaluations in the reverse process according to $\text{NFE} = T/\tau$. In all experiments we use $T=1$. Exact simulation results in a non zero Hellinger distance between the generated and training distributions because of imperfections in the learned $\smash{\hat{R}_t^\theta}$ model.

\subsection{Image Modeling} \label{sec:ImageModeling}

\begin{table}
  \caption{Sample quality metrics and model likelihoods for diffusion methods modeling CIFAR10 in discrete state space. Diffusion methods modeling CIFAR10 in continuous space are included for reference. The Inception Score (IS) and Fréchet Inception Distance (FID) are calculated using 50000 generated samples with respect to the training dataset as is standard practice. The ELBO values are reported on the test set in bits per dimension.}
  \label{tab:fids}
  \centering
  \begin{tabular}{lllll}
    \toprule
    & Method     & IS $(\uparrow)$ & FID $(\downarrow)$ & ELBO $(\uparrow)$  \\
    \midrule
    Discrete state & D3PM Absorbing \cite{austin2021structured} & $6.78 $& $30.97$ & $- 4.40$ \\
    & D3PM Gauss \cite{austin2021structured}    & $8.56$ & $7.34$ & $\mathbf{- 3.44}$  \\
    & $\tau$LDR-$0$ (ours) & $8.74$ & $8.10$ & $-3.59$ \\
    & $\tau$LDR-$10$ (ours) & $\mathbf{9.49}$ & $\mathbf{3.74}$ & $-3.59$ \\
    \midrule
    Continuous state & DDPM \cite{ho2020denoising} & $9.46$ & $3.17$ & $- 3.75$\\
    & NCSN \cite{song2020score} & $9.89$ & $2.20$ & - \\
    \bottomrule
  \end{tabular}
\vspace{-0.5cm}
\end{table}

\begin{figure}
\centering
\subfloat[]{
    \includegraphics[trim= 750 0 500 0, width=0.5\textwidth]{figures/cifar_samples_4by8.pdf}
}
\subfloat[]{
  \includegraphics[trim=20 0 0 0, width=0.5\textwidth]{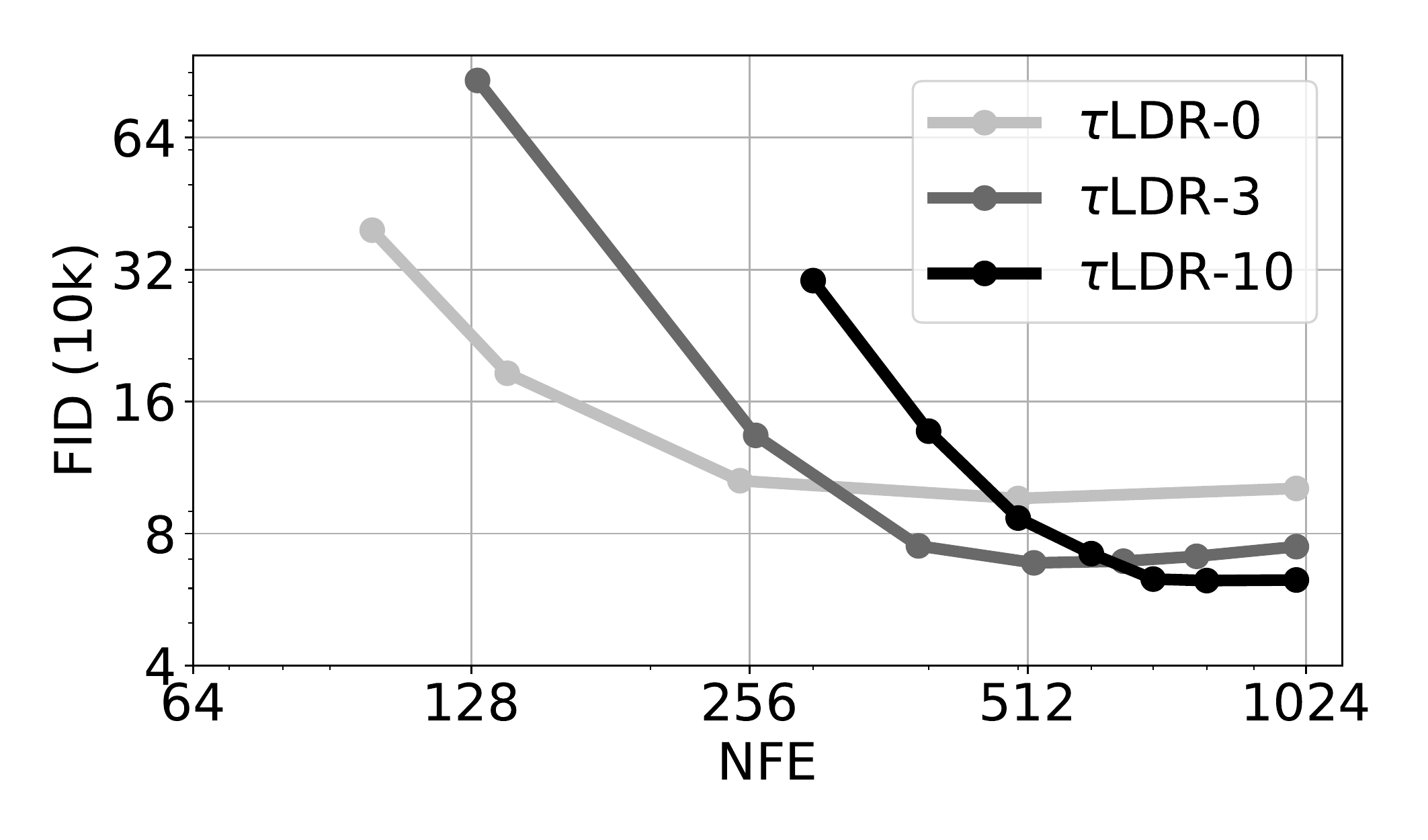}
}

\vspace{-0.9cm}
    \caption{\textbf{Left:} Unconditional CIFAR10 samples from our $\tau$LDR-10 model \textbf{Right:} FID scores for the generated CIFAR10 samples versus number of $p_{0|t}^\theta$ evaluations during sampling (variation induced by varying $\tau$). Calculated with 10k samples, hence the discrepancy with Table \ref{tab:fids} \cite{chong2020effectively}.}
\label{fig:cifar10plots}
\vspace{-0.5cm}
\end{figure}

We now demonstrate that our continuous time framework gives us improved generative modeling performance versus operating in discrete time. We show this on the CIFAR-10 image dataset. Images are typically stored as discrete data, each pixel channel taking one value from 256 possibilities. Continuous state space methods have to somehow get around this fact by, for example, adding a discretization function at the end of the generative process \cite{ho2020denoising} or adding uniform noise to the data. Here, we model the images directly in discrete space. We parameterize $p_{0|t}^\theta$ using the standard U-net architecture \cite{ho2020denoising} with the modifications for discrete state space suggested by \cite{austin2021structured}. We use a spatially biased rate matrix and train with an augmented $\LCT$ loss including direct $p_{0|t}^\theta$ supervision, full experimental details are in Appendix \ref{sec:ApdxCIFAR10Experiment}.

Figure \ref{fig:cifar10plots} left shows randomly generated unconditional CIFAR10 samples from the model and we report sample quality metrics in Table \ref{tab:fids}. We see that our method ($\tau$LDR-$0$) with $0$ corrector steps has better Inception Score but worse FID than the D3PM discrete time method. However, our $\tau$LDR-$10$ method with $10$ corrector steps per predictor step at the end of the reverse sampling process ($t < 0.1T$) greatly improves sample quality, beating the discrete time method in both metrics and further closes the performance gap with methods modeling images as continuous data. The derivation of the corrector rate which gave us this improved performance required our continuous time framework. D3PM achieves the highest ELBO but we note that this does not correlate well with sample quality. In Table \ref{tab:fids}, $\tau$ was adjusted such that both $\tau$LDR-$0$ and $\tau$LDR-$10$ used 1000 $\smash{p_{0|t}^\theta}$ evaluations in the reverse sampling procedure. We show how FID score varies with number of $\smash{p_{0|t}^\theta}$ evaluations for $\tau$LDR-$\{0,3,10\}$ in Figure \ref{fig:cifar10plots} right. The optimum number of corrector steps depends on the sampling budget, with lower numbers of corrector steps being optimal for tighter budgets. This is due to the increased $\tau$ required to maintain a fixed budget when we use a larger number of corrector steps.

\subsection{Monophonic Music} \label{sec:MonophonicMusic}
In this experiment, we demonstrate our continuous time model improves generation quality on non-ordinal/categorical discrete data. We model songs from the Lakh pianoroll dataset \cite{raffel2016learning, dong2018musegan}. We select all monophonic sequences from the dataset such that at each of the 256 time steps either one from 128 notes is played or it is a rest. Therefore, our data has state space size $S=129$ and dimension $D=256$. We scramble the ordering of the state space when mapping to $\mathbb{Z}$ to destroy any ordinal structure. We parameterize $p_{0|t}^\theta$ with a transformer architecture \cite{mittal2021symbolic} and train using a conditional form of $\LCT$ targeting the conditional distribution of the final 14 bars (224 time steps) given the first 2 bars of the song. We use a uniform forward rate matrix, $R_t$, full experimental details are given in Appendix \ref{sec:ApdxMusicExperiment}. Conditional completions of unseen test songs are shown in Figure \ref{fig:conditional_music}.
The model is able to faithfully complete the piece in the same style as the conditioning bars.

We quantify sample quality in Table \ref{tab:piano_metrics}.
We use two metrics: the Hellinger distance between the histograms of generated and ground truth notes and the proportion of outlier notes in the generations but not in the ground truth. 
Using our method, we compare between a birth/death and uniform forward rate matrix $R_t$. 
The birth/death rate is only non-zero for adjacent states whereas the uniform rate allows transitions between arbitrary states which is more appropriate for the categorical case thus giving improved sample quality.
Adding 2 corrector steps per predictor step further improves sample quality.
We also compare to the discrete time method D3PM \cite{austin2021structured} with its most suitable corruption process for categorical data. We find it performs worse than our continuous time method.

\begin{table}
  \caption{Metrics comparing generated conditional samples and ground truth completions. We compute these over the test set showing mean$\pm$std with respect to 5 samples for each test song.}
  \label{tab:piano_metrics}
  \centering
  \begin{tabular}{lll}
    \toprule
    Model  & Hellinger Distance & Proportion of Outliers \\
    \midrule
    $\tau$LDR-0 Birth/Death & $0.3928 \pm 0.0010$ & $0.1316 \pm 0.0012$\\
    $\tau$LDR-0 Uniform & $0.3765 \pm 0.0013$ & $0.1106 \pm 0.0010$ \\
    $\tau$LDR-2 Uniform & $\mathbf{0.3762 \pm 0.0015}$ & $\mathbf{0.1091 \pm 0.0014}$\\
    \midrule
    D3PM Uniform \cite{austin2021structured} & $0.3839 \pm 0.0002$ & $0.1137 \pm 0.0010$\\
    \bottomrule
  \end{tabular}
  \vspace{-0.3cm}
\end{table}
\begin{figure}
    \centering
    \includegraphics[width=14cm]{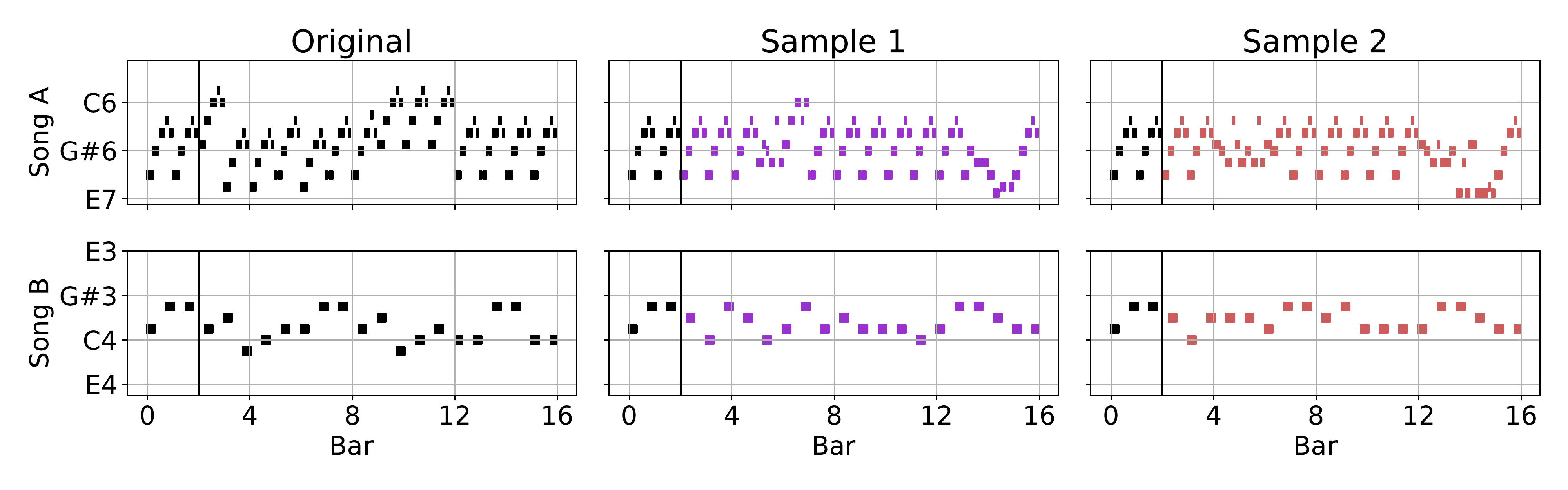}
    \vspace{-0.7cm}
    \caption{Conditional completions of an unseen music sequence. The conditioning 2 bars are shown to the left of the black line. More examples and audio recordings are linked in Appendix \ref{sec:ApdxMusicExperiment}.}
    \label{fig:conditional_music}
\end{figure}

\section{Discussion} \label{sec:Discussion}
We have presented a continuous time framework for discrete denoising models. We showed how to efficiently sample the generative process with tau-leaping and provided a bound on the error of the generated samples. On discrete data problems, we found our predictor-corrector sampler improved sample quality versus discrete time methods.
Regarding limitations, our model requires many model evaluations to produce a sample. Our work has opened the door to applying the work improving sampling speed on continuous data \cite{jolicoeur2021gotta, zhang2022fast, salimans2022progressive, chung2021come, bao2022analytic} to discrete data problems too. Modeling performance on images is also slightly behind continuous state space models, we hope this gap is further closed with bespoke discrete state architectures and corruption process tuning. Finally, we note that the ELBO values for the discrete time model on CIFAR10 are better than for our method. In this work, we focused on sample quality rather than using our model to give data likelihoods e.g. for compression downstream tasks.

\section*{Acknowledgements}
Andrew Campbell and Joe Benton acknowledge support from the EPSRC CDT in Modern Statistics and Statistical Machine Learning (EP/S023151/1). Arnaud Doucet is partly supported by the EPSRC grant EP/R034710/1. He also acknowledges support of the UK Defence Science and Technology Laboratory (DSTL) and EPSRC under grant EP/R013616/1. This is part of the collaboration between US DOD, UK MOD and UK EPSRC under the Multidisciplinary University Research Initiative. This project made use of time on Tier 2 HPC facility JADE2, funded by EPSRC (EP/T022205/1).

\bibliographystyle{unsrt}
\bibliography{neurips_2022}

\newpage
\appendix

{\huge \bfseries Appendix}
\etocdepthtag.toc{mtappendix}
\etocsettagdepth{mtchapter}{none}
\etocsettagdepth{mtappendix}{subsection}
\tableofcontents

The appendix is organized as follows. In Section \ref{sec:ApdxPrimerCTMC}, we provide a short introduction to Continuous Time Markov Chains, including the relevant results we use in this work. Proofs for all the Propositions and Theorems from the main text are in Section \ref{sec:ApdxProofs}. We then describe in Section \ref{sec:ApdxCTELBODetails} some additional intuitions and forms of our proposed objective, $\LCT$. In Section \ref{sec:ApdxDirectDenoisingSupervision}, we describe how an additional direct denoising model supervision term can be added to the objective to improve empirical performance. Details for how we define the forward process in our model can be found in Section \ref{sec:ApdxChoiceOfForwardProcess}. Section \ref{sec:ApdxCTMCSimulation} describes in more detail how CTMCs can be simulated and includes the algorithmic description of tau-leaping. We argue in Section \ref{sec:ApdxDTDimAssumptions} that operating in discrete time forces an implicit assumption when using a factorized forward process on multi-dimensional data. Full experimental details for all investigations can be found in Section \ref{sec:ApdxExperimentDetails} as well as additional plots and results from our models. Finally, in Section \ref{sec:ApdxEthicalConsiderations}, we consider the social impacts of our research.

\section{Primer on Continuous Time Markov Chains} \label{sec:ApdxPrimerCTMC}

\begin{figure}
    \centering
    \includegraphics[width=0.8\textwidth]{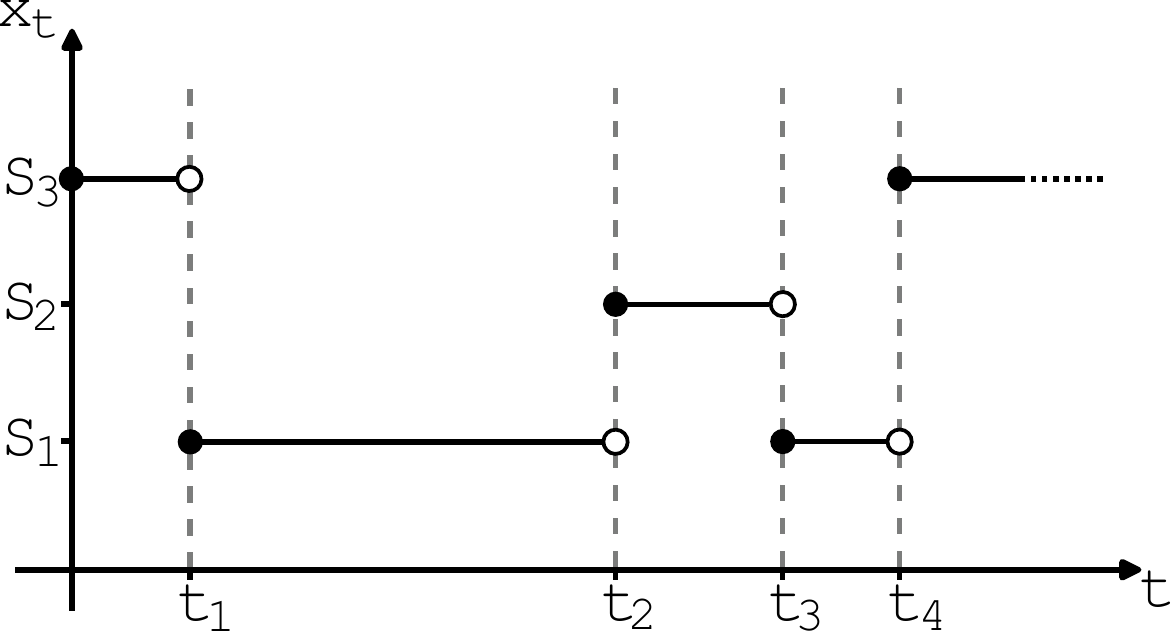}
    \caption{Schematic representation of a 1-dimensional CTMC with 3 states.}
    \label{fig:apdxCTMCSchematic}
\end{figure}

A Continuous Time Markov Chain (CTMC) is a right continuous stochastic process
$\{ x_t \}_{t \in [0, T]}$ satisfying the Markov property, with $x_t$ taking
values in a discrete state space $\mathcal{X}$.  Since the CTMC is Markov,
future behaviour of the process depends only on the current state and not the
history. A schematic representation of a CTMC path is shown in Figure
\ref{fig:apdxCTMCSchematic}. The process repeatedly transitions from one state
to another after having waited in the previous state for a randomly determined
amount of time.\\

A CTMC can be completely characterised by its jumps and holding times. Specifically, the time between each jump or \textit{holding time} is exponentially distributed with mean $\nu(x)$ where $x$ is the state in which the process is holding. The next state that is jumped to is drawn from a jump probability distribution $r(\tilde{x} | x)$. The holding and jumping procedure is then repeated.\\

There is an equivalent definition involving the transition rate matrix,
$R \in \mathbb{R}^{S \times S}$, that we use in the main paper. The transition
rate matrix is defined as
\begin{equation}
    R(\tilde{x},  x) = \underset{\dt \rightarrow 0}{\text{lim}} \frac{q_{t|t-\dt}(x | \tilde{x}) - \delta_{x, \tilde{x}}}{\dt}
    \label{eq:ApdxRlimitDefn}
\end{equation}
where $R(\tilde{x}, x)$ is the $(\tilde{x}, x)$ element of the transition rate matrix and $q_{t|t-\dt}(x | \tilde{x})$ is the infinitesimal transition probability of being in state $x$ at time $t$ given that the process was in state $\tilde{x}$ at time $t-\dt$. Conversely, the CTMC can itself be defined through this infinitesimal transition probability
\begin{equation}
    q_{t|t-\dt}(x | \tilde{x}) = \delta_{x, \tilde{x}} + R(\tilde{x}, x) \dt + o(\dt) \label{eq:ApdxRtransitionDefin}
\end{equation}
where $o(\dt)$ represents terms that tend to zero at a faster rate than $\dt$. From this definition of the transition rate matrix, we can infer the following properties:
\begin{equation}
\textstyle     R(\tilde{x}, x) \geq 0 \quad \text{for} \quad \tilde{x} \neq x, \hspace{1cm} R(x, x) \leq 0, \hspace{1cm} R(x, x) = - \sum_{x' \neq x} R(x, x')
\label{eq:ApdxRconstraints}
\end{equation}
$R(\tilde{x}, x)$ is the rate at which probability mass moves from state $\tilde{x}$ to $x$. $R(x, x)$ is the total rate at which probability mass moves out of state $x$ and is thus negative.

In the time-homogeneous case, $R$ has simple relations to the jump and holding time definitions.
\begin{equation}
    \nu(x) = - \frac{1}{R(x, x)} \hspace{1.5cm} r(\tilde{x} | x) = ( 1 - \delta_{\tilde{x}, x}) \frac{R(x, \tilde{x})}{- R(x, x)}
\end{equation}
In the time-inhomogeneous case, our transition rate matrix will now depend on time, $R_t$, and these simple relations to the jump and holding time definition do not hold. However, $R_t$ will still follow equations (\ref{eq:ApdxRlimitDefn}), (\ref{eq:ApdxRtransitionDefin}) and (\ref{eq:ApdxRconstraints}).\\

The CTMC transition probabilities satisfy the Kolmogorov forward and backward equations. For $t > s$,
\begin{align}
    \text{Kolmogorov forward equation} \quad &\partial_t q_{t|s}(x | \tilde{x}) = \sum_{y} q_{t|s}(y | \tilde{x}) R_t(y, x)\\
    \text{Kolmogorov backward equation} \quad &\partial_s q_{t|s}(x | \tilde{x}) = - \sum_{y} R_s(\tilde{x}, y) q_{t|s}(x | y)
\end{align}
The Kolmogorov forward equation also gives us a differential equation for the marginals of the CTMC.
\begin{equation}
    \partial_t q_t(x) = \sum_{y} q_t(y) R_t(y, x).
\end{equation}

\paragraph{Exponential and Poisson Random Variables} In the time homogeneous case, holding times are exponentially distributed with mean $\nu(x) = -1 / R(x, x)$. The tau-leaping algorithm relies on the fact that the number of events in interval $[0, t]$ is Poisson distributed with mean $\frac{1}{\nu} t$ when the inter-event times are exponentially distributed with mean $\nu$.

\section{Proofs} \label{sec:ApdxProofs}

\subsection{Proof of Proposition \ref{prop:time_reversal}}

\begin{proof}
  We recall that a process $\{x_t\}_{t \in [0,T]}$ taking values in
  $\mathcal{X}$ is called a CTMC if it is right-continuous and satisfies the
  Markov property. Denote $\{y_t\}_{t \in [0,T]} = \{x_{T-t}\}_{t \in [0,T]}$
  except at the jump times of the forward process $\tau_n$ with
  $n \in \mathbb{N}$, where
  $y_{T - \tau_n} = x_{\tau_n}^- = \lim_{t \leq \tau_n, t \to \tau_n}
  x_t$. Hence, $\{y_t\}_{t \in[0,T]}$ is almost surely equal to
  $\{x_{T-t}\}_{t \in [0,T]}$ and is right-continuous. Since the Markov property
  is symmetric, we get that $\{y_t\}_{t \in [0,T]}$ is a CTMC. We now compute
  its transition matrix. Let $x,\tilde{x} \in \mathcal{X}$ with
  $x \neq \tilde{x}$, using the Kolmogorov forward equation, we have
  \begin{equation}
    \textstyle{
      \partial_t p_{t|s}(\tilde{x}|x) = \sum_{y \in \mathcal{X}} p_{t|s}(y|x) \hat{R}_{T-t}(y, \tilde{x}) \; ,
      }
    \end{equation}
    where $\{p_{t|s}, \ s, t \in [0,T], \ t >s\}$ is the
      transition probability system associated with
      $\{y_t\}_{t \in [0,T]}$ and $\{\hat{R}_{T-t}\}_{t \in [0,T]}$
      is the transition rate matrix associated with
      $\{y_t\}_{t \in [0,T]}$.  Note that
    \begin{align}
      p_{t|s}(x = j|\tilde{x} = i) &= \mathbb{P}(y_t = j \ | \ y_s =i)  \\
      &= \mathbb{P}(x_{T-t} = j \ | \ x_{T-s} =i)\\
      &= \mathbb{P}(x_{T-s} = i | x_{T-t} = j) \frac{\mathbb{P}(x_{T-t} = j)}{\mathbb{P}(x_{T-s} = i)}\\
      &= q_{T-s|T-t}(\tilde{x}=i|x=j) \frac{q_{T-t}(x=j)}{q_{T-s}(\tilde{x}=i)} \;
   \end{align}
   where $\{ q_{t|s}, s, t \in [0, T], t>s\}$ is the transition probability system associated with $\{x_t \}_{t \in [0,T]}$ and $\{q_t, t \in [0, T] \}$ are the marginals of $\{ x_t \}_{t \in [0, T]}$.
   Now, writing the backward Kolmogorov equation for $\{ x_t \}_{t \in [0,T]}$
   \begin{equation}
    \textstyle \partial_s q_{t|s}(\tilde{x} | x) = - \sum_{y \in \mathcal{X}} R_s(x, y) q_{t|s}(\tilde{x} | y)
   \end{equation}
   Re-labeling the time indices we obtain,
   \begin{talign}
       \partial_{T-t} q_{T-s | T-t}(\tilde{x} | x) &= - \sum_{y \in \mathcal{X}} R_{T-t}(x, y) q_{T-s | T-t}(\tilde{x} | y)\\
       \partial_t q_{T-s | T-t}(\tilde{x} | x) &= \sum_{y \in \mathcal{X}} R_{T-t}(x, y) q_{T-s | T-t}(\tilde{x} | y)
   \end{talign}
   Letting $s \rightarrow t$ and using that $\lim_{s \to t} q_{T-s|T-t}(x|\tilde{x}) =0$, we get that
   \begin{align}
       \hat{R}_{T-t}(x, \tilde{x}) &= \lim_{s \to t} \partial_t p_{t|s}(\tilde{x} | x)\\
       &= \lim_{s \to t} \partial_t \left( q_{T-s|T-t}(x | \tilde{x}) \frac{q_{T-t}(\tilde{x})}{q_{T-s}(x)} \right)\\
       &= \lim_{s \to t} \left[ \partial_t \left( q_{T-s|T-t}(x|\tilde{x}) \right) \frac{q_{T-t}(\tilde{x})}{q_{T-s}(x)} + q_{T-s|T-t}(x|\tilde{x}) \frac{\partial_t q_{T-t}(\tilde{x})}{q_{T-s}(x)} \right]\\
       &= \lim_{s \to t} \partial_t \left( q_{T-s|T-t}(x|\tilde{x}) \right) \frac{q_{T-t}(\tilde{x})}{q_{T-s}(x)}\\
       &= R_{T-t}(\tilde{x}, x) \frac{q_{T-t}(\tilde{x})}{q_{T-t}(x)}
   \end{align}
   Re-labeling the time-indices on the rate matrices, we obtain
   \begin{equation}
       \hat{R}_t(x, \tilde{x}) = R_t(\tilde{x}, x) \frac{q_t(\tilde{x})}{q_t(x)}
   \end{equation}
   Now we write the marginal ratio $\frac{q_t(\tilde{x})}{q_t(x)}$ in a different form
    \begin{align}
     \frac{q_t(\tilde{x})}{q_t(x)} &= \sum_{x_0} \frac{ \pdata(x_0) }{q_{t}(x)}q_{t|0}(\tilde{x} | x_0)\\
     &= \sum_{x_0} \frac{q_{0|t}(x_0 | x)}{q_{t|0}(x | x_0)} q_{t|0}(\tilde{x} | x_0).
    \end{align}
    Substituting in this form for the marginal ratio concludes the proof.

\end{proof}

\subsection{Proof of Proposition \ref{prop:CTELBO}}

In this section, we detail two proofs for Proposition \ref{prop:CTELBO}. The first is a formal proof using results from stochastic processes. We then provide a second informal proof for the same result to gain intuition into the $\LCT$ objective that only relies on elementary results from CTMCs.

\paragraph{Proof 1 - Stochastic Processes}

\begin{proof} Let us write $\mathbb{Q}$ for the path measure of the forward CTMC with rate matrix $R_t$, $\hat{\mathbb{Q}}$ for the path measure of its exact time reversal and $\mathbb{P}^{\theta}$ for the path measure of the approximate reverse process with rate matrix $\hat R^\theta_t$. Also, we use superscripts to notate conditioning on the starting point, for example $\mathbb{Q}^{x_0}$ denotes the path measure of the forward process conditioned to start in $x_0$.

With this notation, we have 
\begin{align}
-\log p_0^\theta (x_0)
&=-\log \int \pref(\rd x_T) \int_{\{\hat W_T=x_0\}} \mathbb{P}^{\theta,x_T}( \rd w)\\
&=-\log \int q_{T|0}(\rd x_T) \int_{\{\hat W_T=x_0\}} \hat{\mathbb{Q}}^{x_T}( \rd w) \frac{\rd \pref}{\rd q_{T|0}}(x_T) \frac{\rd \mathbb{P}^{\theta,x_T}}{\rd \hat{\mathbb{Q}}^{x_T}}(w)\\
&=-\log \int q_{T|0}(\rd x_T) \int 
\hat{\mathbb{Q}}( \rd w| \hat W_0=x_T, \hat W_T=x_0) \frac{\rd \pref}{\rd q_{T|0}}(x_T) \frac{\rd \mathbb{P}^{\theta,x_T}}{\rd \hat{\mathbb{Q}}^{x_T}}(w) \mathbb{Q}^{x_T}\{\hat W_0=x_0\}\\
&\leq \int q_{T|0}(\rd x_T) \int 
\hat{\mathbb{Q}}(\rd w| \hat W_0=x_T, \hat W_T=x_0) \left\{
-\log \frac{\rd \mathbb{P}^{\theta,x_T}}{\rd \hat{\mathbb{Q}}^{x_T}}(w)\right\} + C,
\end{align}    
where $\mathbb{P}^{\theta}, \hat{\mathbb{Q}}$ run in the reverse time direction. Writing $\hat{W}_s$ for a reverse path and integrating wrt $\pdata(\rd x_0)$ we have

\begin{align}
\int \pdata(x_0)[ -\log p_0^\theta (x_0)]
&\leq \int \pdata(x_0) \int q_{T|0}(\rd x_T) \int 
\hat{\mathbb{Q}}( \rd \hat{W}| \hat{W}_0=x_T, \hat{W}_T=x_0) \\
&\qquad \times
\left\{
\int_{s=0}^T \hat R_{T-s}^\theta(\hat{W}_s) \rd s -
\sum_{s: \hat{W}_{s-}\neq \hat{W}_s} 
\log \mathbb{P}_{T-s}^{\theta}(\hat{W}_s | \hat{W}_{s-}) R_{T-s}^{\theta}(\hat{W}_{s-})
\right\} + C,
\end{align}
where $\hat R_t^\theta(x)$ is shorthand for $- \hat R_{t}^\theta(x,x)$.

When $x_0\sim \pdata, x_T \sim q_{T|0}(\cdot | x_0), \hat{W} \sim  \hat{\mathbb{Q}}( \rd W| \hat{W}_0=x_T, \hat{W}_T=x_0)$, 
the reverse path is distributed according to $\pdata(\rd x_0) \mathbb{Q}_{x_0}( \rd W)$ and therefore 
$(\hat{W}_{s-}, \hat{W}_s)$ is distributed like $(W_{T-s}, W_{(T-s)-})$
and thus we have
\begin{align}
&\int \pdata(x_0)[ -\log p_0^\theta (x_0)]\\
&\leq \int \pdata(x_0) \mathbb{Q}_{x_0}(\rd W)
\left\{
\int_{s=0}^T \hat{R}_{T-s}^\theta({W}_{(T-s)-}) \rd s -
\sum_{s: {W}_{(T-s)-}\neq {W}_{T-s}} 
\log \mathbb{P}^{\theta}_{T-s}({W}_{(T-s)-} | {W}_{T-s}) \hat R_{T-s}^{\theta}({W}_{T-s})
\right\} + C\\
\end{align}    
Using Dynkin's lemma and the fact that 
$\mathbb{P}_t^\theta(x|y) \hat{R}_t^{\theta}(y) 
=\hat{R}_t^{\theta}(y,x)$ we can re-expresss this final line as
\begin{align}
&= \iint_{s=0}^T q_{T-s}(\rd x) 
\left\{\sum_{z\neq x} \hat{R}_{T-s}^\theta (x, z) 
- \sum_{z\neq x} {R}_{T-s}(x,z) \frac{\sum_{y\neq x}{R}_{T-s}(x,y) }{\sum_{z\neq x} {R}_{T-s}(x,z)} \log \hat{R}_{T-s}^\theta (y,x) \right\}\\
&=\iint_{s=0}^T q_{T-s}(\rd x)r_{T-s}(\rd y| x) 
\left\{\sum_{z\neq x} \hat{R}_{T-s}^\theta (x, z) 
- \sum_{z\neq x} {R}_{T-s}(x,z)  \log \hat{R}_{T-s}^\theta (y,x) \right\}\\
&=\iint_{s=0}^T q_{s}(\rd x)r_{s}(\rd y| x) 
\left\{\sum_{z\neq x} \hat{R}_{s}^\theta (x, z) 
- \sum_{z\neq x} {R}_{s}(x,z)  \log \hat{R}_{s}^\theta (y,x) \right\}
\end{align}
which rearranges to give the continuous time ELBO in the form of Proposition \ref{prop:CTELBO}.

\end{proof}

\paragraph{Proof 2 - Limit of Discrete Time ELBO}
\begin{proof}

Consider a partitioning of $[0, T]$, $0 = t_0 < t_1 < \dots < t_{k-1} < t_k < t_{k+1} < \dots < t_{K-1} < t_K = T$. Let $t_k - t_{k-1} = \dt$ for all $k$. In subscripts we use $k$ as a shorthand for $t_k$ when this does not cause confusion. Considering a CTMC with this time partitioning converts the problem into a discrete time Markov Chain with forward transition kernel, $q_{k+1|k}(x_{k+1}|x_k)$ and parameterized reverse kernel, $p^\theta_{k|k+1}(x_k | x_{k+1})$. Therefore, we can write the negative ELBO in its discrete time form, $\LDT$
\begin{align}
    \textstyle \LDT(\theta) = \mathbb{E}_{\pdata(x_0)} \Big[& \KL( q_{K|0}(x_K | x_0) || \pref(x_K) ) - \E_{q_{1|0}(x_1 | x_0)} \left[ \log p^\theta_{0|1}(x_0 | x_1)\right] \\
    & + \sum_{k=1}^{K-1} \mathbb{E}_{q_{k+1|0}(x_{k+1} | x_0)} \left[ \KL(q_{k|k+1,0}(x_{k}|x_{k+1}, x_0) || p_{k|k+1}^{\theta}(x_{k}|x_{k+1})) \right] \Big]
\end{align}
In the following, we will write the transition kernels in terms of the CTMC rate matrices and take the limit as $\dt \rightarrow 0$ to obtain a continuous time negative ELBO.

First, consider one item from the inner sum of $\LDT$
\begin{align}
    L_k &= \E_{\pdata(x_0) q_{k+1|0}(x_{k+1} | x_0)} \left[ \KL( q_{k|k+1, 0}(x_k | x_{k+1}, x_0) || p_{k|k+1}^\theta(x_k | x_{k+1}) )\right]\\
    &= - \E_{\pdata(x_0) q_{k+1|0}(x_{k+1} | x_0) q_{k|k+1, 0}(x_k|x_{k+1}, x_0)} \left[ \log p_{k|k+1}^\theta(x_k|x_{k+1})  \right] + C\\
    &= - \E_{q_k(x_k) q_{k+1|k}(x_{k+1}| x_k) } \left[ \log p_{k|k+1}^\theta(x_k| x_{k+1}) \right] + C
\end{align}
where we have absorbed terms that do not depend on $\theta$ into $C$. We now write $p_{k|k+1}^\theta(x_k|x_{k+1})$ in terms of $\hat{R}_k^\theta$.
\begin{equation}
    p_{k|k+1}^\theta(x_k|x_{k+1}) = \delta_{x_k, x_{k+1}} + \hat{R}^\theta_k(x_{k+1}, x_k) \dt + o(\dt)
\end{equation}
\begin{align}
    \log p_{k|k+1}^\theta(x_k | x_{k+1}) =& \log \left(\delta_{x_k, x_{k+1}} + \hat{R}^\theta_k(x_{k+1}, x_k) \dt + o(\dt) \right)\\
    =& \delta_{x_k, x_{k+1}} \log \left( 1 + \hat{R}_k^\theta(x_k, x_k) \dt + o(\dt) \right)\\
    & \, + (1-\delta_{x_k, x_{k+1}}) \log \left( \hat{R}_k^\theta(x_{k+1}, x_k) \dt + o(\dt) \right)\\
    =& \delta_{x_k, x_{k+1}} \left( \hat{R}_k^\theta(x_k, x_k) \dt + o(\dt)  \right)\\
    & \, + ( 1 - \delta_{x_k, x_{k+1}} ) \log \left( \hat{R}_k^\theta(x_{k+1}, x_k) \dt + o(\dt) \right)
    \label{eq:logpkm1kExpansion}
\end{align}
where on the last line we have used the series expansion for $\log(1+z) = z - \frac{z^2}{2} + o(z^2)$ valid for $|z| \leq 1, z \neq -1$. For any finite $R_k^\theta(x_k, x_k)$, $\dt$ can be taken small enough such that the series expansion holds. We now substitute this form for $\log p_{k|k+1}^\theta$ into $L_k$ and further write the expectation over $q_{k+1|k}(x_{k+1}|x_k) = \delta_{x_k, x_{k+1}} + R_k(x_k , x_{k+1}) \dt + o(\dt)$ as an explicit sum.
\begin{align}
    L_k = - \E_{q_k(x_k)} \Bigg[ \sum_{x_{k+1}} \Bigg\{& \Big[ \delta_{x_k, x_{k+1}} + R_k(x_k, x_{k+1}) \dt + o(\dt) \Big] \times \\
    & \Big[ \delta_{x_k, x_{k+1}} \left( \hat{R}_k^\theta(x_k, x_k) \dt + o(\dt)  \right)\\
    & \quad + \left( 1 - \delta_{x_k, x_{k+1}} \right) \log \left( \hat{R}_k^\theta(x_{k+1}, x_k) \dt + o(\dt) \right) \Big] \Bigg\} \Bigg] + C
\end{align}
\begin{align}
    L_k = - \E_{q_k(x_k)} \Bigg[ \sum_{x_{k+1}} \Bigg\{& \delta_{x_k, x_{k+1}} \hat{R}_k^\theta(x_k, x_k) \dt\\
    &+ ( 1 - \delta_{x_k, x_{k+1}}) R_k(x_k, x_{k+1}) \dt \times \\
    & \,\, \log \Big( \hat{R}_k^\theta(x_{k+1}, x_k) \dt + o(\dt) \Big) + o(\dt) \Bigg\} \Bigg] + C
\end{align}
We can isolate $\hat{R}_k^\theta$ within the $\log$ through the following re-arrangement
\begin{align}
   &\dt \log \left( \hat{R}_k^\theta(x_{k+1}, x_k) \dt + o(\dt) \right)\\
   &= \dt \log \dt + \dt \log \left( \hat{R}_k^\theta(x_{k+1}, x_k) + o(1) \right)\\
   &= \dt \log \dt  + \dt \log \left(1 + o(1) \right)+ \dt \log \left( \hat{R}_k^\theta(x_{k+1}, x_k) \right)
\end{align}
where the first two terms are independent of $\theta$ and tend to $0$ as $\dt \rightarrow 0$. Note that we assume $\hat{R}_k^\theta(x_{k+1}, x_k) > 0$ for $x_{k+1} \neq x_k$ pairs which have $R_k(x_k, x_{k+1}) > 0$. This assumption is valid because, for $x_{k+1} \neq x_k$, we have
\begin{equation}
    \hat{R}_k^\theta(x_{k+1}, x_k) = R_k(x_k ,x_{k+1}) \sum_{x_0} \frac{q_{k|0}(x_k | x_0)}{q_{k|0}(x_{k+1} | x_0)} p^\theta_{0|k}(x_0 | x_{k+1})
\end{equation}
and we assume $p_{0|k}^\theta(x_0 | x_{k+1}) > 0$ which is valid when we parameterize $p_{0|k}^\theta$ with a softmax output. We assume an irreducible Markov chain, hence $q_{k|0} > 0$ for $t_k > 0$.\\

With this re-arrangement, and absorbing constant terms into $C$, we obtain
\begin{align}
    L_k = - \E_{q_k(x_k)} \Bigg[ \sum_{x_{k+1}} \Bigg\{& \delta_{x_k, x_{k+1}} \hat{R}_k^\theta(x_k, x_k) \dt \\
    &  +(1 - \delta_{x_k, x_{k+1}}) R_k(x_k, x_{k+1}) \dt \log \Big( \hat{R}_k^\theta(x_{k+1}, x_k) \Big) \\
    &+ o(\dt) \Bigg\} \Bigg] + C
\end{align}
\begin{equation}
    L_k = - \E_{q_k(x_k)} \left[ \hat{R}_k^\theta(x_k, x_k)\dt + \sum_{x_{k+1} \neq x_k} R_k(x_k, x_{k+1}) \dt \log \hat{R}_k^\theta(x_{k+1}, x_k) + o(\dt)\right]
\end{equation}

The second term can be re-written so that it is more efficient to approximate with Monte Carlo. Currently the denoising model $p_{0|k}^\theta$ has to be evaluated for each term in the sum $\sum_{x_{k+1} \neq x_k}$ which would require multiple forward passes of the neural network. We can instead create a new probability distribution to sample from as follows. Define
\begin{equation}
    r_k(x_{k+1} | x_k) = (1 - \delta_{x_k, x_{k+1}} ) \frac{R_k(x_k, x_{k+1})}{\mathcal{Z}^k(x_k)}
\end{equation}
where
\begin{equation}
    \mathcal{Z}^k(x_k) = \sum_{x'_{k+1} \neq x_k} R_k(x_k, x'_{k+1})
\end{equation}
So we now have
\begin{equation}
    L_k = -\E_{q_k(x_k) r_k(x_{k+1}|x_k) } \left[ \hat{R}_k^\theta(x_k, x_k) \dt + \mathcal{Z}^k(x_k) \dt \log \hat{R}_k^\theta(x_{k+1}, x_k) + o(\dt) \right]
\end{equation}

Examining the other terms in $\LDT$ we have $\E_{\pdata(x_0)} \left[  \KL (q_{K|0}(x_K | x_0) || \pref(x_K) ) \right]$ which does not depend on $\theta$ and $\E_{q_{1|0}(x_1|x_0)} \left[ \log p_{0 | 1}^\theta(x_0 | x_1) \right]$ which we expand here
\begin{align}
    &\E_{q_{1|0}(x_{1}|x_{0})} \left[ \log p_{0 | 1}^\theta(x_{0} | x_{1}) \right] \\
    & \hspace{1cm} = \sum_{x_{1}} \left\{ \delta_{x_{1}, x_{0}} + \dt R_{1}(x_{0}, x_{1}) + o(\dt) \right\} \log p_{0 | 1}^\theta(x_{0} | x_{1})\\
    & \hspace{1cm} = \log p_{0 | 1}^\theta(x_{0} | x_{0}) + \dt \sum_{x_{1}} R_{1}(x_{0}, x_{1}) \log p_{0 | 1}^\theta(x_{0} | x_{1}) + o(\dt)\\
    & \hspace{1cm} = \dt \hat{R}_{1}^\theta(x_{0}, x_{0}) + \dt \sum_{x_{1}} R_{1}(x_{0}, x_{1}) \log p_{0 | 1}^\theta(x_{0} | x_{1}) + o(\dt)
\end{align}
where on the final line we have used eq \ref{eq:logpkm1kExpansion}. In summary,
\begin{align}
    \LDT =& \dt \E_{\pdata(x_0) q_{1|0}(x_1 | x_0)} \left[ - \hat{R}_1^\theta(x_0, x_0) + \sum_{x_1} R_1(x_0, x_1) \log p_{0|1}^\theta(x_0 | x_1) \right]\\
    & - \dt \sum_{k=1}^{K-1} \E_{q_k(x_k) r_k(x_{k+1} | x_k)} \left[ \hat{R}_k^\theta(x_k, x_k) + \mathcal{Z}^k(x_k) \log \hat{R}_k^\theta (x_{k+1}, x_k) \right]\\
    & + o(\dt) + C
\end{align}
We now take the limit of $\LDT$ as $\dt \rightarrow 0$ and $K \rightarrow \infty$.
\begin{equation}
    \underset{\dt \rightarrow 0}{\text{lim}} \LDT = \LCT = - \int_0^T \E_{q_t(x)r_t(\tilde{x} | x)} \left[\hat{R}_t^\theta(x, x) + \mathcal{Z}^t(x) \log \left( \hat{R}_t^\theta(\tilde{x}, x) \right) \right] dt + C
\end{equation}
We can estimate the integral with Monte Carlo if we consider it to be an expectation with respect to a uniform distribution over times $(0, T)$. We also write $\hat{R}_t^\theta(x, x)$ explicitly as the negative off diagonal row sum to obtain
\begin{equation}
   \textstyle \LCT(\theta) = T \, \mathbb{E}_{t\sim \mathcal{U}(0, T) q_t(x) r_t(\tilde{x}|x)} \Big[ \Big\{ \sum_{x' \neq x} \hat{R}_t^\theta(x, x') \Big\} - \mathcal{Z}^t(x) \log \left(\hat{R}_t^\theta(\tilde{x}, x) \right) \Big] + C.
\end{equation}
\end{proof}

\subsection{Proof of Proposition \ref{prop:dimfactorize}}

\begin{proof}
We assume $q_{t|s}(\bm{x}_t^{1:D} | \bm{x}_s^{1:D})$ factorizes as $\prod_{d=1}^D q_{t|s}(x^d_t | x_s^d)$ where $q_{t|s} (x_t^d | x_s^d), \,\, d=1, \dots, D$ are the transition probabilities for independent singular dimensional CTMCs each with forward rate $R_t^d(\tilde{x}^d, x^d)$. In the following, we will drop time subscripts on $x$ arguments. To find the correspondence between $R_t^{1:D}$ and $R_t^d$, we use the Kolmogorov forward equation
\begin{equation}
    \partial_t q_{t|s}(\bm{x}^{1:D} | \tilde{\bm{x}}^{1:D}) = \sum_{\bm{y}^{1:D}} q_{t|s}(\bm{y}^{1:D} | \tilde{\bm{x}}^{1:D}) R_t^{1:D}(\bm{y}^{1:D}, \bm{x}^{1:D})
\end{equation}
Substitute in our factorized form for $q_{t|s}$ into the LHS
\begin{align}
    \partial_t q_{t|s}(\bm{x}^{1:D}|\tilde{\bm{x}}^{1:D}) &= \partial_t \left\{ \prod_{d=1}^D q_{t|s}(x^d | \tilde{x}^d) \right\}\\
    &= \sum_{d=1}^D q_{t|s}(\bm{x}^{\Dd} | \tilde{\bm{x}}^{\Dd}) \partial_t q_{t|s}(x^d | \tilde{x}^d)\\
    &= \sum_{d=1}^D q_{t|s}(\bm{x}^{\Dd} | \tilde{\bm{x}}^{\Dd}) \sum_{y^d} q_{t|s}(y^d | \tilde{x}^d) R_t^d(y^d, x^d)\\
    &= \sum_{d=1}^D \sum_{\bm{y}^{1:D}} q_{t|s}(\bm{x}^{\Dd} | \tilde{\bm{x}}^{\Dd}) q_{t|s}(y^d | \tilde{x}^d) R_t^d(y^d, x^d) \delta_{\bm{x}^{\Dd}, \bm{y}^{\Dd}}\\
    &= \sum_{d=1}^D \sum_{\bm{y}^{1:D}} q_{t|s}(\bm{y}^{\Dd} | \tilde{\bm{x}}^{\Dd}) q_{t|s}(y^d | \tilde{x}^d) R_t^d(y^d, x^d) \delta_{\bm{x}^{\Dd}, \bm{y}^{\Dd}}\\
    &= \sum_{\bm{y}^{1:D}} q_{t|s}(\bm{y}^{1:D}|\tilde{\bm{x}}^{1:D}) \sum_{d=1}^D R_t^d(y^d, x^d) \delta_{\bm{x}^{\Dd}, \bm{y}^{\Dd}}
\end{align}
We therefore obtain
\begin{equation}
    \sum_{\bm{y}^{1:D}} q_{t|s}(\bm{y}^{1:D} | \tilde{\bm{x}}^{1:D}) R_t^{1:D}(\bm{y}^{1:D}, \bm{x}^{1:D}) = \sum_{\bm{y}^{1:D}} q_{t|s}(\bm{y}^{1:D}|\tilde{\bm{x}}^{1:D}) \sum_{d=1}^D R_t^d(y^d, x^d) \delta_{\bm{x}^{\Dd}, \bm{y}^{\Dd}}
\end{equation}
This must be true for all possible factorizable forward process transitions, $q_{t|s}$, including $q_{t|s}(\bm{y}^{1:D} | \tilde{\bm{x}}^{1:D}) = \delta_{\bm{y}^{1:D}, \tilde{\bm{x}}^{1:D}}$. This choice gives us our forward rate relation
\begin{equation}
    R_t^{1:D}(\tilde{\bm{x}}^{1:D}, \bm{x}^{1:D}) = \sum_{d=1}^D R_t^d(\tilde{x}^d, x^d) \delta_{\bm{x}^{\Dd}, \tilde{\bm{x}}^{\Dd}}
\end{equation}

Substituting this into our expression for the reverse rate from Proposition \ref{prop:time_reversal} we obtain
\begin{align}
    \hat{R}_t^{1:D}(\bm{x}^{1:D}, \tilde{\bm{x}}^{1:D}) &= \sum_{\bm{x}_0^{1:D}} \sum_{d=1}^D R_t^d(\tilde{x}^d, x^d) \frac{q_t(\tilde{\bm{x}}^{1:D} | \bm{x}_0^{1:D})}{q_t(\bm{x}^{1:D} | \bm{x}_0^{1:D})} q_{0|t}(\bm{x}_0^{1:D} | \bm{x}^{1:D}) \delta_{\bm{x}^{\Dd}, \tilde{\bm{x}}^{\Dd}}\\
    &= \sum_{\bm{x}_0^{1:D}} \sum_{d=1}^D R_t^d(\tilde{x}^d, x^d) \frac{q_{t|0}(\tilde{x}^{d} | x_0^{d})}{q_{t|0}(x^{d} | x_0^{d})} q_{0|t}(\bm{x}_0^{1:D} | \bm{x}^{1:D}) \delta_{\bm{x}^{\Dd}, \tilde{\bm{x}}^{\Dd}}\\
    &= \sum_{d=1}^D R_t^d(\tilde{x}^d, x^d) \delta_{\bm{x}^{\Dd}, \tilde{\bm{x}}^{\Dd}} \sum_{x_0^d} q_{0|t}(x_0^d | \bm{x}^{1:D}) \frac{q_{t|0}(\tilde{x}^d | x_0^d)}{q_{t|0}(x^d | x_0^d)} \sum_{\bm{x}_0^{\Dd}} q_{0|t}(\bm{x}_0^{\Dd} | x_0^d, \bm{x}^{1:D})\\
    &= \sum_{d=1}^D R_t^d(\tilde{x}^d, x^d) \delta_{\bm{x}^{\Dd}, \tilde{\bm{x}}^{\Dd}} \sum_{x_0^d} q_{0|t}(x_0^d | \bm{x}^{1:D}) \frac{q_{t|0}(\tilde{x}^d | x_0^d)}{q_{t|0}(x^d | x_0^d)}
\end{align}
\end{proof}

\subsection{Proof of Proposition \ref{prop:corrector_rate}}

\begin{proof}
By the Kolmogorov forward equation applied to the forwards process, we have
\begin{equation}
    \partial_t q_t(x_t) = \sum_y R_t(y, x_t) q_t(y)
\end{equation}
In addition, applying the Kolmogorov forward equation to the reverse process, which has the same marginals as the forward but time-reversed, we get
\begin{equation}
    - \partial_t q_t(x_t) = \sum_y \hat{R}_t(y, x_t) q_t(y)
\end{equation}
Summing these two equations gives
\begin{equation}
    \sum_y \left\{ R_t(y, x_t) + \hat{R}_t(y, x_t) \right\} q_t(y) = 0
\end{equation}
Therefore, by comparison with the Kolmogorov equation, $R_t + \hat{R}_t$ is the rate matrix of a CTMC with invariant distribution $q_t$.
\end{proof}

\subsection{Proof of Theorem \ref{thm:error_bound}}
\label{sec:ApdxErrorBound}

In this section, we derive a bound on the error of our tau-leaping diffusion model. Because the tau-leaping approximation is only interesting in the case where multiple jumps are made along different dimensions in a single step, we choose to make the dependence of our bound on the dimension of our model explicit, rather than simply considering the case of fixed $D$ and $\tau \rightarrow 0$.

Recall from the main text that we have a time-homogeneous rate matrix $R_t$ on $\mathcal{X}$, from which we construct the factorised rate matrix $R_t^{1:D}$ on $\mathcal{X}^D$ by setting $R^d_t = R_t$ for each $d$, and will denote $|R| = \sup_{t \in [0,T], x \in \mathcal{X}} |R_t(x,x)|$, and let $t_{\textup{mix}}$ be the (1/4)-mixing time of the CTMC with rate $R_t$. We also define addition on the state space $\mathcal{X}^D$ using a mapping from $\mathcal{X}$ to $\mathbb{Z}$ as in Section {\ref{sec:TauLeaping}} and component-wise addition.

\setcounter{theorem}{0}
\begin{theorem}
For any $D \geq 1$ and distribution $\pdata$ on $\mathcal{X}^D$, let $\{x_t\}_{t \in [0,T]}$ be a CTMC starting in $\pdata$ with rate matrix $R^{1:D}_t$ as above. Suppose that $\hat{R}_t^{\theta \, 1:D}$ is an approximation to the reverse rate matrix and let $(y_k)_{k = 0, 1, \dots, N}$ be a tau-leaping approximation to the reverse dynamics with maximum step size $\tau$. Suppose further that there is some constant $M > 0$ independent of $D$ such that
\begin{equation}
    \sum_{y \neq x} \left| \hat R_t^{1:D}(x,y) - \hat{R}_t^{\theta \, 1:D}(x,y) \right| \leq M
\end{equation}
for all $t \in [0,T]$. Then under the assumptions listed below, there are constants $C_1, C_2 > 0$ depending on $\mathcal{X}$ and $R_t$ but not $D$ such that, if $\mathcal{L}(y_0)$ denotes the law of $y_0$, we have the total variation bound
\begin{equation}
\textstyle
    ||\mathcal{L}(y_0) - \pdata||_{\textup{TV}} \leq 3MT + \left\{ \big(|R|SDC_1\big)^2 + \frac{1}{2} C_2(M + C_1 SD|R|) \right\} \tau T + 2 \exp \left \{ - \frac{T \log^2 2}{t_{\textup{mix}} \log 4D} \right \}
\end{equation}
\end{theorem}
\setcounter{theorem}{2}

The above theorem holds under the following assumptions, where we write $x \sim y$ for $x,y \in S^D$ if they differ in at most one coordinate.

\begin{assumption}
\label{ass:positivity}
The data distribution $\pdata$ is strictly positive.
\end{assumption}

\begin{assumption}
\label{ass:bounded_ratios}
There exists a constant $C_1 > 0$, depending on $S$ and $R_t$ but not $D$, such that for all $t \in [0,T]$ and $x,y \in S^D$ such that $x \sim y$, we have
\begin{equation}
\label{eq:bounded_ratios}
    \frac{q_t(x)}{q_t(y)} \leq C_1.
\end{equation}
\end{assumption}

\begin{assumption}
\label{ass:bounded_rate_change}
There exists a constant $C_2 > 0$, depending on $S$ and $R_t$ but not $D$, such that for all $t \in [0,T]$ and all $x,y \in S^D$ such that $x \sim y$, we have
\begin{equation}
\label{eq:error_bound_assumption_2}
    \sum_z \left|\hat R_t(x,x+z) - \hat R_t(y,y+z)\right| \leq C_2.
\end{equation}
\end{assumption}

If instead we were to allow $C_1$ and $C_2$ to depend on the dimension $D$, then Assumptions \ref{ass:bounded_ratios} and \ref{ass:bounded_rate_change} follow trivially from Assumption \ref{ass:positivity} and the finiteness of the state space. However, we choose the stronger formulation above in order to make explicit the dependence of the error bound on the dimension, as previously explained.

As remarked in the main text, in most cases of practical interest (including the two examples explored in Section \ref{sec:Experiments}), Assumption \ref{ass:bounded_rate_change} holds only approximately. However, we still expect the bound in Assumption \ref{ass:bounded_rate_change} to hold whenever $x,y$ are in addition chosen such that the tau-leaping approximation of the reverse process makes a jump between them with reasonably high probability. For example, in the case where our data is ordinal, we expect that for any $x \sim y$ jumps from $x$ to $y$ are only common when $x$ is close to $y$, and thus $\hat R_t(x, x+z)$ and $\hat R_t(y, y+z)$ should be reasonably close whenever a jump from $x$ to $y$ occurs. Under a weaker assumption of this form, the proof of Theorem \ref{thm:error_bound} can be adapted to work along similar lines, at the cost of a significant increase in technicality. We therefore choose to focus on the simpler case where Assumption \ref{ass:bounded_rate_change} holds as it illustrates the key ideas.

In order to prove Theorem \ref{thm:error_bound} we will require the following lemmas.

\begin{proposition}
\label{prop:single_step_tv}
Let $(x_t)_{t \in [0,T]}$ and $(y_t)_{t \in [0,T]}$ be continuous time Markov chains on a finite state space $S$ with generators $G_t$ and $H_t$ respectively which are both bounded and continuous in $t$. Let the Markov kernels associated to $X$ and
$Y$ be $K$ and $L$ respectively. Then for any probability distribution $\nu$ on $S$ we have
\begin{equation}
    ||\nu K - \nu L||_{\textup{TV}} \leq \int_0^T \sup_{x \in S} \Big\{ \sum_{y \neq x} |G_t(x,y) - H_t(x,y)| \Big\} \; \mathrm{d}t
\end{equation}
\end{proposition}
\vspace{-1.0cm}
\begin{proof}
We define a coupling of $(x_t)_{t \in [0,T]}$ and $(y_t)_{t \in [0,T]}$ as follows, based on the construction in Chapter 20.1 of \cite{levin2009markovchains}. First take $Z \sim \nu$ and set $x_0 = y_0 = Z$. Also define the variables $\tilde{x}_0=\tilde{y}_0 = Z$.

Next, fix $\lambda$ such that $|G_t(x,x)|, |H_t(x,x)| \leq \lambda$ for all $x \in S$, $t \in [0,T]$, let $(N_s)_{1 \leq s \leq T}$ be a Poisson process on $[0, T]$ of rate $\lambda$, and set $N_0 = 0$. We write $N=N_T$, and 
 $S_1, S_2, \dots, S_{N}$ for the arrival times and set $S_{n+1}=T$. 
We construct $x_t$ and $y_t$ for $t > 0$ inductively as follows.
For $t\in [0, S_1)$ let $x_t=y_t=x_0$. 
Let $1\leq j\leq N$. Given $(x_{r}:r< S_j)$, $(y_r: r< S_j)$, and $\tilde{x}_j, \tilde{y}_j$, 
define the following probability measures 
\begin{align}
    &\rho_j(\tilde{x}_j, w):= \begin{cases}
                  \displaystyle G_{S_j}(\tilde{x}_j, w) / \lambda, \quad w\neq \tilde{x}_j\\
                  \displaystyle   1-G_{S_j}(\tilde{x}_j, w)/ \lambda, \quad w=\tilde{x}_j,
                    \end{cases}\\
    &\rho'_j(\tilde{y}_j, w):= \begin{cases}
                 \displaystyle    H_{S_j}(\tilde{y}_j, w) / \lambda, \quad w\neq \tilde{y}_j\\
               \displaystyle      1-H_{S_j}(\tilde{y}_j, w)/ \lambda, \quad w=\tilde{y}_j.
                    \end{cases}
\end{align}

Sample $(\tilde{x}_{j+1}, \tilde{y}_{j+1})$ from a maximal coupling of $(\rho_j, \rho'_j)$ and for $t\in [S_j, S_{j+1})$ set 
$x_t=\tilde{x}_{j+1}$, $y_t=\tilde{y}_{j+1}$. Finally set $x_T=x_{S_N}$ and $y_T=y_{S_N}$.

Now, observe that $(x_t, y_t)_{t \in [0,T]}$ defined in this way is a coupling of the given Markov chains. Moreover,
\begin{align}
    ||\nu K - \nu L||_{\textup{TV}} &\leq \mathbb{P}(x_T \neq y_T) \\
    &= \E\left[\sum_{j=1}^N \mathbb{I} \left\{x_s=y_s, \, s< S_j\right\} \mathbb{I}\left\{ x_{S_j}\neq y_{S_j} \right\} \right]\\
    &= \sum_{n=0}^\infty \frac{\lambda^n e^{-\lambda}}{n!}
    \sum_{j=0}^n\E\left[\mathbb{I} \left\{x_s=y_s, \, s< S_j\right\} \mathbb{I}\left\{ x_{S_j}\neq y_{S_j} \right\} \right]
    \intertext{and using the fact that jumps are coupled maximally}
    &= \sum_{n=0}^\infty \frac{\lambda^n e^{-\lambda}}{n!}
    \sum_{j=0}^n\E\Big[\mathbb{I} \left\{x_s=y_s, \, s< S_j\right\} 
    \times 
    \|\rho_j(X_{S_{j-1}}, \cdot) - \tilde{\rho}_j(X_{S_{j-1}}, \cdot)\|_\mathrm{TV} \Big]\\
     &= \sum_{n=0}^\infty \frac{\lambda^n e^{-\lambda}}{n!}
     \sum_{j=0}^n\E\Big[\mathbb{I} \left\{x_s=y_s, \, s< S_j\right\} 
     \frac{1}{\lambda}\sum_{z}| G_{S_j}(x_{S_{j-1}},z ) -  H_{S_j}(x_{S_{j-1}},z )| \Big]\\
     &= \frac{1}{\lambda}\E \Big[ \sum_{s:x_s \neq x_{s-}}
     \sum_{z}| G_{s}(x_{s-},z ) -  H_{s}(x_{s-},z )|
     \Big]\\
     &= \frac{1}{\lambda}\int_{s=0}^T\E \Big[ \lambda 
     \sum_{z}| G_{s}(x_{s-},z ) -  H_{s}(x_{s-},z )|
     \Big]\\
     &=\int_{s=0}^T\E \Big[ \sum_{z}| G_{s}(x_{s-},z ) -  H_{s}(x_{s-},z )|
     \Big]\mathrm{d} s
\end{align}
as required.
\end{proof}

\begin{proposition}
\label{prop:regularity_of_R}
For all $t \in [0,T]$ and $x, y \in \mathcal{X}^D$ such that $x \sim y$, we have
\begin{equation}
    |\partial_t \hat R_t(x,y)| \leq 2 |R|^2 S D C_1^2
\end{equation}
Moreover, it follows that $\hat R_t$ is bounded and continuous in $t$.
\end{proposition}

\begin{proof} Omitting the superscripts for brevity where the notation is clear, we have
\begin{align}
    \left| \partial_t \hat R_t^{1:D}(x^{1:D},y^{1:D}) \right| &= \left| R_t(y, x) \partial_t \left\{ \frac{q_t(y)}{q_t(x)} \right\} \right| \\
    &= \left| R_t(y, x) \left\{ \frac{q_t(y)}{q_t(x)} \frac{\sum_{z} R_t(z,y) q_t(z)}{q_t(y)} - \frac{q_t(y)}{q_t(x)}\frac{\sum_{z} R_t(z,x) q_t(z)}{q_t(x)} \right\} \right| \\
    &\leq 2|R|^2 SD C_1^2
\end{align}
where the second line follows from Kolmogorov's forward equation and the final inequality follows from Assumption \ref{ass:bounded_ratios} plus the fact that $R_t(z,x)$ (resp. $R_t(z,y)$) is only non-zero when $x \sim z$ (resp. $y \sim z$), and there are at most $|S||D|$ values of $x$ (resp. $y$) for which this holds.
\end{proof}

We now give the proof of Theorem \ref{thm:error_bound}.

\begin{proof}[Proof of Theorem \ref{thm:error_bound}]
Let us label the time steps used in tau-leaping by $0 = t_0 < t_1 < \dots < t_N = T$, denote $\tau_k = t_{k} - t_{k-1}$, and denote the target stationary distribution by $\pi^D(x^{1:D}) = \prod_{d=1}^D \pi(x^d)$, where $\pi$ is the invariant distribution of the single-dimensional transition matrix $R^1_t$.

Also, let $\mathcal{R}^{\theta, (\tau)}_k$ be the Markov kernel corresponding to applying the tau-leaping approximation with rate matrix $\hat R^\theta_{t_k}$ to move from $t_k$ to $t_{k-1}$, and denote $\mathcal{R}^{\theta, (\tau)} = \mathcal{R}^{\theta, (\tau)}_N \mathcal{R}^{\theta, (\tau)}_{N-1} \dots \mathcal{R}^{\theta, (\tau)}_1$ so that $\mathcal{R}^{\theta, (\tau)}$ expresses the full dynamics of the tau-leaping process and we have $\mathcal{L}(\hat y_0) = \pi^D \mathcal{R}^{\theta, (\tau)}$.

Then, as in \cite{de2021diffusion} we can decompose
\begin{equation}
    ||\pi^D \mathcal{R}^{\theta, (\tau)} - p_d||_{\textup{TV}} \leq ||\pi^D \mathcal{R}^{\theta, (\tau)} - \pi^D (\mathbb{P}^R)_{T|0}||_{\textup{TV}} + ||\pi^D - q_T||_{\textup{TV}}
\end{equation}
where $\mathbb{P}^R$ is the path measure of the exact reverse process.

We deal with the second term first. Let $t_{\textup{mix}}$ be the (1/4)-mixing time of the single-dimension CTMC with rate matrix $R^1_t$, i.e.
\begin{equation}
    t_{\textup{mix}} = \inf \left \{t \geq 0 : \sup_{x_0^1 \in S} ||q_{t|0}(\; \cdot \;| x_0^1) - \pi||_{\textup{TV}} \leq \frac{1}{4} \right \}
\end{equation}
It then follows from
\begin{equation}
    ||q_{t|0}(\; \cdot \; |  x_0^{1:D}) - \pi^D||_{\textup{TV}} \leq \sum_{d=1}^D ||q_{t|0}(\; \cdot \;| x_0^d) - \pi||_{\textup{TV}}
\end{equation}
that $t_{\textup{mix}}^D$, the (1/4)-mixing time of the full CTMC with rate matrix $R^{1:D}_t$, satisfies the inequality $t^D_{mix} \leq \{1+ \lceil \log_2 D \rceil \}t_{\textup{mix}}$. If we view $(x_{m t^D_{mix}})_{m \in \mathbb{N}}$ as a discrete-time Markov chain, then standard results on Markov chain mixing (see, for example, Chapter 4.5 of \cite{levin2009markovchains}) show that
\begin{equation}
    ||q_{mt^D_{mix} | 0}(\; \cdot \; | x_0^{1:D}) - \pi^D||_{\textup{TV}} \leq 2^{-m}
\end{equation}
It then follows that for any $T \geq 0$ we have
\begin{equation}
    ||\pi^D - q_T||_{\textup{TV}} \leq 2 \exp \left\{ - \frac{T \log 2}{t^D_{mix}} \right\} \leq 2 \exp \left \{ - \frac{T \log^2 2}{t_{\textup{mix}} \log 4D} \right \}
\end{equation}
completing the bound on the second term.

To bound the first term, we define $\mathcal{P}_k = (\mathbb{P}^R)_{T - t_{k-1} | T - t_k}$ and decompose it as
\begin{align}
    ||\pi \mathcal{R}^{\theta, (\tau)} - \pi (\mathbb{P}^R)_{T|0}||_{\textup{TV}} &\leq \sup_{\nu} ||\nu \mathcal{R}^{\theta, (\tau)}_N \dots \mathcal{R}^{\theta, (\tau)}_1 - \nu \mathcal{P}_N \dots \mathcal{P}_1||_{\textup{TV}} \\
    &\leq \sup_\nu ||\nu \mathcal{R}^{\theta, (\tau)}_N \mathcal{R}^{\theta, (\tau)}_{N-1} \dots \mathcal{R}^{\theta, (\tau)}_1 - \nu \mathcal{R}^{\theta, (\tau)}_N \mathcal{P}_{N-1} \dots \mathcal{P}_1||_{\textup{TV}} \\
    &\hspace{5mm} + \sup_\nu ||\nu \mathcal{R}^{\theta, (\tau)}_N \mathcal{P}_{N-1} \dots \mathcal{P}_1 - \nu \mathcal{P}_N \mathcal{P}_{N-1} \dots \mathcal{P}_1||_{\textup{TV}} \\
    &\leq \sup_\nu ||\nu \mathcal{R}^{\theta, (\tau)}_{N-1} \dots \mathcal{R}^{\theta, (\tau)}_1 - \nu \mathcal{P}_{N-1} \dots \mathcal{P}_1||_{\textup{TV}} + \sup_\nu ||\nu \mathcal{R}^{\theta, (\tau)}_N - \nu \mathcal{P}_N||_{\textup{TV}} \\
    &\leq \sum_{k = 1}^N \sup_\nu ||\nu \mathcal{R}^{\theta, (\tau)}_k - \nu \mathcal{P}_k||_{\textup{TV}}
\end{align}
by proceeding inductively. So it suffices to find bounds on the total variation distance accumulated on each interval $[t_{k-1}, t_k]$.

Let $\mathcal{R}_k^\theta$ be the Markov kernel corresponding to running the chain from $t_k$ to $t_{k-1}$ with constant rate matrix $\hat R^\theta_{t_k}$. Since by Proposition \ref{prop:regularity_of_R} the reverse rate matrix $\hat R_t$ is bounded and continuous in $t$, using Proposition \ref{prop:single_step_tv} we made deduce that for any distribution $\nu$ on $S$ we have
\begin{align}
    ||\nu \mathcal{P}_k - \nu \mathcal{R}_k^\theta||_{\textup{TV}} &\leq \int_{t_{k-1}}^{t_k} \sup_{x \in S} \Big\{ \sum_{y \neq x} \big| \hat R_t(x,y) - \hat R^\theta_{t_k}(x,y) \big| \Big\} \; \mathrm{d} t \\
    &\leq \int_{t_{k-1}}^{t_k} \sup_{x \in S} \Big\{ \sum_{y \neq x} \big| \hat R_t(x,y) - \hat R_{t_k}(x,y) \big| \Big\} \; \mathrm{d} t  \\
    &\hspace{5mm} + \int_{t_{k-1}}^{t_k} \sup_{x \in S} \Big\{ \sum_{y \neq x} \big| \hat R_{t_k}(x,y) - \hat R^\theta_{t_k}(x,y) \big| \Big\} \; \mathrm{d} t
\end{align}
The first half of this expression can be bounded using the Mean Value Theorem, according to
\begin{align}
    \int_{t_{k-1}}^{t_k} \sup_{x \in S} \Big\{ \sum_{y \neq x} \big| \hat R_t(x,y) - \hat R_{t_k}(x,y) \big| \Big\} \; \mathrm{d} t &\leq \int_{t_{k-1}}^{t_k} |t - t_k| \cdot 2|R|^2S^2D^2C_1^2 \; \mathrm{d} t \\
    &\leq \big(|R|SDC_1\tau_k\big)^2
\end{align}
where in the first line we have used that the summand is only non-zero when $y \sim x$, and there are at most $|S||D|$ values of $y$ for which this holds. The second term can be bounded using condition (\ref{eq:error_bound_condition}), to get 
\begin{equation}
    \int_{t_{k-1}}^{t_k} \sup_{x \in S} \Big\{ \sum_{y \neq x} \big| \hat R_{t_k}(x,y) - \hat R^\theta_{t_k}(x,y) \big| \Big\} \; \mathrm{d} t \leq M \tau_k
\end{equation}
Combining these two expressions, we get a bound on $||\nu \mathcal{P}_k - \nu \mathcal{R}_k^\theta||_{\textup{TV}}$.
\begin{equation}
    ||\nu \mathcal{P}_k - \nu \mathcal{R}_k^\theta||_{\textup{TV}} \leq \big(|R|SDC_1\tau_k\big)^2 + M\tau_k
\end{equation}

It remains to bound $||\nu \mathcal{R}^{\theta}_k - \nu \mathcal{R}^{\theta, (\tau)}_k||_{\textup{TV}}$. Note that performing tau-leaping with rate matrix $\hat R_{t_k}^\theta$ starting in $x_{t_k}$ is equivalent to running a continuous time Markov chain from time $t_k$ to $t_{k-1}$ with constant rate matrix $\hat R^{\theta, (\tau)}_{t_k}$ given by
\begin{equation}
    \hat R^{\theta, (\tau)}_{t_k}(x,y) = \hat R^\theta_{t_k}(x_{t_k}, y - x + x_{t_k})
\end{equation}
(followed potentially by a clamping operation to keep us within $\mathcal{X}^D$). By an analogous argument to the proof of Proposition \ref{prop:single_step_tv},
\begin{equation}
    ||\delta_{x_{t_k}} \mathcal{R}^{\theta}_k - \delta_{x_{t_k}} \mathcal{R}^{\theta, (\tau)}_k||_{\textup{TV}} \leq \int_{t_{k-1}}^{t_k} \E \Big[ \sum_{y \neq x_t} |\hat R_{t_k}^\theta(x_t, y) - \hat R^\theta_{t_k}(x_{t_k}, y - x_t + x_{t_k})| \Big] \; \mathrm{d} t 
\end{equation}
where the expectation is taken over $(x_t)_{t \in [t_{k-1},t_k]}$ distributed according to the exact CTMC with rate matrix $\hat R_{t_k}^\theta$. (Note we have disregarded the clamping operation, since this can only decrease the resulting total variation distance.)

We may rewrite this bound in terms of the exact reverse process using condition (\ref{eq:error_bound_condition}) to get
\begin{equation}
    ||\delta_{x_{t_k}} \mathcal{R}^{\theta}_k - \delta_{x_{t_k}} \mathcal{R}^{\theta, (\tau)}_k||_{\textup{TV}} \leq \int_{t_{k-1}}^{t_k} \E \Big[ 2M + \sum_{y \neq x_t} |\hat R_{t_k}(x_t, y) - \hat R_{t_k}(x_{t_k}, y - x_t + x_{t_k})| \Big] \; \mathrm{d} t 
\end{equation}
Let $J_t$ be the number of jumps that $(x_t)$ makes between $t_k$ and $t$, and label the times of these jumps as $s_1, \dots, s_j$ where $t \leq s_1 \leq \dots \leq s_j \leq t_k$ and $j = J_t$ for convenience. Then by Assumption \ref{ass:bounded_rate_change}, we have
\begin{align}
    \sum_{y \neq x_t} |\hat R_{t_k}(x_t, y) - \hat R_{t_k}(x_{t_k}, y - x_t + x_{t_k})| &\leq \sum_{z} |\hat R_{t_k}(x_t, x_t + z) - \hat R_{t_k}(x_{s_1}, x_{s_1} + z)| + \dots \\
    & \hspace{5mm} + \sum_{z} |\hat R_{t_k}(x_{s_j}, x_{s_j} + z) - \hat R_{t_k}(x_{t_k}, x_{t_k} + z)| \\
    & \leq C_2 J_t
\end{align}
where we have made the substitution $z = y - x_{t_k}$. We conclude that
\begin{align}
    ||\delta_{x_{t_k}} \mathcal{R}^{\theta}_k - \delta_{x_{t_k}} \mathcal{R}^{\theta, (\tau)}_k||_{\textup{TV}} &\leq \int_{t_{k-1}}^{t_k} \E \left[ 2M + C_2 J_t \right] \; \mathrm{d} t \\
    &\leq 2M |t_k - t_{k-1}| + C_2 \int_{t_{k-1}}^{t_k} |t_k - t| \cdot \sup_x |\hat R^\theta_{t_k}(x,x)| \; \mathrm{d} t \\
    &\leq 2M\gamma_k + \frac{1}{2}C_2|\hat R^\theta_{t_k}|\tau_k^2 \\
    &\leq 2 M \tau_k + \frac{1}{2} C_2(M + C_1 SD|R|) \tau_k^2
\end{align}
where to bound $\mathbb{E}[J_t]$ we have observed that jumps of $(x_t)$ occur at a rate bounded above by $\sup_x |\hat R^\theta_{t_k}(x,x)|$, and in the last line we have used the condition (\ref{eq:error_bound_condition}) and Assumption \ref{ass:bounded_ratios}. Since the above holds for any choice of $x_{t_k}$, it follows that
\begin{equation}
    \sup_{\nu} ||\nu \mathcal{R}^{\theta}_k - \nu \mathcal{R}^{\theta, (\tau)}_k||_{\textup{TV}} \leq 2 M \tau_k + \frac{1}{2} C_2(M + C_1 SD|R|) \tau_k^2
\end{equation}
Summing over $k$ and putting all our bounds together, we get
\begin{equation}
\textstyle
    ||\mathcal{L}(y_0) - \pdata||_{\textup{TV}} \leq 3MT + \left\{ \big(|R|SDC_1\big)^2 + \frac{1}{2} C_2(M + C_1 SD|R|) \right\} \tau T + 2 \exp \left \{ - \frac{T \log^2 2}{t_{\textup{mix}} \log 4D} \right \}
\end{equation}
as required.
\end{proof}

\section{Continuous Time ELBO Details} \label{sec:ApdxCTELBODetails}

\subsection{Comparison with the Discrete Time ELBO} \label{sec:ApdxCTDTELBOComparison}
It is easiest to gain intuition on the $\LCT$ objective by comparing it to its discrete time counterpart, $\LDT$, and examining the way in which $\LDT$ in the limit becomes $\LCT$ when we take the time step size to be very small. We repeat the definition of $\LCT$ here for convenience
\begin{equation}
   \textstyle \LCT(\theta) = T \, \mathbb{E}_{t\sim \mathcal{U}(0, T) q_t(x) r_t(\tilde{x}|x)} \Big[ \Big\{ \sum_{x' \neq x} \hat{R}_t^\theta(x, x') \Big\} - \mathcal{Z}^t(x) \log \left(\hat{R}_t^\theta(\tilde{x}, x) \right) \Big] + C.
\end{equation}
Recall that a single term from the KL sum in $\LDT$ up to an additive constant independent of $\theta$ is
\begin{equation}
  \textstyle -\mathbb{E}_{q_k(x_k) q_{k+1|k}(x_{k+1} | x_k)} \left[ \log p_{k|k+1}^\theta(x_k|x_{k+1})\right].
\end{equation}
Minimizing this term is to sample $(x_k , x_{k+1})$ from the forward dynamics and then maximize the assigned model probability for the pairing in the reverse direction. A similar idea can be used to understand $\LCT$. First, we write $\log p_{k|k+1}^\theta(x_k | x_{k+1})$ in terms of $\hat{R}_k^\theta$ as
\begin{align}
    \log p_{k|k+1}^\theta(x_k | x_{k+1}) =& \delta_{x_k, x_{k+1}} \left( \hat{R}_k^\theta(x_k, x_k) \dt + o(\dt)  \right)\\
    & \, + ( 1 - \delta_{x_k, x_{k+1}} ) \log \left( \hat{R}_k^\theta(x_{k+1}, x_k) \dt + o(\dt) \right)
\end{align}
where we have separated the cases when $x_k = x_{k+1}$ and when $x_k \neq x_{k+1}$ (see the proof of $\LCT$ for the full details). The first term will become the $\sum_{x' \neq x} \hat{R}_t^\theta(x, x')$ term in $\LCT$ whilst the second term will become the $\mathcal{Z}^t(x) \log \big( \hat{R}_t^\theta(\tilde{x}, x) \big)$ term. Now, when we minimize $\LCT$, we are sampling $(x, \tilde{x})$ from the forward process and then maximizing the assigned model probability for the pairing in the reverse direction, just as in $\LDT$. The slight extra complexity comes from the fact we are considering the case when $x_k = x_{k+1}$ and the case when $x_k \neq x_{k+1}$ separately. When $x_k = x_{k+1}$, this corresponds to the first term in $\LCT$ which we can see is minimizing the reverse rate out of $x$ which is exactly maximizing the model probability for no transition to occur. When $x_k \neq x_{k+1}$, this corresponds to the second term in $\LCT$, which is maximizing the reverse rate from $\tilde{x}$ to $x$ which in turn maximizes the model probability for the $\tilde{x}$ to $x$ transition to occur.

\subsection{Conditional Form} \label{sec:ApdxConditionalObjective}
For the conditional form of $\LCT$, denoted as $\bar{\mathcal{L}}_{\textup{CT}}$, we instead upper bound the negative conditional model log-likelihood, $\E_{\pdata(x_0, y)} [ - \log p_0^\theta(x_0 | y)]$ where $y$ is our conditioner. $\bar{\mathcal{L}}_{\textup{CT}}$ has the following form
\begin{equation}
   \textstyle \bar{\mathcal{L}}_{\textup{CT}}(\theta) = T \, \mathbb{E}_{t\sim \mathcal{U}(0, T) \pdata(x_0, y) q_{t|0}(x|x_0) r_t(\tilde{x}|x)} \Big[ \Big\{ \sum_{x' \neq x} \hat{R}_t^\theta(x, x' | y) \Big\} - \mathcal{Z}^t(x) \log \left(\hat{R}_t^\theta(\tilde{x}, x | y) \right) \Big] + C,
\end{equation}
where
\begin{align}
    \textstyle \hat{R}_t^\theta(x, \tilde{x}| y) =& R_t(\tilde{x}, x) \sum_{x_0} \frac{q_{t|0}(\tilde{x} | x_0)}{q_{t|0}(x | x_0)} p^\theta_{0|t}(x_0 | x, y) \quad \text{for} \quad x \neq \tilde{x}.\\
    =& -\sum_{x' \neq x} \hat{R}_t^\theta(x, x' | y) \quad \text{for} \quad x = \tilde{x}
\end{align}
This follows easily from considering the conditional form of the discret time ELBO, $\bar{\mathcal{L}}_{\textup{DT}}$ and using the same arguments as before to go from discrete time to continuous time.
\begin{equation}
     \E_{\pdata(x_0, y)} [-\log p_0^\theta(x_0 | y)] \leq \E_{\pdata(x_0, y) q_{1:K|0}(x_{1:K}|x_0)} \left[ - \log \frac{p_{0:K}^\theta(x_{0:K}|y)}{q_{1:K|0}(x_{1:K} | x_0)}\right] = \bar{\mathcal{L}}_{\textup{DT}}
\end{equation}

\subsection{Continuous Time ELBO with Factorization Assumptions} \label{sec:ApdxCTELBOFactorization}
\newcommand{\done}{\psi_t} 
\newcommand{\dtwo}{\phi_t} 
In the following Proposition, we show the form of $\LCT$ when we use a factorized forward process. We note that in the proof we rearrange the sampling distribution from $\pdata(\bm{x}_0^{1:D}) q_{t|0}(\bm{x}^{1:D} | \bm{x}_0^{1:D}) r_t(\tilde{\bm{x}}^{1:D} | \bm{x}^{1:D})$ to $\pdata(\bm{x}_0^{1:D}) \done(\tilde{\bm{x}}^{1:D} | \bm{x}_0^{1:D}) \dtwo(\bm{x}^{1:D} | \tilde{\bm{x}}^{1:D}, \bm{x}_0^{1:D})$. This is not strictly necessary but it allows us to analytically sum over the intermediate $\bm{x}^{1:D}$ variable which greatly reduces the variance of the resulting objective.

\begin{proposition} \label{prop:FactorizedELBO}
The $\LCT$ objective when we substitute in the factorized forms for the forward and reverse process given in Proposition \ref{prop:dimfactorize} is
\begin{align}
    \LCT &= T \, \E_{t\sim \mathcal{U}(0,T) \pdata(\bm{x}_0^{1:D}) q_{t|0}(\bm{x}^{1:D} | \bm{x}_0^{1:D})} \left[\sum_{d=1}^D \sum_{x'^d \neq x^d} \hat{R}_t^{\theta \, d}(\bm{x}^{1:D}, x'^d) \right]\\
    & - T \, \E_{t\sim \mathcal{U}(0, T) \pdata(\bm{x}_0^{1:D}) \done(\tilde{\bm{x}}^{1:D} | \bm{x}_0^{1:D})} \Bigg[ \sum_{d=1}^D \sum_{x^d \neq \tilde{x}^d} \dtwo(x^d | \tilde{\bm{x}}^{1:D}, \bm{x}_0^{1:D}) \mathcal{Z}^t(\tilde{\bm{x}}^{1:D/d} \circ x^d) \log \left( \hat{R}_t^{\theta \, d}(\tilde{\bm{x}}^{1:D}, x^d) \right) \Bigg]\\
    & + C
\end{align}
with
\begin{equation}
    \hat{R}_t^{\theta \, d} (\bm{x}^{1:D}, \tilde{x}^d) = R_t^d(\tilde{x}^d, x^d) \sum_{x_0^d} p_{0|t}^\theta (x_0^d | \bm{x}^{1:D}) \frac{q_{t|0}(\tilde{x}^d | x_0^d)}{q_{t|0}(x^d | x_0^d)}
\end{equation}
\begin{equation}
    \mathcal{Z}^t(\bm{x}^{1:D}) = \sum_{d=1}^D \sum_{\tilde{x}^d \neq x^d} R_t^d(x^d, \tilde{x}^d)
\end{equation}
\begin{equation}
    \dtwo(x^d | \tilde{\bm{x}}^{1:D}, \bm{x}_0^{1:D}) = \frac{R_t^d(x^d, \tilde{x}^d)  q_{t|0}(\tilde{\bm{x}}^{\Dd} \circ x^d | \bm{x}_0^{1:D})}{\mathcal{Z}^t(\tilde{\bm{x}}^{\Dd} \circ x^d  )   \sum_{d'=1}^D \sum_{x'^{d'} \neq \tilde{x}^{d'}} \frac{R_t^{d'}(x'^{d'}, \tilde{x}^{d'})}{\mathcal{Z}^t(\tilde{\bm{x}}^{1:D \backslash {d'}} \circ x'^{d'}  )} q_{t|0}(\tilde{\bm{x}}^{1:D \backslash d'} \circ x'^{d'} | \bm{x}_0^{1:D})}
\end{equation}
where $\circ$ represents the concatenation of a $D-1$ dimensional vector, $\bm{x}^{\Dd}$ with a scalar $x^d$, such that the resultant $D$ dimensional vector has $x^d$ at its $d^{\textup{th}}$ dimension.
$\done(\tilde{\bm{x}}^{1:D} | \bm{x}_0^{1:D})$ is defined as the marginal of the forward noising process joint, $\int q_{t|0}(\bm{x}^{1:D} | \bm{x}_0^{1:D}) r_t(\tilde{\bm{x}}^{1:D} | \bm{x}^{1:D}) d \bm{x}^{1:D}$.
\end{proposition}
\begin{proof}
We first re-write the general form of $\LCT$ here
\begin{equation}
   \textstyle \LCT(\theta) = T \, \mathbb{E}_{t\sim \mathcal{U}(0, T) q_t(x) r_t(\tilde{x}|x)} \Big[ \Big\{ \sum_{x' \neq x} \hat{R}_t^\theta(x, x') \Big\} - \mathcal{Z}^t(x) \log \left(\hat{R}_t^\theta(\tilde{x}, x) \right) \Big] + C
\end{equation}
where
\begin{equation}
    \textstyle \mathcal{Z}^t(x) = \sum_{x' \neq x} R_t(x, x') \hspace{2cm} r_t(\tilde{x}|x) = (1 - \delta_{\tilde{x}, x} ) R_t(x, \tilde{x})/\mathcal{Z}^t(x).
\end{equation}
With a factorized forward process, $\hat{R}_t^\theta$ becomes
\begin{equation}
    \hat{R}_t^{\theta \, 1:D}(\bm{x}^{1:D}, \tilde{\bm{x}}^{1:D}) = \sum_{d=1}^D \hat{R}_t^{\theta \, d}(\bm{x}^{1:D}, \tilde{x}^d) \delta_{\bm{x}^{\Dd}, \tilde{\bm{x}}^{\Dd}}
\end{equation}
where
\begin{equation}
    \hat{R}_t^{\theta \, d} (\bm{x}^{1:D}, \tilde{x}^d) = R_t^d(\tilde{x}^d, x^d) \sum_{x_0^d} p_{0|t}^\theta (x_0^d | \bm{x}^{1:D}) \frac{q_{t|0}(\tilde{x}^d | x_0^d)}{q_{t|0}(x^d | x_0^d)}
\end{equation}
Substituting this form for $\hat{R}_t^{\theta \, 1:D}$ into the first term in $\LCT$ we get
\begin{align}
    &\sum_{\bm{x}'^{1:D} \neq \bm{x}^{1:D}} \sum_{d=1}^D \hat{R}_t^{\theta \, d}(\bm{x}^{1:D}, x'^d) \delta_{\bm{x}^{\Dd}, \bm{x}'^{\Dd} }\\
    &= \sum_{d=1}^D \sum_{x'^d} \hat{R}_t^{\theta \, d}(\bm{x}^{1:D}, x'^d) \sum_{\bm{x}'^{\Dd}} \delta_{\bm{x}^{\Dd}, \bm{x}'^{\Dd}} ( 1 - \delta_{\bm{x}'^{1:D}, \bm{x}^{1:D}})\\
    &= \sum_{d=1}^D \sum_{x'^d \neq x^d} \hat{R}_t^{\theta \, d}(\bm{x}^{1:D}, x'^d)
\end{align}

Now we tackle the second term in $\LCT$. We first re-arrange the distribution over which we take the expectation:
\begin{equation}
    \pdata(\bm{x}_0^{1:D}) q_{t|0}(\bm{x}^{1:D} | \bm{x}_0^{1:D}) r_t(\tilde{\bm{x}}^{1:D} | \bm{x}^{1:D}) = \pdata(\bm{x}_0^{1:D}) \done(\tilde{\bm{x}}^{1:D} | \bm{x}_0^{1:D}) \dtwo(\bm{x}^{1:D} | \tilde{\bm{x}}^{1:D}, \bm{x}_0^{1:D}) 
\end{equation}
We have,
\begin{align}
    \dtwo(\bm{x}^{1:D} | \tilde{\bm{x}}^{1:D}, \bm{x}_0^{1:D}) &\propto q_{t|0}(\bm{x}^{1:D} | \bm{x}_0^{1:D}) r_t(\tilde{\bm{x}}^{1:D} | \bm{x}^{1:D})\\
    & = q_{t|0}(\bm{x}^{1:D} | \bm{x}_0^{1:D}) ( 1- \delta_{\tilde{\bm{x}}^{1:D}, \bm{x}^{1:D}}) \frac{\sum_{d=1}^D R_t^d(x^d, \tilde{x}^d) \delta_{\bm{x}^{\Dd}, \tilde{\bm{x}}^{\Dd}}}{\mathcal{Z}^t(\bm{x}^{1:D})}\\
    & = \sum_{d=1}^D \frac{R_t^d(x^d, \tilde{x}^d)}{\mathcal{Z}^t(\tilde{\bm{x}}^{\Dd} \circ x^d  )} q_{t|0}(\tilde{\bm{x}}^{\Dd} \circ x^d | \bm{x}_0^{1:D})\delta_{\bm{x}^{\Dd}, \tilde{\bm{x}}^{\Dd}} ( 1- \delta_{\tilde{\bm{x}}^{1:D}, \bm{x}^{1:D}})  \\
\end{align}
To find the normalization constant, we can sum the proportional term over $\bm{x}^{1:D}$
\begin{align}
    &\sum_{\bm{x}^{1:D}} \sum_{d=1}^D \frac{R_t^d(x^d, \tilde{x}^d)}{\mathcal{Z}^t(\tilde{\bm{x}}^{\Dd} \circ x^d  )} q_{t|0}(\tilde{\bm{x}}^{\Dd} \circ x^d | \bm{x}_0^{1:D})  \delta_{\bm{x}^{\Dd}, \tilde{\bm{x}}^{\Dd}} ( 1- \delta_{\tilde{\bm{x}}^{1:D}, \bm{x}^{1:D}})\\
    & \quad = \sum_{d=1}^D \sum_{x^d \neq \tilde{x}^d} \frac{R_t^d(x^d, \tilde{x}^d)}{\mathcal{Z}^t(\tilde{\bm{x}}^{\Dd} \circ x^d  )} q_{t|0}(\tilde{\bm{x}}^{\Dd} \circ x^d | \bm{x}_0^{1:D})
\end{align}
Therefore,
\begin{equation}
    \dtwo(\bm{x}^{1:D} | \tilde{\bm{x}}^{1:D}, \bm{x}_0^{1:D}) = (1-\delta_{\tilde{\bm{x}}^{1:D}, \bm{x}^{1:D}}) \sum_{d=1}^D \dtwo(x^d | \tilde{\bm{x}}^{1:D}, \bm{x}_0^{1:D}) \delta_{\bm{x}^{\Dd}, \tilde{\bm{x}}^{\Dd}}
\end{equation}
where
\begin{equation}
    \dtwo(x^d | \tilde{\bm{x}}^{1:D}, \bm{x}_0^{1:D}) = \frac{R_t^d(x^d, \tilde{x}^d)  q_{t|0}(\tilde{\bm{x}}^{\Dd} \circ x^d | \bm{x}_0^{1:D})}{\mathcal{Z}^t(\tilde{\bm{x}}^{\Dd} \circ x^d  )   \sum_{d'=1}^D \sum_{x'^{d'} \neq \tilde{x}^{d'}} \frac{R_t^{d'}(x'^{d'}, \tilde{x}^{d'})}{\mathcal{Z}^t(\tilde{\bm{x}}^{1:D \backslash {d'}} \circ x'^{d'}  )} q_{t|0}(\tilde{\bm{x}}^{1:D \backslash d'} \circ x'^{d'} | \bm{x}_0^{1:D})}
\end{equation}
Now we write the second term as
\begin{align}
    &T \, \E_{t\sim \mathcal{U}(0, T) \pdata(\bm{x}_0^{1:D}) \done(\tilde{\bm{x}}^{1:D} | \bm{x}_0^{1:D})} \left[ - \sum_{\bm{x}^{1:D}} \dtwo(\bm{x}^{1:D} | \tilde{\bm{x}}^{1:D}, \bm{x}_0^{1:D}) \mathcal{Z}^t(\bm{x}^{1:D}) \log \hat{R}_t^{\theta \, 1:D}(\tilde{\bm{x}}^{1:D}, \bm{x}^{1:D}) \right]\\
   &= - T \, \E_{t\sim \mathcal{U}(0, T) \pdata(\bm{x}_0^{1:D}) \done(\tilde{\bm{x}}^{1:D} | \bm{x}_0^{1:D})} \Bigg[ \sum_{d=1}^D \sum_{x^d \neq \tilde{x}^d} \dtwo(x^d | \tilde{\bm{x}}^{1:D}, \bm{x}_0^{1:D}) \mathcal{Z}^t(\tilde{\bm{x}}^{1:D/d} \circ x^d) \log \left( \hat{R}_t^{\theta \, d}(\tilde{\bm{x}}^{1:D}, x^d) \right) \Bigg]
\end{align}
\end{proof}

\subsection{One Forward Pass} \label{sec:ApdxCTELBOOneForwardPass}

To evaluate the $\LCT$ objective, we naively need to perform two forward passes of the denoising network: $p_{0|t}^\theta(x_0 | x)$ to calculate $\hat{R}_t^\theta(x, x')$ and $p_{0|t}^\theta(x_0 | \tilde{x})$ to calculate $\hat{R}_t^\theta(\tilde{x}, x)$. This is wasteful because $\tilde{x}$ is created from $x$ by applying a single forward transition which on multi-dimensional problems means $\tilde{x}$ differs from $x$ in only a single dimension. To exploit the fact that $\tilde{x}$ and $x$ are very similar, we approximate the sample $x \sim q_t(x)$ with the sample $\tilde{x} \sim \sum_{x} q_t(x) r_t(\tilde{x} | x)$. This gives the more efficient objective,
\begin{equation}
\textstyle
    \LeCT(\theta) = T \, \mathbb{E}_{t \sim \mathcal{U}(0, T) q_t(x) r_t(\tilde{x} | x)} \left[ \left\{ \sum_{x' \neq \tilde{x}} \hat{R}_t^\theta(\tilde{x}, x') \right\} - \mathcal{Z}^t(x) \log \left( \hat{R}_t^\theta(\tilde{x}, x) \right) \right] + C
\end{equation}
The approximation is valid because $q_t(x)$ and $\sum_x q_t(x) r_t(\tilde{x} | x)$ are very similar distributions, as we now show.
\begin{align}
    \sum_x q_t(x) r_t(\tilde{x} | x) &= \sum_x q_t(x) (1 - \delta_{x, \tilde{x}}) \frac{R_t(x, \tilde{x})}{\sum_{x' \neq x} R_t(x, x')}\\
    & \propto - q_t(\tilde{x}) R_t(\tilde{x}, \tilde{x}) + \sum_x q_t(x) R_t(x, \tilde{x}) \\
    & = q_t(\tilde{x}) \sum_{x' \neq \tilde{x}} R_t(\tilde{x}, x') + \partial_t q_t(\tilde{x}) \\
    & \propto q_t(\tilde{x}) + \frac{1}{\sum_{x' \neq x} R_t(\tilde{x}, x')} \partial_t q_t(\tilde{x})\\
    & = q_t(\tilde{x}) + \delta t \, \partial_t q_t(\tilde{x})
\end{align}
where on the third line we have used the Kolmogorov forward equation and defined $\delta_t = 1 / \sum_{x' \neq x} R_t(\tilde{x}, x')$. The distribution $\sum_x q_t(x) r_t(\tilde{x} | x)$ is therefore $q_{t+\delta t}(\tilde{x})$ approximated using a first-order Taylor expansion around $q_t(\tilde{x})$. We notice that $\delta t$ is the average time to the next transition at time $t$. $\delta t$ can be calculated for the practical settings we consider, its varies between $2 \times 10^{-6} T$ and $2 \times 10^{-8} T$ in the image modelling task and is $1 \times 10^{-3} T$ in the monophonic music task.\\

We perform an ablation experiment comparing between training with $\LeCT$ and $\LCT$ on the monophonic music dataset, the results are shown in Table \ref{tab:apdx_two_forward_pass_ablation}. We find that we gain a small boost in performance when using the more efficient $\LeCT$ objective alongside the improved efficiency. We hypothesize that this is because of a slight reduction in variance for the $\LeCT$ objective due to increased negative correlation between the two terms in the objective when $\tilde{x}$ is shared between them.

\begin{table}
  \caption{Metrics on the monophonic music dataset comparing training with the efficient $\LeCT$ objective vs the original $\LCT$ objective. We compute these over the test set showing mean$\pm$std with respect to 5 samples for each test song.}
  \label{tab:apdx_two_forward_pass_ablation}
  \centering
  \begin{tabular}{lll}
    \toprule
    Model  & Hellinger Distance & Proportion of Outliers \\
    \midrule
    $\tau$LDR-0 Uniform $\LeCT$ & $0.3765 \pm 0.0013$ & $0.1106 \pm 0.0010$ \\
    $\tau$LDR-0 Uniform $\LCT$ & $0.3797 \pm 0.0009$ & $0.1128 \pm 0.0007$\\
    \bottomrule
  \end{tabular}
\end{table}

\section{Direct Denoising Model Supervision} \label{sec:ApdxDirectDenoisingSupervision}
Following \cite{austin2021structured}, we can introduce direct $p_{0|t}^\theta$ supervision into the optimization objective which has been found empirically to improve performance. We first contextualize the change by expressing $\LCT$ with the dependence on $p_{0|t}^\theta$ made explicit.
\begin{align}
    \LCT = T \, \E_{t \sim \mathcal{U}(0,T) q_t(x) r_t(\tilde{x}|x)} \Big[ &\Big\{ \sum_{x' \neq x} R_t(x', x) \sum_{x_0} \frac{q_{t|0}(x' | x_0)}{q_{t|0}(x | x_0)} p_{0|t}^\theta(x_0 | x) \Big\}\\
    & - \mathcal{Z}^t(x) \log \Big( R_t(x, \tilde{x}) \sum_{x_0} \frac{q_{t|0}(x | x_0)}{q_{t|0}(\tilde{x} | x_0)} p_{0|t}^\theta (x_0 | \tilde{x}) \Big) \Big] + C
\end{align}
The signal for $p_{0|t}^\theta(x_0 | x)$ comes through a sum over $x_0$ weighted by the ratio $\frac{q_{t|0}(x | x_0)}{q_{t|0}(\tilde{x} | x_0)}$. We can also provide a direct denoising signal by predicting the clean datapoint $x_0$ from the corrupted version $x$ and using the negative log-likelihood loss. 
\begin{equation}
   \textstyle L_{ll}(\theta) = T \, \E_{t \sim \mathcal{U}(0, T) \pdata(x_0) q_{t|0}(x | x_0) }  \left[ - \log p_{0|t}^\theta(x_0 | x) \right]
\end{equation}
\begin{proposition}
The true denoising distribution, $q_{0|t}$, minimizes $L_{ll}$
\end{proposition}
\begin{proof}
\begin{talign}
    & T \, \E_{t \sim \mathcal{U}(0, T) q_t(x)} \left[ \KL\left( q_{0|t}(x_0 | x) \, || \, p_{0|t}^\theta(x_0 | x) \right)\right]\\
    & = T \, \E_{t \sim \mathcal{U}(0, T) \pdata(x_0) q_{t|0}(x | x_0)} \left[ - \log p_{0|t}^\theta(x_0 | x) \right] + C
\end{talign}
where $C$ is a constant independent of $\theta$. Therefore, minimizing $L_{ll}$ is equivalent to minimizing the KL divergence between $q_{0|t}$ and $p_{0|t}^\theta$, which is minimized when $p_{0|t}^\theta = q_{0|t}$.
\end{proof}
If we obtain the true denoising distribution, $p_{0|t}^\theta = q_{0|t}$, then we will have the true reverse rate, $\hat{R}_t$. \cite{austin2021structured} find that optimizing with an objective combining $L_{ll}$ and $\LDT$ performs best, which we can also do in continuous time
\begin{equation}
    \underset{\theta}{\text{min}} \quad \LCT(\theta) + \lambda L_{ll}(\theta)
\end{equation}
where $\lambda$ is a hyperparameter. In \cite{austin2021structured}, it was found that training with $L_{ll}$ alone resulted in poorer performance than when the ELBO was included in the loss. We provide a theoretical hypothesis as to why this may be the case here. We show that minimizing $L_{ll}$ is equivalent to minimizing an upper bound on the negative ELBO in discrete time and thus by training with $L_{ll}$ we are simply minimizing a looser bound on the negative model log-likelihood than if we were to use the negative ELBO directly.

\begin{proposition}
Minimizing the sum of negative log-likelihoods
\begin{equation}
 \sum_{k=0}^{K-1} \E_{\pdata(x_0) q_{k+1|0}(x_{k+1} | x_0)} \left[ - \log p_{0|k+1}^\theta(x_0 | x_{k+1}) \right]
\end{equation}
is equivalent to minimizing an upper bound on the negative ELBO.
\end{proposition}
\begin{proof}

\begin{talign}
    \textstyle \mathcal{L}_{\textup{DT}}(\theta) = \mathbb{E}_{\pdata(x_0)} \Big[& \KL( q_{K|0}(x_K | x_0) || \pref(x_K) ) - \E_{q_{1|0}(x_1 | x_0)} \left[ \log p^\theta_{0|1}(x_0 | x_1)\right] \\
    & + \sum_{k=1}^{K-1} \mathbb{E}_{q_{k+1|0}(x_{k+1} | x_0)} \left[ \KL(q_{k|k+1,0}(x_{k}|x_{k+1}, x_0) || p_{k|k+1}^{\theta}(x_{k}|x_{k+1})) \right] \Big]
\end{talign}

Consider one term from the sum
\begin{align}
    L_k &= \E_{\pdata(x_0) q_{k+1|0}(x_{k+1} | x_0)} \left[ \KL ( q_{k|k+1, 0}(x_{k}|x_{k+1}, x_0) || p_{k|k+1}^\theta(x_{k}|x_{k+1}) \right]\\
    &= \E_{q_{k+1}(x_{k+1})q_{k|k+1}(x_{k}|x_{k+1})} \left[ -\log p_{k|k+1}^\theta(x_{k}|x_{k+1}) \right] \\
    & \quad + \E_{\pdata(x_0) q_{k+1|0}(x_{k+1} | x_0) q_{k|k+1, 0}(x_{k}|x_{k+1} , x_0) } \left[ \log q_{k|k+1, 0}(x_{k}|x_{k+1}, x_0) \right]
\end{align}
Now,
\begin{align}
    &\E_{q_{k+1}(x_{k+1}) q_{k|k+1}(x_{k}|x_{k+1})} \left[ - \log p_{k|k+1}^\theta(x_{k}|x_{k+1}) \right] \\
    & \quad = \E_{q_{k+1}(x_{k+1})q_{k|k+1}(x_{k}|x_{k+1})} \left[ - \log \sum_{\tilde{x}_0} q(x_{k} | \tilde{x}_0, x_{k+1}) p_{0|k+1}^\theta(\tilde{x}_0 | x_{k+1}) \right]\\
    & \quad = \E_{q_{k+1}(x_{k+1})q_{k|k+1}(x_{k}|x_{k+1})} \left[ - \log \sum_{\tilde{x}_0} \frac{q_{0|k}(\tilde{x}_0 | x_{k}) q_{k|k+1}(x_{k}|x_{k+1})}{q_{0|k+1}(\tilde{x}_0 | x_{k+1})} p_{0|k+1}^\theta(\tilde{x}_0 | x_{k+1}) \right]\\
    & \quad \leq \E_{q_{k+1}(x_{k+1}) q_{k|k+1}(x_{k}| x_{k+1}) q_{0|k}(\tilde{x}_0 | x_{k})} \left[ - \log \frac{q_{k|k+1}(x_{k}|x_{k+1})}{q_{0|k+1}(\tilde{x}_0 | x_{k+1})} p_{0|k+1}^\theta(\tilde{x}_0 | x_{k+1}) \right]\\
    & \quad = \E_{\pdata(x_0) q_{k+1|0}(x_{k+1} | x_0)} \left[ -\log p_{0|k+1}^\theta(x_0 | x_{k+1}) \right] \\
    & \hspace{1cm} + \E_{\pdata(x_0) q_{k|0}(x_{k}|x_0) q_{k+1|k}(x_{k+1} | x_{k})} \left[ -\log \frac{q_{k|k+1}(x_{k}|x_{k+1})}{q_{0|k+1}(x_0 | x_{k+1})} \right]
\end{align}
Therefore,
\begin{align}
    \mathcal{L}_{\textup{DT}} \leq& \sum_{k=0}^{K-1} \Bigg\{ \E_{\pdata(x_0) q_{k+1|0}(x_{k+1} | x_0)} \left[ - \log p_{0|k+1}^\theta(x_0 | x_{k+1}) \right]\Bigg\} \\
    &+ \E_{\pdata(x_0)} \left[ \KL ( q_{K|0}(x_K | x_0) || p_K(x_K) ) \right] \\
    &+ \sum_{k=1}^{K-1} \Bigg \{ \E_{\pdata(x_0) q_{k+1|0}(x_{k+1} | x_0) q_{k|k+1, 0}(x_{k} | x_{k+1}, x_0)} \left[ \log q_{k|k+1, 0}(x_{k} | x_{k+1}, x_0) \right] \\
    &\hspace{1.5cm}+ \E_{\pdata(x_0) q_{k|0}(x_{k} | x_0) q_{k+1|k}(x_{k+1} | x_{k})} \left[ - \log \frac{ q_{k|k+1}(x_{k}|x_{k+1})   }{ q_{0|k+1}(x_0 | x_{k+1})   }\right] \Bigg\}
\end{align}
We can see that only the first term depends on $\theta$.
\end{proof}

\section{Choice of Forward Process} \label{sec:ApdxChoiceOfForwardProcess}
We need to choose the structure of $R_t$ such that we can analytically obtain $q_{t|0}$ marginals to enable efficient training.
\begin{proposition}
If $R_t$ and $R_{t'}$ commute for all $t$, $t'$ then $q_{t|0}(x = j | x_0 = i) = \big( \text{exp} \big[ \int_0^t R_s ds \big] \big)_{ij}$ where $\text{exp}$ here is understood to be the matrix exponential function.
\end{proposition}
\begin{proof}
Let $(P_t)_{ij} = q_{t|0}(x = j | x_0 = i)$. We show that $P_t = \text{exp} \left( \int_0^t R_s ds \right)$ is a solution to the Kolmogorov forward equation, which in matrix form reads, $\partial_t P_t = P_t R_t$. Writing the matrix exponential in sum form
\begin{align}
    P_t &= \sum_{k=0}^\infty \frac{1}{k!} \left( \int_0^t R_s ds \right)^k\\
    &= \mathrm{Id} + \int_0^t R_s ds + \frac{1}{2!} \left( \int_0^t R_s ds \right)^2 + \frac{1}{3!} \left( \int_0^t R_s ds \right)^3 + \dots
\end{align}
Now, differentiating and using the fact that $R_t$, $R_{t'}$ commute.
\begin{align}
    \partial_t P_t &= R_t + \int_0^t R_s ds R_t + \frac{1}{2!} \left( \int_0^t R_s ds \right)^2 R_t + \dots\\
    &= \left\{ \sum_{k=0}^\infty \frac{1}{k!} \left( \int_0^t R_s ds \right)^k \right\} R_t\\
    &= P_t R_t
\end{align}
\end{proof}
As stated in the main text, we achieve the commutative property by selecting $R_t = \beta(t) R_{b}$ where $\beta(t)$ is a time dependent scalar and $R_b$ is a constant base matrix. We can utilize the eigendecomposition of $R_b = Q \Lambda Q^{-1}$ to efficiently calculate $P_t$, 
\begin{align}
    P_t &= \text{exp} \left( \int_0^t \beta(s) R_b ds \right)\\
    &= \sum_{k=0}^\infty \frac{1}{k!} \left( \int_0^t \beta(s) R_b ds \right)^k\\
    &= \sum_{k=0}^\infty \frac{1}{k!} \left(Q \Lambda Q^{-1} \int_0^t \beta(s) ds \right)^k\\
    &= \sum_{k=0}^\infty \frac{1}{k!} Q \left(\Lambda \int_0^t \beta(s) ds \right)^k Q^{-1}\\
    &= Q \left\{ \sum_{k=0}^\infty \frac{1}{k!} \left(\Lambda \int_0^t \beta(s) ds \right)^k \right\} Q^{-1}\\
    &= Q \text{exp} \left[ \Lambda \int_0^t \beta(s) ds \right] Q^{-1}
\end{align}
Since $\Lambda$ is a diagonal matrix, the matrix exponential coincides with the element wise exponential making the final expression tractable to compute. We choose $\beta(t) = a b^t \log b$ because this makes the integral which dictates the variance of $q_{t|0}$ have a simple form $\int_0^t \beta(s) ds = a b^t - a$.\\

For categorical problems, we found a uniform rate matrix works well, $R_t = \beta \mathbbm{1} \mathbbm{1}^T - \beta S \mathrm{Id}$.
This is directly analogous to the discrete time uniform transition matrix: $P = \alpha \mathbbm{1}\mathbbm{1}^T + (1 - S \alpha) \mathrm{Id}$ with $\alpha$ depending on the time discretization used. Indeed, if one calculates the corresponding discrete transition matrix for the uniform $R_t$ rate through the matrix exponential, the uniform transition matrix is obtained. Another categorical corruption process is the absorbing state process. In discrete time, the transition matrix is given by $P = \alpha \mathbbm{1} \mathbf{e}_\ast^T + (1-\alpha) \mathrm{Id}$ where $\mathbf{e}_\ast$ is the one-hot encoding of the absorbing state. The corresponding absorbing state continuous time process has transition rate matrix: $R_t = \beta \mathbf{1} \mathbf{e}_\ast^T - \beta \mathrm{Id}$. The correspondence for more complex transition matrices e.g. the Discretized Gaussian matrix in \cite{austin2021structured} is much harder to find analytically especially if the time inhomogeneous case is considered. For datasets with an ordinal structure, we construct a new rate matrix that maintains a bias towards nearby states using a similar approach as that taken by \cite{austin2021structured} to construct the Discretized Gaussian matrix.

We construct this matrix by first picking a desired stationary distribution, $\pref$, and then filling in matrix entries such that we encourage transitions to nearby states whilst keeping $\pref$ as our stationary distribution. Specifically, we let $\pref$ be a discretized Gaussian over the state space, i.e.
\begin{equation}
    \pref(x) \propto \text{exp} \left[ - \frac{(x - \mu_0)^2}{2 \sigma_0^2} \right]
\end{equation}
To find a condition on the rate such that this is the case, recall the Kolmogorov differential equation for the marginals
\begin{equation}
    \partial_t q_t(x) = \sum_{\tilde{x}} q_t(\tilde{x}) R_b(\tilde{x}, x)
\end{equation}
Now, consider a rate that is in detailed balance with $\pref$
\begin{equation}
    \pref(\tilde{x}) R_b(\tilde{x}, x) = \pref(x) R_b(x, \tilde{x})
\end{equation}
Substituting this rate into the Kolmogorov equation, we see that $\pref$ is the stationary distribution
\begin{align}
    \partial_t \pref(x) &= \sum_{\tilde{x}} \pref(\tilde{x}) R_b(\tilde{x}, x)\\
    &= \sum_{\tilde{x}} \pref(x) R_b(x, \tilde{x})\\
    &= \pref(x) \sum_{\tilde{x}} R_b(x, \tilde{x})\\
    &= 0
\end{align}
where the last line follows from the fact that the row sum of a rate matrix is zero. Note that any $R_t = \beta(t) R_b$ will also have this stationary distribution as the multiplication by $\beta(t)$ can be seen as just a scaling of the time axis. From the detailed balance equation, we gain a condition on $R_b$ such that our desired $\pref$ is the stationary distribution
\begin{equation}
    \frac{R_b(\tilde{x}, x)}{R_b(x, \tilde{x})} = \frac{\pref(x)}{\pref(\tilde{x})} = \text{exp} \left[ \frac{(\tilde{x} - \mu_0)^2}{2 \sigma_0^2} - \frac{(x - \mu_0)^2}{2 \sigma_0^2} \right]
\end{equation}
This gives constraints on diagonal elements within $R_b$ but does not fully define the entire matrix. To do this, we first make the assumption that $\mu$ is selected to be at the center of the state space. Then we set off diagonal terms to the right of the diagonal in the top half of the rate matrix and off diagonal terms to the left of the diagonal in the bottom half to be 1. Finally, progressing in from the top and bottom of the rate matrix we make definitions of rate matrix values that have not already been defined by the detailed balance condition.
For clarity, we provide a pictorial representation of this scheme for an $8 \times 8$ rate matrix below

\newcommand{\syma}{\square}
\newcommand{\symb}{\triangle}
\newcommand{\symc}{\cdot}
\begin{equation}
    \begin{bmatrix}
        \symc&1&\syma&\syma&\syma&\syma&\syma&\syma\\
        \symb&\symc&1&\syma&\syma&\syma&\syma&\symb\\
        \symb&\symb&\symc&1&\syma&\syma&\symb&\symb\\
        \symb&\symb&\symb&\symc&1&\symb&\symb&\symb\\
        \symb&\symb&\symb&\symb&\symc&\symb&\symb&\symb\\
        \symb&\symb&\symb&\syma&1&\symc&\symb&\symb\\
        \symb&\symb&\syma&\syma&\syma&1&\symc&\symb\\
        \symb&\syma&\syma&\syma&\syma&\syma&1&\symc\\
    \end{bmatrix}
\end{equation}
where $\syma$ represents a value we will define, $\symb$ represents a value that is defined relative to another entry through the detailed balance condition and $\symc$ is a diagonal entry that is equal to the negative off diagonal row sum. We could define $\syma$ values to be 0 to gain a sparse rate matrix, however, we found in early experiments that allowing transitions to further away states greatly reduces the mixing time and gives better performance. We define $\syma$ in each row similarly, by setting it equal to $\text{exp}[- i^2 / \sigma_r^2 ]$ where $i$ is the distance away from the `$1$' value in that row and $\sigma_r$ is a hyperparameter defining the length scale in state space of a typical transition. This biases our forward process to make transitions between nearby states, at a length scale of $\sigma_r$.

\section{CTMC Simulation} \label{sec:ApdxCTMCSimulation}
\subsection{Exact CTMC and Tau-Leaping}
In this section, we first describe exact CTMC simulation before giving an algorithmic description of tau-leaping.\\

When a CTMC has a time-homogeneous rate matrix, we can use Gillespie's Algorithm \cite{gillespie1976general, gillespie1977exact,wilkinson2018stochastic} to exactly simulate it. This algorithm is based on the jump chain/holding time definition of the CTMC. It repeats the following two steps:
\begin{itemize}
    \item Draw a holding time from an exponential distribution with mean $-1 / R(x, x)$ and wait in the current state $x$ for that amount of time.
    \item Sample the next state from $r(\tilde{x} | x) = (1 - \delta_{x, \tilde{x}}) \frac{R(x, \tilde{x})}{\sum_{x' \neq x} R(x, x')}$
\end{itemize}
This Algorithm can be adjusted for the case when we have a time-inhomogeneous rate matrix using the modified next reaction method \cite{anderson2007modified}. However, both algorithms still step through each transition in the CTMC individually and are thus unsuitable in our case because only one dimension would change for each simulation step making it very computationally expensive to produce a sample. Instead we use tau-leaping that allows multiple dimensions to change in a single simulation step. We detail this method in Algorithm \ref{alg:tauleaping}.

\begin{algorithm}[t]
\SetAlgoNoLine
\DontPrintSemicolon
$t \leftarrow T$\;
$\bm{x}_t^{1:D} \sim \pref(\bm{x}_T^{1:D})$\;
\While{ $t > 0$}{
Compute $p_{0|t}^\theta(x_0^d | \bm{x}_t^{1:D}), d=1,\dots,D$ with one forward pass of the denoising network\;
\For{$d=1, \dots, D$}{
    \For{$s = 1, \dots,  S \backslash x_t^d$}{
    $\hat{R}_t^{\theta \, d}(\bm{x}_t^{1:D}, s ) \leftarrow R_t^d(s, x_t^d) \sum_{x_0^d} p^\theta_{0|t}(x_0^d | \bm{x}_t^{1:D}) \frac{q_{t|0}(s | x_0^d)}{q_{t|0}(x_t^d | x_0^d)}$\;
        $P_{ds} \leftarrow$ Poisson$\Big(\tau \hat{R}_t^{\theta \, d}(\bm{x}_t^{1:D}, s)\Big)$\;
    }
}
\For{$d = 1, \dots, D$}{
\eIf{data is categorical AND $\sum_{s=1}^S P_{ds} > 1$    }{
$x_{t-\tau}^d \leftarrow x_t^d$ \tcp{reject change}
}{
 $x_{t-\tau}^d \leftarrow x_t^d + \sum_{s=1}^S P_{ds} \times (s - x_t^d)$
}
}
$\bm{x}_{t-\tau}^{1:D} \leftarrow \text{Clamp}(\bm{x}_{t-\tau}^{1:D}, \text{min}=1, \text{max}=S)$\;
$t \leftarrow t - \tau$
}
 \caption{Generative Reverse Process Simulation with Tau-Leaping}
 \label{alg:tauleaping}
\end{algorithm}

\subsection{Predictor-Corrector Discussion} \label{sec:ApdxPredictorCorrectorDiscussion}

In this section we compare predictor-corrector sampling schemes as applied to continuous state spaces and discrete state spaces.

The predictor-corrector scheme in continuous state spaces was introduced in \cite{song2020score}. It consists of alternating between a predictor step and a corrector step:
\begin{align}
    &\text{Predictor} \quad \bm{x}_i \leftarrow \bm{x}_{i+1} + \gamma_i s_\theta(\bm{x}_{i+1}, i+1) + \sqrt{\gamma_i} z, \quad z \sim \mathcal{N}(0, I)\\
    &\text{Corrector} \quad \bm{x}_i \leftarrow \bm{x}_{i} + \epsilon_i s_\theta(\bm{x}_i, i) + \sqrt{2 \epsilon_i} z, \quad z \sim \mathcal{N}(0, I)
\end{align}
where $\bm{x}_i$ is the state at sampling step $i$, $s_\theta$ is the learned score model approximating $\nabla_{\bm{x}} \log q_t(\bm{x})$ and $\gamma_i$, $\epsilon_i$ are the step sizes for the predictor and corrector respectively. We see that both take similar forms, except the corrector adds in a factor $\sqrt{2}$ more Gaussian noise during the update step.

In discrete state spaces, rather than sampling using gradient guided stochastic steps as in the continuous state space case, we sample by simulating CTMCs with defined rates. When we take a predictor step, we simulate using $\hat{R}_t^\theta$ and when we take a corrector step we simulate using $R_t^{c \, \theta} = \hat{R}_t^\theta + R_t$. If we simulate the CTMC exactly, we have seen in the previous section that this amounts to sampling next states from the categorical distribution defined by normalizing the row of the rate matrix corresponding to the current state. Therefore, corrector sampling can be seen as sampling from a slightly noisier categorical distribution defined through $R_t^{c \, \theta}$ as compared to the predictor categorical distribution defined through $\hat{R}_t^\theta$. This is analogous to the increased Gaussian noise applied during a corrector step in continuous state spaces.

Adding corrector steps brings the marginal of the samples closer to $q_t(\bm{x})$ and continued application of the corrector will further explore the domain of $q_t(\bm{x})$. In previous work on continuous state predictor-corrector methods, the number of corrector steps has been small (e.g. 1 or 2 corrector steps per predictor step) or indeed the corrector steps have been removed altogether. In this work we have found that using up to 10 corrector steps per predictor steps can be beneficial during certain regions of the reverse generative process. Additionally, in continuous state spaces, it has been observed that too many corrector steps can result in unwanted noise in the generated data \cite{furusawa2021generative}.

We hypothesize that corrector steps are better utilized in discrete state spaces to explore the domain of $q_t(\bm{x})$ than in continuous state spaces. This is because, the corrector update is defined largely through the reverse rate itself, $\hat{R}_t^\theta$, just with the categorical probabilities being annealed slightly more towards uniform through the addition of the forward rate $R_t$. This may be a more effective update than simply adding extra Gaussian noise in the continuous state space case. Furthermore, the denoising model in continuous state spaces can be seen as outputting a point estimate of $\bm{x}_0$ of dimension $D$. However, in discrete state spaces, the denoising model outputs a categorical distribution over every dimension (output dimension $D \times S$) allowing it to express some uncertainty information in the $\bm{x}_0$ prediction, albeit with conditional independence between the dimensions. Adding corrector steps in discrete state spaces would then allow information to mix between dimensions for the current time step, exploring modes of $q_t(\bm{x})$.

We explore this idea on the image modelling task in Figure \ref{fig:apdxCorrectorProgressionImages}. We run the reverse generative process until time $t=0.4$ at which point we hold the time constant and apply $1000$ corrector steps. We see that the resulting progression of $\bm{x}_t$ states explores potential local modes of $q_t(\bm{x})$ in the local region of image space.

\begin{figure}
    \centering
    \includegraphics[width=\linewidth]{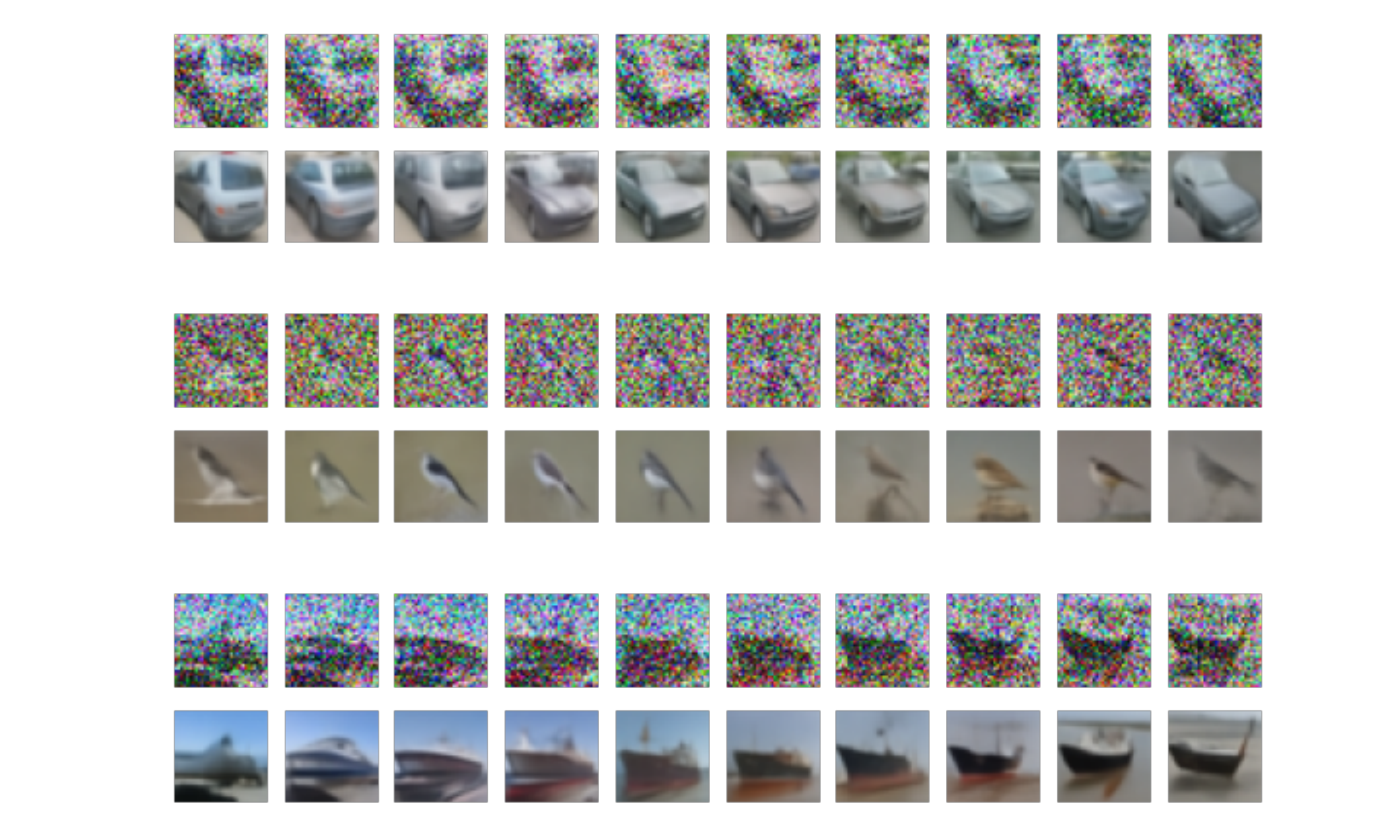}
    \caption{Progression of $\bm{x}_t$ for $t=0.4$ by repeated application of corrector steps. In each pair of rows, the top row is $\bm{x}_t$ whilst the bottom row is the $\bm{x}_0$ prediction made by $p^\theta_{0|t}(\bm{x}_0 | \bm{x}_t)$ (argmax of the categorical probabilities in each dimension). Each column represents an additional $100$ corrector steps.}
    \label{fig:apdxCorrectorProgressionImages}
\end{figure}

\section{Implicit Dimensional Assumptions Made in Discrete Time} \label{sec:ApdxDTDimAssumptions}
In discrete time, the parametric reverse kernel, $p_{k|k+1}^\theta$, is commonly defined through a denoising model $p_{0|k+1}^\theta$. Here, we examine this definition in the multi-dimensional case where the forward process factorizes, as in Appendix \ref{sec:ApdxCTELBOFactorization} and previous discrete time work \cite{austin2021structured}. We begin by writing the true full dimensional reverse kernel, $q_{k|k+1}$, in terms of the true denoising distribution, $q_{0|k+1}$.
\begin{align}
    q_{k|k+1}(\bm{x}_{k}^{1:D} | \bm{x}_{k+1}^{1:D}) &= \prod_{d=1}^D q_{k|k+1}(x_{k}^d | \bm{x}_{k}^{1:d-1}, \bm{x}_{k+1}^{1:D})\\
    & = \prod_{d=1}^D \sum_{x_0^d} q_{k,0|k+1}(x_{k}^d, x_0^d | \bm{x}_{k}^{1:d-1}, \bm{x}_{k+1}^{1:D})\\
    & = \prod_{d=1}^D \sum_{x_0^d} q_{0|k+1}(x_0^d | \bm{x}_{k}^{1:d-1}, \bm{x}_{k+1}^{1:D}) q_{k|0, k+1}(x_{k}^d | x_0^d, x_{k+1}^d)
\end{align}
where on the final line we have used the fact that the forward process is independent across dimensions. To create our approximate reverse kernel, $p_{k|k+1}^\theta$, we approximate $q_{0|k+1}(x_0^d | \bm{x}_k^{1:d-1}, \bm{x}_{k+1}^{1:D})$ with $p_{0|k+1}^\theta(x_0^d | \bm{x}_{k+1}^{1:D})$,
\begin{equation}
    p^\theta_{k|k+1}(\bm{x}_{k}^{1:D} | \bm{x}_{k+1}^{1:D}) = \prod_{d=1}^D \sum_{x_0^d} p^\theta_{0|k+1}(x_0^d | \bm{x}_{k+1}^{1:D}) q_{k|0,k+1}(x_{k}^d | x_0^d, x_{k+1}^d)
\end{equation}
We throw away the extra $\bm{x}_k^{1:d-1}$ conditioning because we use a non-autoregressive model that takes in $\bm{x}_{k+1}^{1:D}$ and in a single forward pass gives conditionally independent probabilities over $x_0^d$, $d = 1, \dots, D$. For finite $K$, this approximation can never match the true kernel because we are not conditioning on all relevant information. Of course, as $K$ gets larger, this approximation becomes more accurate. Since we operate in the continuous regime, we do not have to make this approximation because the conditionally independent denoising model, $q_{0|t}(x_0^d | \bm{x}^{1:D})$, appears directly in our reverse rate, $\hat{R}_t^{1:D}$, when we factorize the forward process (see Proposition \ref{prop:dimfactorize}).

\section{Experimental Details} \label{sec:ApdxExperimentDetails}
In this section, we provide additional details for the experiments we performed applying our method to practical problems. The code for our models is available at \url{https://github.com/andrew-cr/tauLDR}. Before describing the specifics for each experiment, we first explain the implementation details common to all.\\

When we evaluate the objective $\LCT$ on each minibatch of training datapoints, we must sample a time for each from $t \sim \mathcal{U}(0, T)$ which represents the point in the forward process which we will noise to. Training instabilities can be found if $t$ is sampled very close to $0$ because the reverse rate, $\hat{R}_t$, becomes ill-conditioned in this region. This phenomenon is also observed in continuous state space models because the score, $\nabla_x \log q_t(x)$, becomes ill-conditioned close to $t=0$. The reverse rate and score become ill conditioned close to the start of the forward process because the marginal probability, $q_t(x)$, will be highly peaked around the data manifold and $\log q_t(x)$ will explode in regions that are not close to the data. To avoid these issues, a common trick is to set a minimum time such that $t \sim \mathcal{U}(\epsilon, T)$. $\epsilon$ is set such that the level of noising at $t=\epsilon$ is very small and reverse sampling to this point will produce samples very close to $\pdata$. In our experiments, we set $\epsilon = 0.01T$.\\

During reverse sampling, we use tau-leaping to simulate the reverse process from $t=T$ until $t=\epsilon$ because the reverse rate is not trained for $t < \epsilon$. This produces a sample close to $\pdata$. We found improved performance in metrics such as FID if we then complete a final step to remove the small amount of noise that may still be present in the sample. Specifically, we pass the sample through the denoising model $p_{0|t}^\theta(x_0 | x_t)$ with $t=\epsilon$ to obtain an output of shape $D \times S$ where $D$ is the dimensionality of the problem. This is a probability distribution over the states for each of the dimensions. We set the value of each dimension to the state with the highest probability. This then produces a sample which has all of the noise removed.\\

The specific value of $T$ within our model is arbitrary because the forward process can be scaled in the time axis to provide the same noising process for any $T$. Therefore, we simply set $T=1$.
\subsection{Demonstrative Example} \label{sec:ApdxToyExperiment}
Our 2d dataset is created by sampling 1M 2d points from a $32 \times 32$ state space with probability proportional to the pixel values of a $32 \times 32$ grayscale image of a $\tau$ character.\\

For our forward process, we use a Gaussian rate (see Appendix \ref{sec:ApdxChoiceOfForwardProcess}) with stationary distribution standard deviation $\sigma_0 = 8$ and rate length scale $\sigma_r = 1$. We use a rate schedule of $\beta(t) = 5 \times 5^t \log (5) $.\\

To represent $p_{0|t}^\theta$ we use a residual MLP. The architecture consists of an input linear layer to lift the input dimension of $2$ to the internal network dimension of $16$. Then, there are $2$ residual blocks each consisting of: a single hidden layer MLP of hidden dimension $32$, a residual connection to the input of the MLP, a layer norm, and finally a FiLM layer \cite{perez2018film} modulated by the time embedding. At the output, there is a single linear layer with output size of $2 \times 32 = 64$ representing state probabilities in each of the $2$ dimensions. The time is embedded using the Transformer sinusoidal position embedding \cite{vaswani2017attention} creating an embedding of size $32$. Then, the embedding is further processed by a single hidden layer MLP with hidden layer size $32$ and output size $128$. To create the FiLM parameters in each residual block, the time embedding is passed through a linear layer with output of size $32$ to provide a multiplicative and additive modulation to the state dimension of $16$. We minimize the $\LCT$ objective using Adam with a learning rate of $0.0001$ and batch size of 32 for 1M steps.\\

For the exact simulation we use the next reaction method with modifications for time dependent transition rates \cite{anderson2007modified}. This method steps through each transition in the exact simulation path individually by calculating the time to the next occurrence of each transition type and applying the transition that occurs soonest. Exact algorithmic details can be found in \cite{anderson2007modified}. To calculate the time to the next occurrence for a transition, we need to integrate the reverse rate matrix (eq (13) in \cite{anderson2007modified}). We do this with euler integration with a step size of 0.001.

\subsection{Image Modeling} \label{sec:ApdxCIFAR10Experiment}
We train on the CIFAR10 training dataset that contains 50000 images of dimension $3 \times 32 \times 32$. We evaluate the test ELBO on the  CIFAR10 test dataset which consists of 10000 images. For the forward noising process, we use the the Gaussian rate (see Appendix \ref{sec:ApdxChoiceOfForwardProcess}) with stationary distribution standard deviation of $\sigma_0 = 512$ and rate length scale $\sigma_r = 6$. This effectively defines a uniform stationary distribution since the state space is of size $256$. We found this performs better than a more concentrated Gaussian. Our $\beta$ schedule is $\beta(t) = 3 \times 100^t \log 100$. This was selected in accordance with $\sigma_r$ such that the overall shape of progression of the $q_{t|0}$ variances approximately matches that of the schedule proposed in \cite{ho2020denoising}.

Our $p_{0|t}^\theta$ model is parameterized with the standard U-net \cite{ronneberger2015u} architecture introduced in \cite{ho2020denoising}. The network follows the PixelCNN++ backbone \cite{salimans2017pixelcnn++} with group normalization layers. There are four feature map resolutions ($32 \times 32$ to $4 \times 4$) in the downsampling/upsampling stacks. At each resolution there are two convolutional residual blocks. There is a self-attention block between the residual blocks at the $16 \times 16$ resolution level \cite{chen2018pixelsnail}. The time is input into the network by first embedding with the Transformer sinusoidal position embedding \cite{vaswani2017attention}. This time embedding is passed into each residual block by passing it through a SiLU activation \cite{elfwing2018sigmoid} and then a linear layer before adding it onto the hidden state within the residual block between the two convolution operations.\\

The original architecture of \cite{ho2020denoising} has an output of dimension $3 \times 32 \times 32$ as it makes a point prediction of $x_0$ given $x_t$. In order for the model to output probabilities over $x_0$ (i.e. an output dimension of $3 \times 32 \times 32 \times 256$) we make the adjustments suggested in \cite{austin2021structured}. Specifically, we use their truncated logistic distribution parameterization where the model outputs the mean and log scale of a logistic distribution i.e. an output dimension of $3 \times 32 \times 32 \times 2$. The probability for a state is then the integral of this continuous distribution between this state and the next when mapped onto the real line. To impart a residual inductive bias on the output, the mean of the logistic distribution is taken to be $\text{tanh}(x_t + \mu')$ where $x_t$ is the normalized input into the model and $\mu'$ is mean outputted from the network. The normalization operation takes the input in the range $0, \dots, 255$ and maps it to $[-1, 1]$. In total, our network has approximately $35.7$ million parameters.\\

We optimize with the auxiliary objective described in Appendix \ref{sec:ApdxDirectDenoisingSupervision} with $\lambda = 0.001$. Within the auxiliary objective, we use the one-forward pass version of the continuous time ELBO, $\mathcal{L}_{\textup{eCT}}$. We optimize with Adam for 2M steps with a learning rate of 0.0002 and batch size of 128. We use the standard set of training tricks to improve optimization \cite{ho2020denoising, song2020score}. Throughout training we maintain an exponential moving average of the parameters with decay factor 0.9999. These average parameters are used during testing. At the start of optimization we use a linear learning rate warm-up for the first 5000 steps. We clip the gradient norm at a norm value of 1.0. We set the dropout rate for the network at 0.1. The skip connections for each residual block are rescaled by a factor of $\frac{1}{\sqrt{2}}$. The input images have random horizontal flips applied to them during training.\\

For sampling in Table \ref{tab:fids} we set $\tau=0.001$ for $\tau$LDR-$0$ and set $\tau=0.002$ for $\tau$LDR-$10$. The 10 corrector steps per predictor steps for $\tau$LDR-$10$ are introduced after $t < 0.1T$. We found that introducing the corrector steps near the end of the reverse sampling process had the best improvement in sample quality for the smallest increase in computational cost. When performing tau-leaping with the corrector rate, $R_t^c$, we have control over what $\tau$ we use since we are sampling a different CTMC (with $q_t$ as its stationary distribution) to the original reverse CTMC. We found that setting the corrector rate $\tau$ to be $1.5$ times the original $\tau$ for the reverse CTMC achieves the best performance in this example.

We train using 4 V100 GPUs on an academic research cluster. To calculate Inception and FID values, we use pytorch-fid \cite{Seitzer2020FID} and a further development \footnote{\url{https:/github.com/w86763777/pytorch-gan-metrics}}. We verified this library produced comparable values to previous work by calculating the Inception and FID scores for the published images from the DDPM \cite{ho2020denoising} method.\\

We show a large array of unconditional samples from the $\tau$LDR-$10$ model in Figure \ref{fig:apdxCIFAR10LotsOfSamples}. We now also present statistics from the reverse sampling process with standard tau-leaping with $\tau=0.001$. Figure \ref{fig:apdxcifarPropDimsJump} shows the proportion of dimensions that transition during a single step of tau-leaping. We see that during the initial stages, every dimensions changes during every tau-leaping step, but nearer the end of the process, more dimensions will have settled in their final positions and the proportion is less. In Figure \ref{fig:apdxcifarPropDimsOOB} we show the proportion of dimensions that are clipped due to proposing an out of bounds jump. Overall, the proportion is small. It is largest at the start of the process when we have initially sampled from the approximately uniform $\pref$ and there will be dimensions close to the boundary. As pixel values settle to their final values, the proportion reduces. Figure \ref{fig:apdxcifarFollowDims} shows the progression of a selection of dimensions during the reverse sampling process. A similar picture emerges where dimensions eventually settle in a region of the state space. We also note that larger jumps are made in a single tau-leaping step nearer the start of the reverse process and smaller jumps are made nearer the end.

\begin{figure}
    \centering
    \includegraphics[trim=300 400 300 400, width=\textwidth]{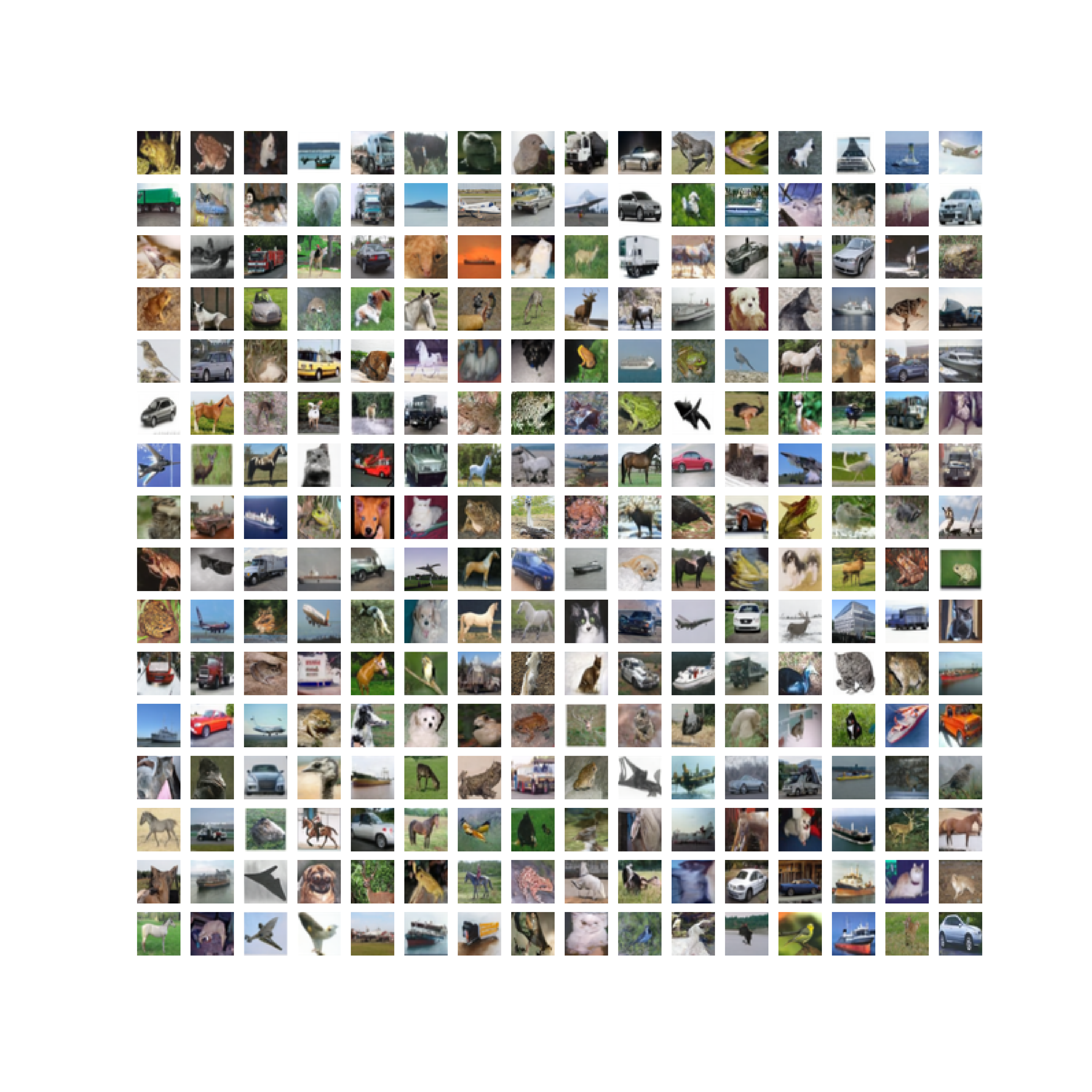}
    \caption{Unconditional CIFAR10 samples from our $\tau$LDR-$10$ model.}
    \label{fig:apdxCIFAR10LotsOfSamples}
\end{figure}

\begin{figure}
    \centering
    \includegraphics[width=0.8\textwidth]{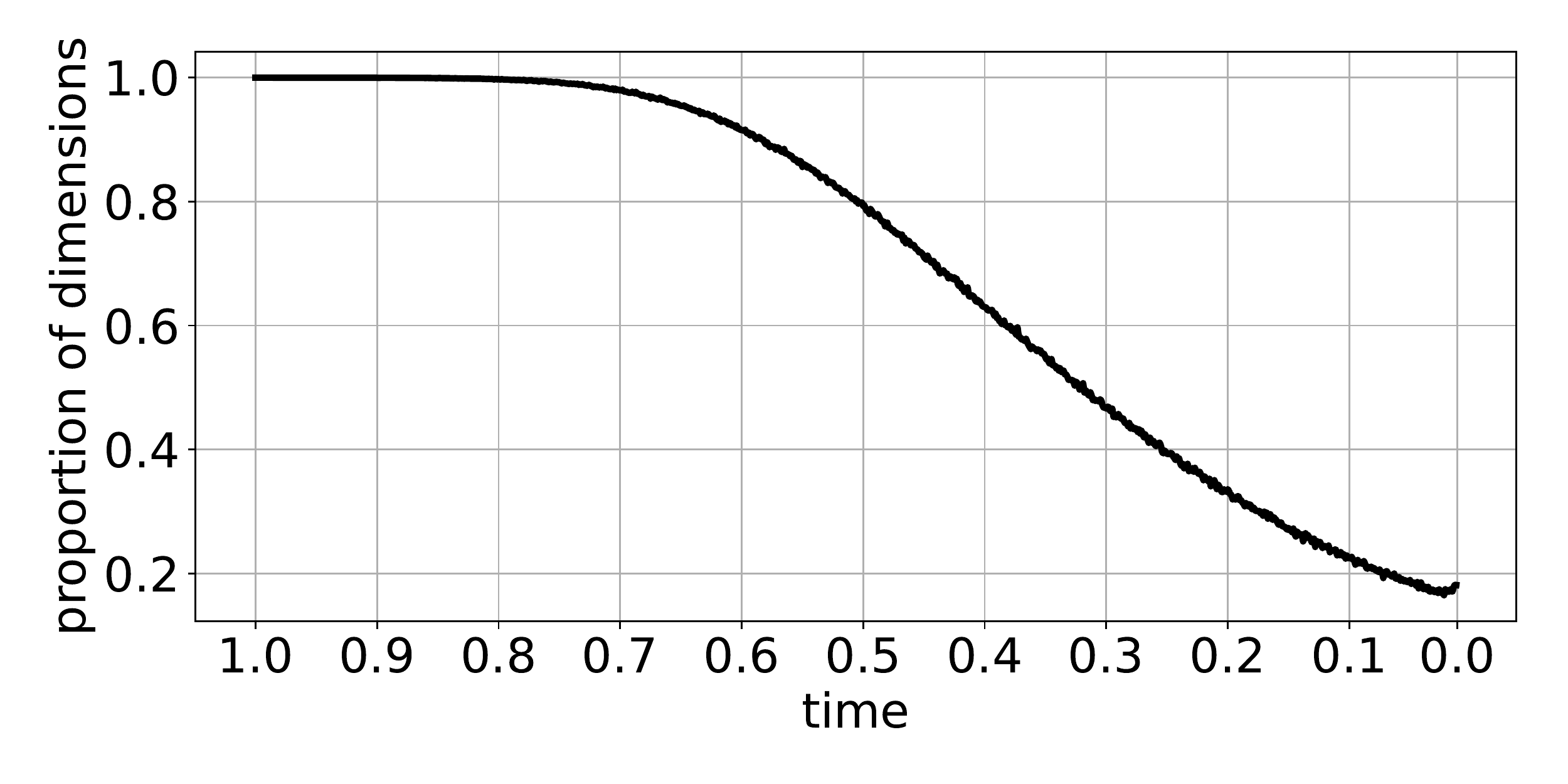}
    \caption{Proportion of dimensions that transition during a single step of tau-leaping during the reverse sampling process.}
    \label{fig:apdxcifarPropDimsJump}
\end{figure}

\begin{figure}
    \centering
    \includegraphics[width=0.8\textwidth]{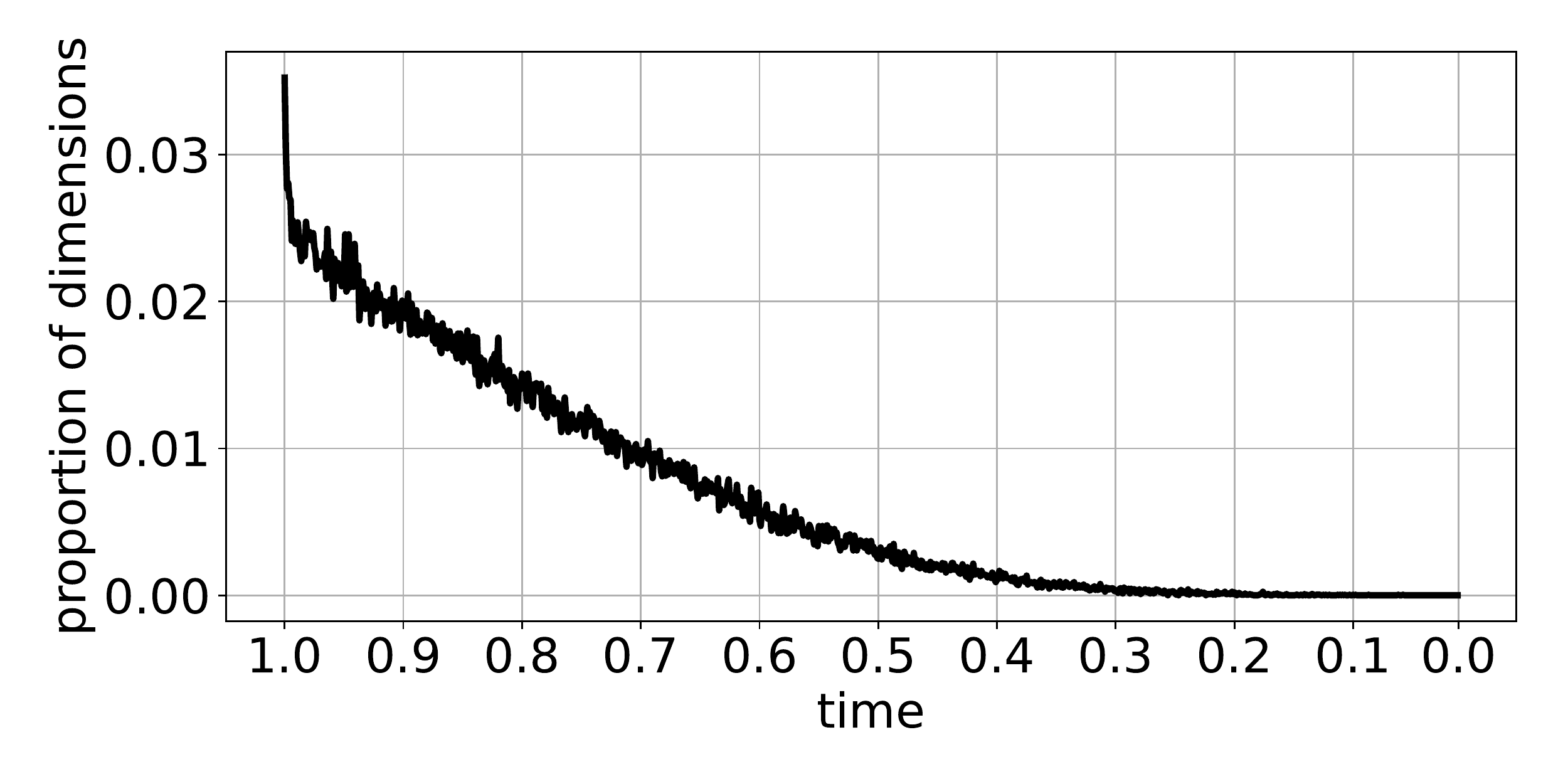}
    \caption{Proportion of dimensions that are clipped during a tau-leaping step due to proposing an out of bounds jump during the reverse sampling process.}
    \label{fig:apdxcifarPropDimsOOB}
\end{figure}

\begin{figure}
    \centering
    \includegraphics[width=0.8\textwidth]{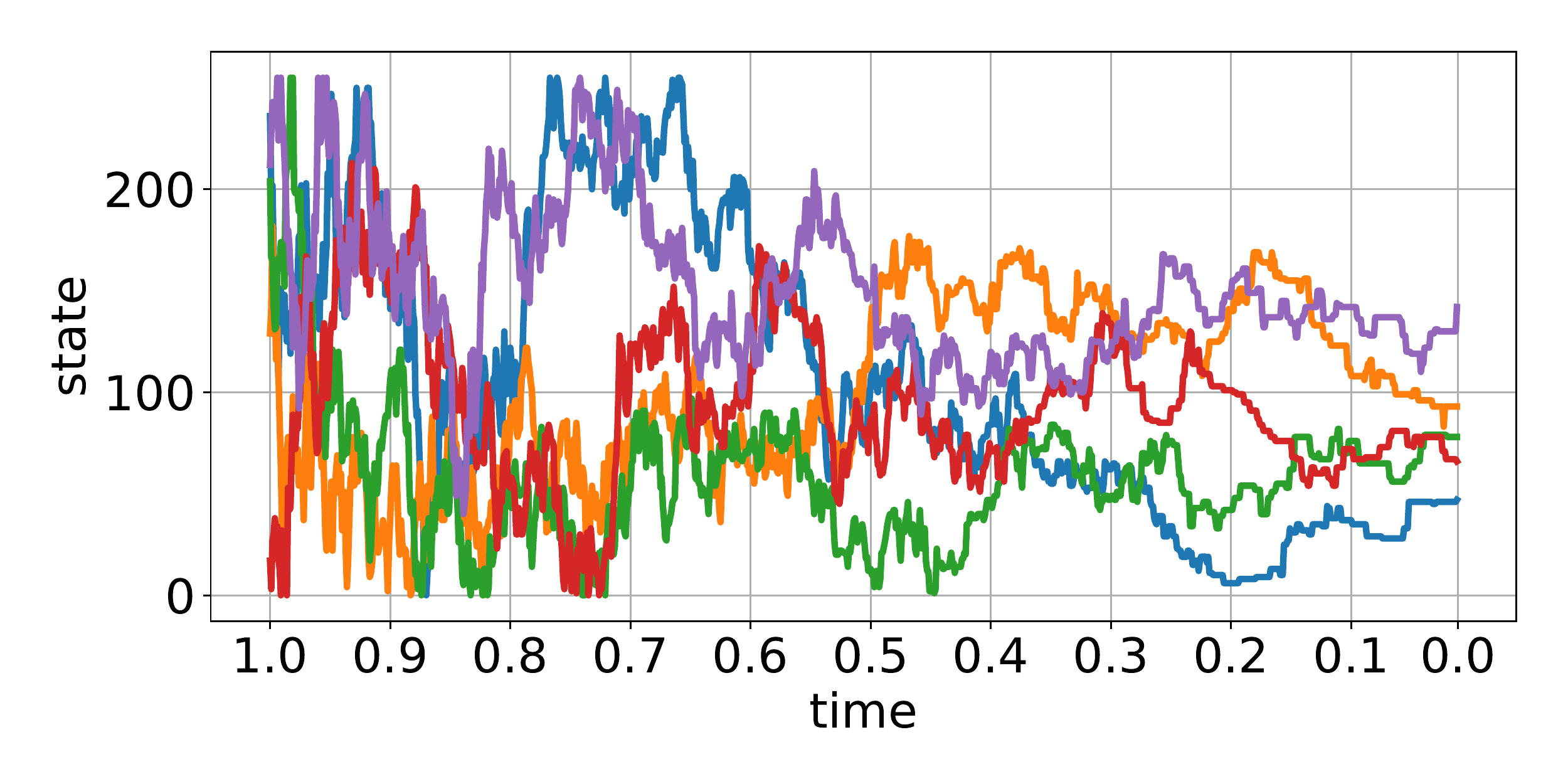}
    \caption{The progression of a selection of dimensions during the reverse sampling process.}
    \label{fig:apdxcifarFollowDims}
\end{figure}

\subsection{Monophonic Music} \label{sec:ApdxMusicExperiment}

We generate our training dataset from the Lakh pianoroll dataset \cite{raffel2016learning, dong2018musegan} (license CC By 4.0). This dataset consists of 174,154 multitrack pianorolls. We go through all songs and all tracks within each song and select sequences that match the following criteria: they are monophonic (only one note played at a time), there is not a period longer than one bar in which no note is played, there is more than one type of note played in the sequence and finally there is no one note played for more than 50 time steps out of the total 256 time steps. This removes the uninteresting and trivial sequences present within the dataset. We then remove any duplicates in the result. This leaves us with 6000 training examples and 950 testing examples. Each song consists of 256 time steps (16 per bar) and each time step takes one from 129 values i.e. we have $D=256$ and $S=129$. This state value represents either a note from 128 options or a rest. We scramble the ordering of this state space when mapping to the integers from 0 to 128. When we input into the denoising network, we input as one-hot $129$ dimensional vectors. \\

For the forward noising process, we use a uniform rate matrix, $R_b = \mathbbm{1} \mathbbm{1}^T - S \mathrm{Id}$ and set $\beta(t) = 0.03$. We found a constant in time $\beta(t)$ was sufficient for this dataset. In our comparison, we used a birth/death rate matrix defined as 
\begin{equation}
    \begin{bmatrix}
    -\lambda&\lambda&0&0 & \dots &0 & 0\\
    \lambda & -2\lambda & \lambda & 0 & \dots & 0 & 0\\
    0 & \lambda & -2 \lambda & \lambda & \dots & 0 & 0\\
    \vdots & \vdots & \vdots & \vdots & \ddots & \vdots & \vdots\\
    0 & 0 & 0 & 0& \dots & \lambda & -\lambda
    \end{bmatrix}
\end{equation}
this is the rate matrix for a birth/death process. We set $\lambda = 1$ and $\beta(t) = \frac{1}{2} \times 10000^t \log 10000$. These hyperparameters were selected such that the forward process has a steady rate of noising whilst still having $q_T$ very close to $\pref$. We chose to compare these types of rate matrix because the birth/death rate is inappropriate for this categorical data as adjacent states have no meaning since the mapping to the integers was arbitrary. The uniform rate is suitable for this categorical data because, during a time interval, it has a uniform probability to transition to any other state. The D3PM baseline was implemented also with a time homogeneous uniform forward kernel set such that the rate of noising is matched in the discrete and continuous time cases.\\

We define our conditional denoising network, $\smash{p_{0|t}^\theta}(x_0 | x, y)$ using a transformer architecture inspired by \cite{mittal2021symbolic}. It takes an input of shape $(B, D, S)$ where $B$ is the batch size, $D$ is the dimensionality ($256$) and $S$ is the state size ($129$). This final dimension contains the one-hot vectors. The conditioning on the initial bars is achieved by concatenating the conditioning information $y$ with the noisy input $x$. At the start of the network, there is an input embedding linear layer with output of size $128$ which is our model dimension for the transformer. Then a transformer positional embedding is added to the hidden state. Next a stack of $6$ transformer encoder layers are applied which consist of a self attention block and a one hidden layer MLP. The self attention block uses $8$ heads and the MLP has a hidden layer size of $2048$. At the output of each internal block, we apply dropout with rate $0.1$. Finally, there is a stack of 2 residual MLP layers. Each consists of a one hidden layer MLP with a hidden dimension of $2048$. There is a residual connection between the input and output of the MLP. A layer norm is applied to the output of the block. To create the output of the network, there is an output linear layer with an output shape of $(B, D, S)$ where now the $S$ dimension has logit probabilities. To instill a residual bias into the network, we add the one-hot input to the output logits. All activations are ReLU. The time is input into the network through FiLM layers \cite{perez2018film}. First, the time is embedded using the sinusoidal transformer position embedding as in the U-net architecture used for image modeling to create an embedding size of $128$. This is then passed into a single hidden layer MLP with hidden size $2048$ and output size $512$. Within each encoder and residual MLP block, there is a FiLM linear layer which takes in the $512$ time embedding and outputs two FiLM parameters each of size $128$. These are the scale and offset applied to the hidden state. In the encoder blocks, this FiLM transform is applied after the self attention block and again after the fully connected block. In the residual MLP blocks, it is applied after the layer norm operation. Our network has approximately $7$ million parameters in total.\\

We optimize using Adam for 1M steps with a batch size of 64 and learning rate of 0.0002. We use the conditional $\bar{\mathcal{L}}_{\textup{CT}}$ objective with additional direct $p_{0|t}^\theta$ supervision as described in Appendix \ref{sec:ApdxDirectDenoisingSupervision} with weight $\lambda=0.001$. We also make the same one forward pass approximation as explained in Appendix \ref{sec:ApdxCTELBOOneForwardPass}.  We use the standard set of training tricks to improve optimization \cite{ho2020denoising, song2020score}. Throughout training we maintain an exponential moving average of the parameters with decay factor 0.9999. These average parameters are used during testing. At the start of optimization we use a linear learning rate warm-up for the first 5000 steps. We clip the gradient norm at a norm value of 1.0. We train on a single V100 GPU on an academic cluster.\\

For sampling with $\tau$LDR-$0$ we use $\tau=0.001$ and for sampling with $\tau$LDR-$2$ we include $2$ corrector steps per predictor step after $t < 0.9T$. The corrector rate is simulated with $\tau=0.0001$ which we found to perform best. We reject any dimension in which 2 or more jumps are proposed as this is categorical data. We plot the rejection rate in Figure \ref{fig:ApdxPianoRejectionRate}. Most of the time, the rejection rate is zero and there are few steps for which it increases slightly. We show a large batch of samples from the first 10 songs in the test dataset in Figure \ref{fig:ApdxPianoExtraSamples}. We see that there is variation between the sampled completions and they consistently follow the style of the conditioning first two bars of the song. Audio samples from the model are available at \url{https://github.com/andrew-cr/tauLDR}. Finally, we examine the progression of a random selection of dimensions during reverse sampling for the uniform and birth/death rate matrix cases. Figure \ref{fig:ApdxPianoUniformFollowDim} shows the progression for the uniform case, we see that large jumps through the state space are made throughout the reverse process. Figure \ref{fig:ApdxPianoBDFollowDim} shows the progression for the birth/death case. At the start of reverse sampling, no dimensions move as the rejection rate is high in this case because the rate matrix is not suitable for categorical data. Nearer the end, small jumps are made between adjacent states but since large jumps between any category do not occur for this rate matrix, the performance will overall be worse.

\begin{figure}
    \centering
    \includegraphics[width=0.8\textwidth]{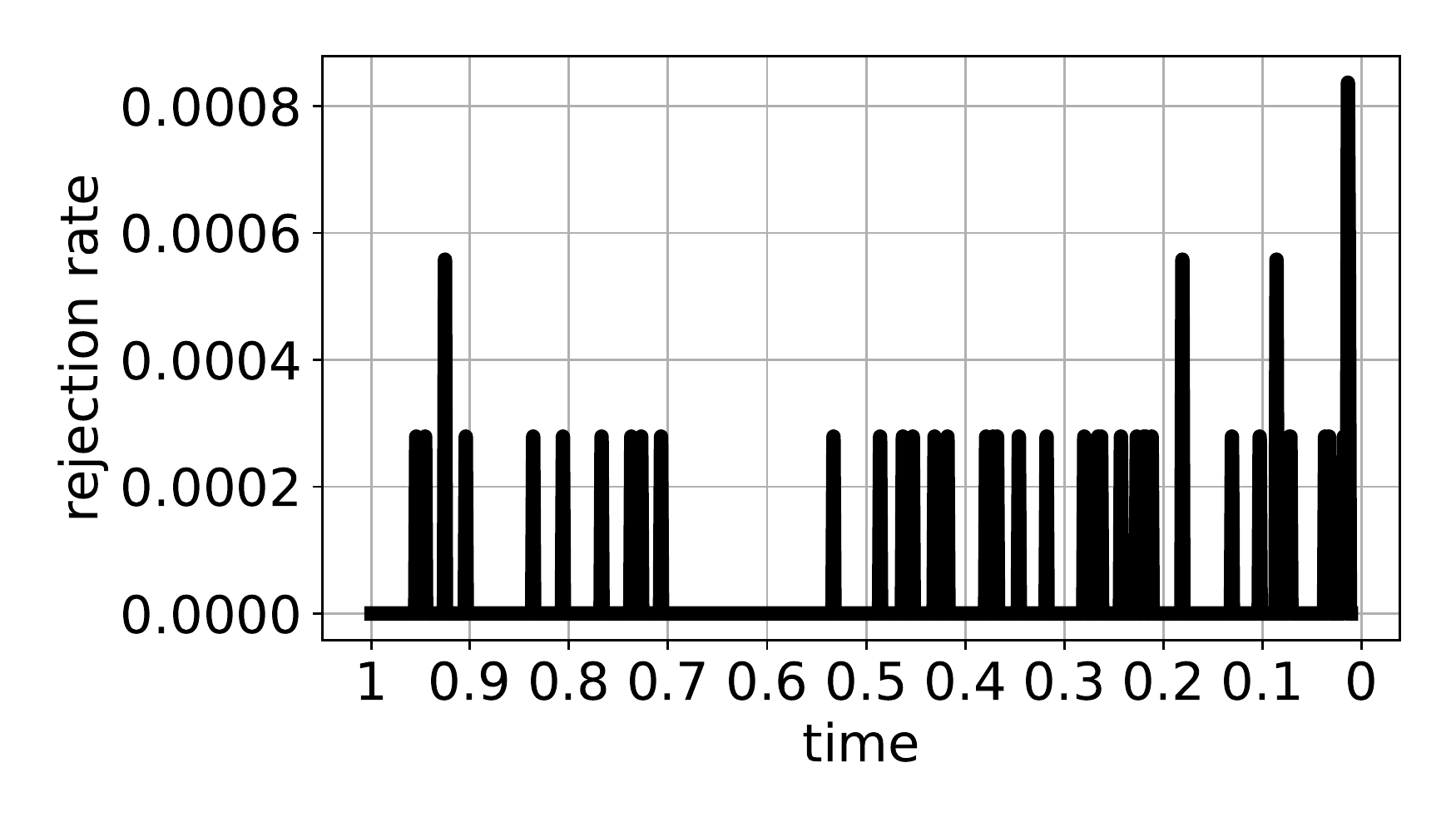}
    \caption{Proportion of jumps rejected during reverse sampling. The rejection rate is calculated as the proportion of dimensions in a tau leaping step that have their jump rejected. The results are averaged over a batch of 16.}
    \label{fig:ApdxPianoRejectionRate}
\end{figure}

\begin{figure}
    \centering
    \includegraphics[width=\textwidth]{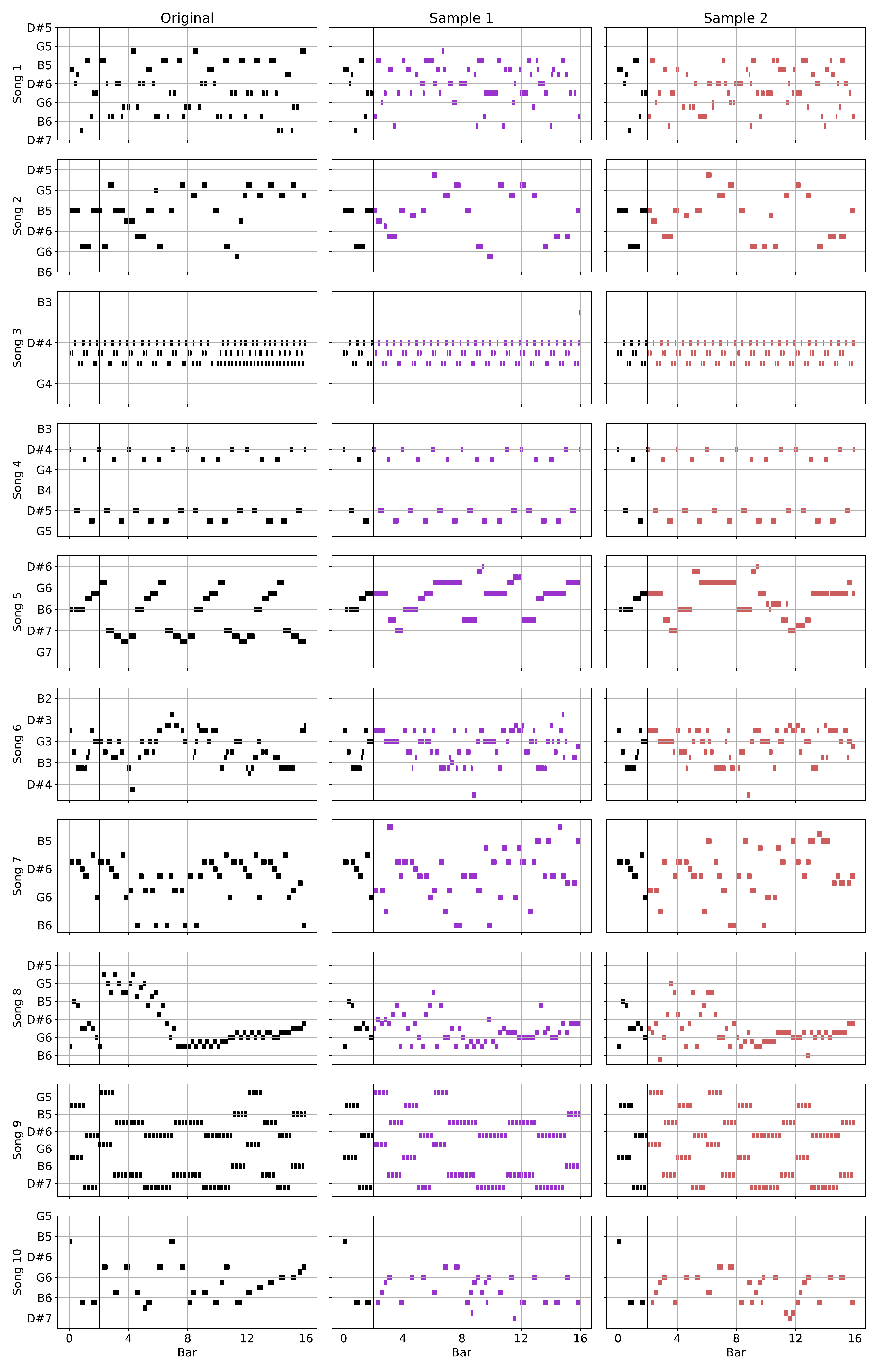}
    \caption{Two conditional samples from the $\tau$LDR-$0$ model for each of the first 10 songs in the test dataset.}
    \label{fig:ApdxPianoExtraSamples}
\end{figure}

\begin{figure}
    \centering
    \includegraphics[width=0.8\textwidth]{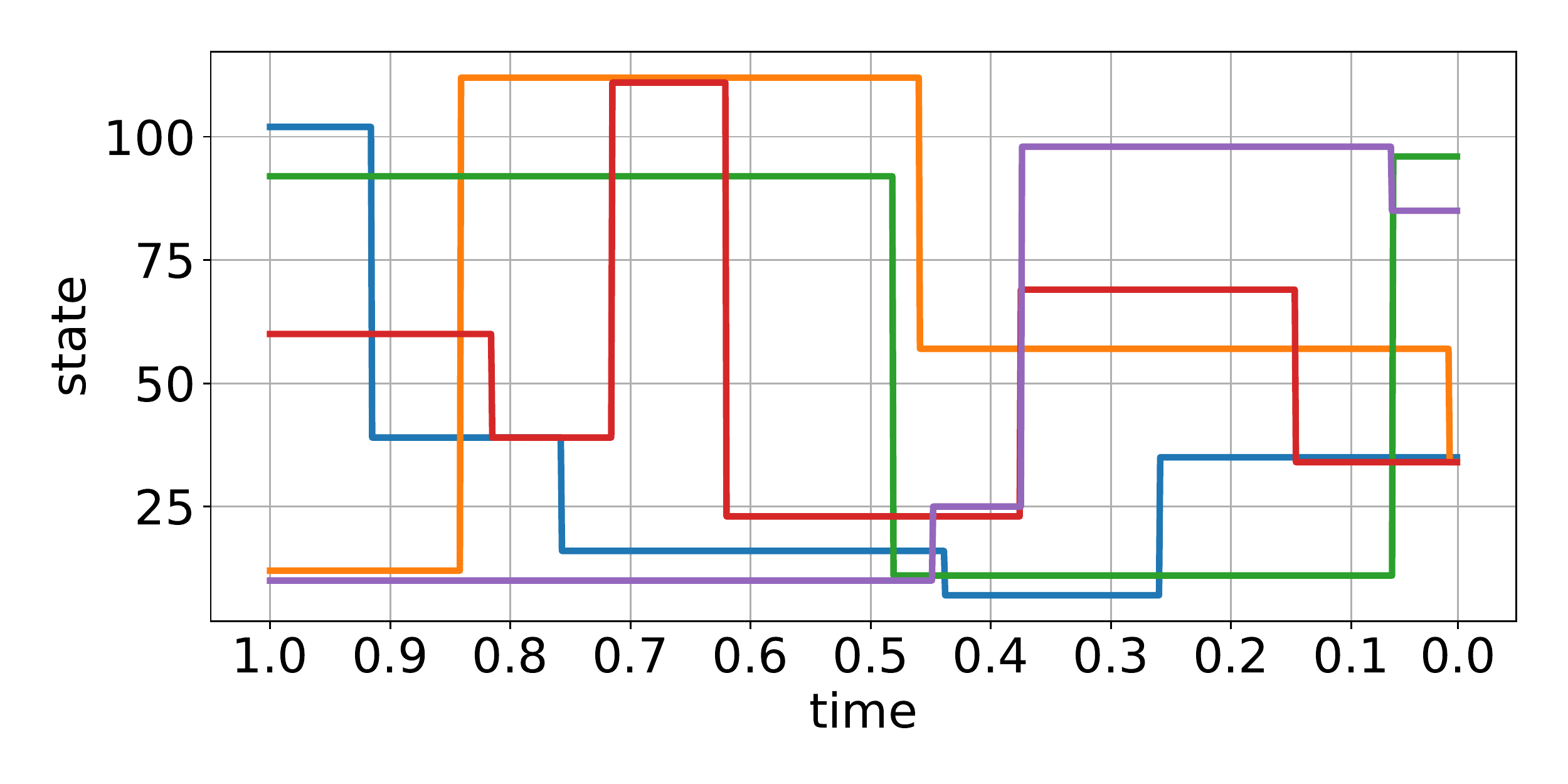}
    \caption{The progression of a selection of dimensions during the reverse sampling process for the uniform rate matrix.}
    \label{fig:ApdxPianoUniformFollowDim}
\end{figure}

\begin{figure}
    \centering
    \includegraphics[width=0.8\textwidth]{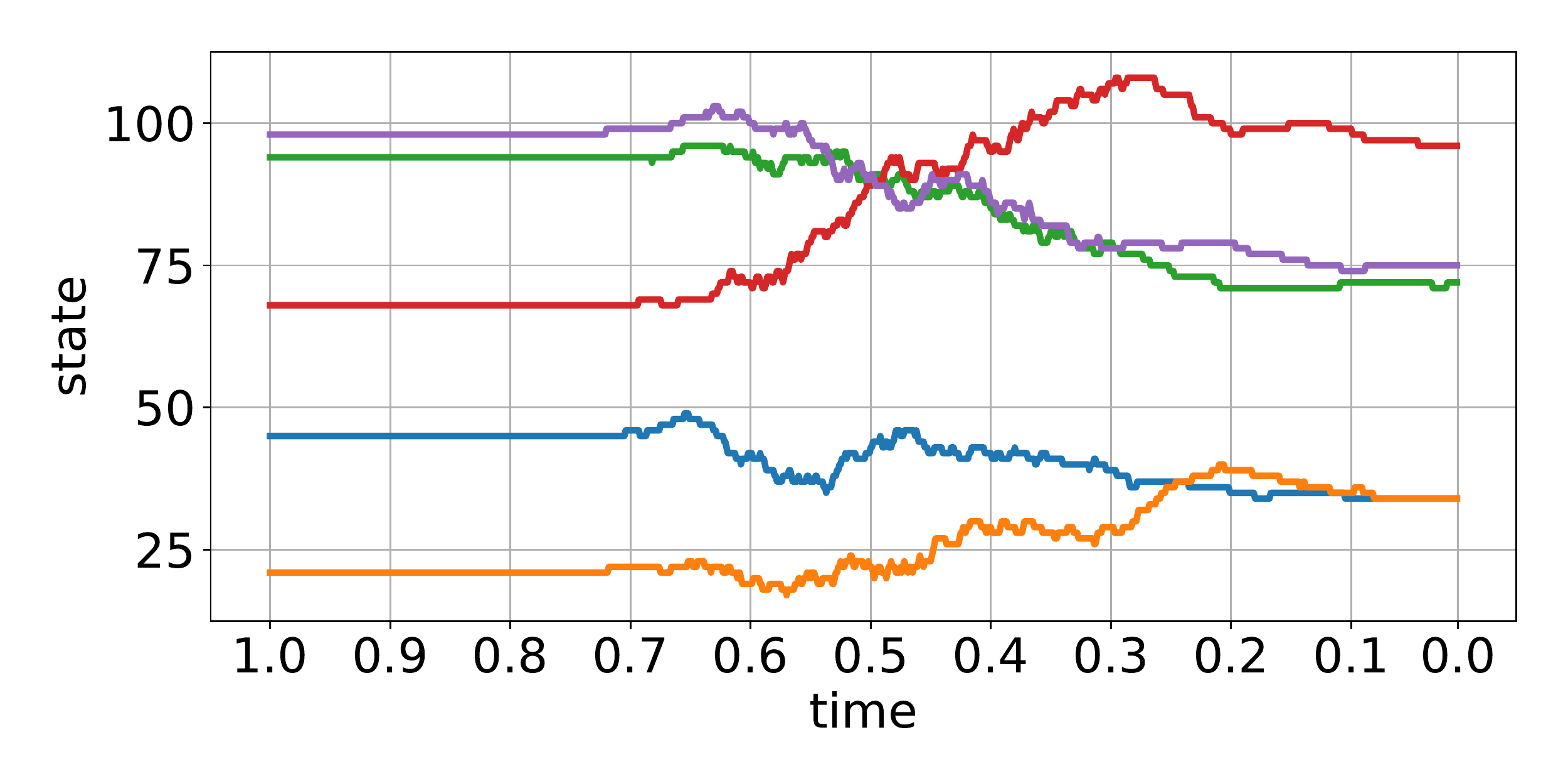}
    \caption{The progression of a selection of dimensions during the reverse sampling process for the birth/death rate matrix.}
    \label{fig:ApdxPianoBDFollowDim}
\end{figure}

\section{Ethical Considerations} \label{sec:ApdxEthicalConsiderations}
Our work increases our theoretical understanding of denoising generative models and also improves generation capabilities within some discrete datasets. Deep generative models are generic methods for learning from unstructured data and can have negative social impacts when misused. For example, they can be used to spread misinformation by reducing the resources required to create realistic fake content. Furthermore, generative models will produce samples that accurately reflect the statistics of their training dataset. Therefore, if samples from these models are interpreted as an objective truth without fully considering the biases present in the original data, then they can perpetuate discrimination against minority groups.

In this work, we train on datasets that contain less sensitive data such as pictures of objects and music samples. The methods we presented, however, could be used to model images of people or text from the internet which will contain biases and potentially harmful content that the model will then learn from and reproduce. Great care must be taken when training these models on real world datasets and when deploying them so as to mitigate and prevent the harms that they can cause.

\end{document}